\documentclass[10pt,twocolumn,twoside]{IEEEtran}
 \usepackage{cite}
\ifCLASSINFOpdf
\else
\fi
\usepackage{amsfonts,amssymb}
\usepackage{amsthm,bbm,bm}
\usepackage{color}
\usepackage[cmex10]{amsmath}
\newtheorem{theorem}{Theorem}
\newtheorem{lemma}{Lemma}

\newtheorem{definition}{Definition}
\newtheorem{proposition}{Proposition}

\newtheorem{condition}{Condition}

\usepackage{graphicx}

\usepackage{algorithm}
\usepackage{algorithmic}
\DeclareMathOperator*{\argmin}{arg\,min}

\DeclareMathOperator{\diag}{diag}

\DeclareMathOperator{\eE}{\mathbb{E}}
\newcommand{\rR}{\mathbb{R}}

\def\numofproj{{P}}

\usepackage{tabularx,ragged2e,booktabs,array}

\ifCLASSOPTIONcompsoc
  \usepackage[caption=false,font=normalsize,labelfont=sf,textfont=sf]{subfig}
\else
  \usepackage[caption=false,font=footnotesize]{subfig}
\fi
\usepackage{url}
\usepackage[breaklinks = true]{hyperref}

\definecolor{edit}{rgb}{0,0,0}

\begin{document}
\title{Necessary and Sufficient Conditions and a Provably Efficient
  Algorithm for Separable Topic Discovery}
%A Provably Efficient Algorithm for \\Separable Topic Discovery}

\author{Weicong Ding,$^*$~\IEEEmembership{}
        Prakash Ishwar,~\IEEEmembership{IEEE Senior Member}
        and~Venkatesh Saligrama,~\IEEEmembership{IEEE Senior Member }% <-this % stops a space
%\thanks{W.Ding, P. Ishwar and V. Saligrama are with the Department of Electrical and Computer Engineering, Boston University, 8 Saint Mary's Street, Boston, MA, 02215 USA. e-mails: $\{$dingwc,pi,srv$\}$@bu.edu.}% <-this % stops a space
%\thanks{Manuscript received April 19, 2005; revised September 17, 2014.}
}


\markboth{
%Journal of Special Issue on Signal Processing
}%
{
%Ding \MakeLowercase{\textit{et al.}}: A Provably Efficient Algorithm for Separable Topic Discovery
}



\maketitle
\begin{abstract}
We develop necessary and sufficient conditions and a novel provably
consistent and efficient algorithm for discovering topics (latent
factors) from observations (documents) that are realized from a
probabilistic mixture of shared latent factors that have certain
properties. Our focus is on the class of topic models in which each
shared latent factor contains a novel word that is unique to that
factor, a property that has come to be known as separability. Our
algorithm is based on the key insight that the novel words correspond
to the extreme points of the convex hull formed by the row-vectors of
a suitably normalized word co-occurrence matrix. We leverage this
geometric insight to establish polynomial computation and sample
complexity bounds based on a few isotropic random projections of the
rows of the
%an appropriately 
normalized word co-occurrence matrix. Our proposed
random-projections-based algorithm is naturally amenable to an
efficient distributed implementation and is attractive for modern
web-scale distributed data mining applications.
\end{abstract}


% \fi
\IEEEpeerreviewmaketitle

\section{Introduction}
\IEEEPARstart{T}{opic} modeling refers to a family of generative
models and associated algorithms for discovering the (latent) topical
structure shared by a large corpus of documents. They are important
for organizing, searching, and making sense of a large text corpus
\cite{Blei2012Review:ref}. In this paper we describe a novel geometric
approach, with provable statistical and computational efficiency
guarantees, for learning the latent topics in a document collection.
This work is a culmination of a series of recent publications on
certain structure-leveraging methods for topic modeling with provable
theoretical guarantees \cite{Dingetal2013ICASSP:ref, DDP:ref,
  Ding13b:ref, Ding14:ref}.

% overall models
We consider a corpus of $M$ documents, indexed by $m = 1, \ldots, M$,
each composed of words from a fixed vocabulary of size $W$. The
distinct words in the vocabulary are indexed by $w = 1,\ldots,W$. Each
document $m$ is viewed as an unordered ``bag of words'' and is
represented by an empirical $W\times 1$ word-counts vector
$\mathbf{X}^m$, where $X_{w,m}$ is the number of times that word $w$
appears in document $m$
\cite{LDA:ref,Blei2012Review:ref,Arora2:ref,Ding14:ref}. The entire
document corpus is then represented by the $W\times M$ matrix
$\mathbf{X} = \left[\mathbf{X}^{1},\ldots, \mathbf{X}^{M} \right]$.
\footnote{When it is clear from the context, we will use $X_{w,m}$ to
  represent either the empirical word-count or, by suitable
  column-normalization of $\mathbf{X}$, the empirical word-frequency.}
A ``topic'' is a $W\times 1$ distribution over the vocabulary.  A
topic model posits the existence of $K < \min(W, M)$ latent topics
that are {\it shared} among all $M$ documents in the corpus. The
topics can be collectively represented by the $K$ columns
$\bm\beta^{1}, \ldots, \bm\beta^{K}$ of a $W\times K$
column-stochastic ``topic matrix'' $\bm\beta$.
Each document $m$ is conceptually modeled as being generated
independently of all other documents through a two-step process:
1) first draw a $K \times 1$ document-specific distribution over
topics $\bm\theta^m$ from a prior distribution $\Pr(\bm\alpha)$ on the
probability simplex with some hyper-parameters $\bm\alpha$;
2) then draw $N$ iid words according to a $W \times 1$
document-specific word distribution over the vocabulary given by
$\mathbf{A}^{m} = \sum_{k=1}^{K}\bm\beta^{k}\theta_{k,m}$ which is a
convex combination (probabilistic mixture) of the latent topics.
Our goal is to estimate $\bm\beta$ from the matrix of empirical
observations $\mathbf{X}$.
%{\it primary goal} in this paper is to present (i) a novel
%theoretic framework, (ii) algorithms with provable theoretical
%guarantees, and (iii) experimental results for the problem
%of estimating $\bm\beta$ from the matrix of empirical observations
%$\mathbf{X}$.
%
To appreciate the difficulty of the problem, consider a typical
benchmark dataset such as a news article collection from the New York
Times (NYT) \cite{UCIdataset:ref} that we use in our experiments.
%after pruning the vocabulary based on document frequencies and then
%removing a standard stop-word list \cite{Arora2:ref,Ding14:ref}, we
%get
In this dataset, after suitable pre-processing, $W = 14,943$, $M
=300,000$, and, on average, $N = 298$. Thus, $N \ll W \ll M$,
$\mathbf{X}$ is very sparse, and $M$ is very large. Typically, $K
\approx 100 \ll \min(W,M)$.
%
%

% difficulties, separability, 
This estimation problem in topic modeling has been extensively studied. The prevailing approach is to compute the MAP/ML estimate \cite{Blei2012Review:ref}. The true
posterior of $\bm\beta$ given $\mathbf{X}$, however, is intractable to
compute and the associated MAP and ML estimation problems are in fact
NP-hard in the general case \cite{A12:ref, Sontag11:ref}. This
necessitates the use of sub-optimal methods based on approximations
and heuristics such as Variational-Bayes and MCMC
\cite{LDA:ref,Griffiths04Gibbs:ref,wainwright2008graphical,
  MMLVM14:ref}.
While they produce impressive empirical results on many real-world
datasets, guarantees of asymptotic consistency or efficiency for these
approaches are either weak or non-existent.
This makes it difficult to evaluate {\it model fidelity}:  failure to
produce satisfactory results in new datasets could be due to the use
of approximations and heuristics or due to model mis-specification
which is more fundamental.
Furthermore, these sub-optimal approaches are computationally intensive
for large text corpora \cite{Arora2:ref,Ding14:ref}.

To overcome the hardness of the topic estimation problem in its full
generality, a new approach has emerged to learn the topic model by
imposing additional structure on the model parameters
\cite{A12:ref,Arora2:ref,DDP:ref,Ding14:ref,anandkumar2014tensor,Kumar13:ref}.
This paper focuses on a key structural property of the topic matrix
$\bm\beta$ called {\bf topic separability}
\cite{Arora2:ref,DDP:ref,Ding14:ref,Kumar13:ref} wherein every latent
topic contains at least one word that is {\bf novel} to it, i.e., the
word is unique to that topic and is absent from the other topics. This
is, in essence, a property of the support of the latent topic matrix
$\bm\beta$.
The topic separability property can be motivated by the fact that for many
real-world datasets, the empirical topic estimates produced by popular
Variational-Bayes and Gibbs Sampling approaches are approximately
separable \cite{Arora2:ref,Ding14:ref}.
Moreover, it has recently been shown that the separability property
will be approximately satisfied with high probability when the
dimension of the vocabulary $W$ scales sufficiently faster than the
number of topics $K$ and $\bm\beta$ is a realization of a Dirichlet
prior that is typically used in practice \cite{Ding15HighProb:ref}.
Therefore, separability is a {\it natural} approximation for {\it
  most} high-dimensional topic models.

Our approach exploits the following geometric implication of the key
separability structure.
If we associate each word in the vocabulary with a row-vector of a
suitably normalized empirical word co-occurrence matrix, {\bf the set
  of novel words correspond to the extreme points} of the convex hull
formed by the row-vectors of all words.
We leverage this geometric insight and develop a provably consistent
and efficient algorithm. Informally speaking, we establish the
following result:
\begin{theorem}
\label{thm:informal-statement}
If the topic matrix is separable and the mixing weights satisfy a
minimum information-theoretically necessary technical condition, then
our proposed algorithm runs in {\bf polynomial time} in $M,W,N,K$, and
estimates the topic matrix {\bf consistently} as $M\rightarrow\infty$
with $N \geq 2$ held fixed.
Moreover, our proposed algorithm can estimate $\bm\beta$ to within an
$\epsilon$ element-wise error with a probability at least $1-\delta$
if $M \geq \textbf{Poly}\left( W, 1/N, K, \log(1/\delta), 1/\epsilon
\right)$.
\end{theorem}
%
%an approach that can consistently recover the topic matrix $\bm\beta$ when $N$ is fixed and the number of documents $M$ increases to infinity, and have polynomial sample and computation complexity with respect to all  model parameters ($M,N,K,W$).
%
\noindent The asymptotic setting $M\rightarrow\infty$ with $N$ held
fixed is motivated by text corpora in which the number of words in a
single document is small while the number of documents is large.
We note that our algorithm can be applied to any family of topic
models whose topic mixing weights prior $\Pr(\bm\alpha)$ satisfies a
minimum information-theoretically necessary technical condition.  In
contrast, the standard Bayesian approaches such as Variational-Bayes
or MCMC need to be hand-designed separately for each specific topic
mixing weights prior.
The highlight of our approach is to identify the novel words as
extreme points through appropriately defined {\bf random projections}.
Specifically, we project the row-vector of each word in an
appropriately normalized word co-occurrence matrix along a few
independent and isotropically distributed random directions. The
fraction of times that a word attains the maximum value along a random
direction is a measure of its degree of robustness as an extreme
point.
This process of random projections followed by counting the
number of times a word is a maximizer can be efficiently computed and
is robust to the perturbations induced by sampling noise associated
with having only a very small number of words per document $N$.
In addition to being computationally efficient, it turns out that this
random projections based approach $(1)$ requires the {\it minimum}
information-theoretically necessary technical conditions on the topic
prior for asymptotic consistency, and $(2)$ can be naturally
parallelized and distributed. As a consequence, it can provably
achieve the efficiency guarantees of a centralized method while
requiring insignificant communication between {\it distributed}
document collections \cite{Ding14:ref}. This is attractive for
web-scale topic modeling of large distributed text corpora.

Another advance of this paper is the identification of necessary and
sufficient conditions on the mixing weights for consistent separable
topic estimation.
In previous work we showed that a {\it simplicial} condition on the
mixing weights is both necessary and sufficient for consistently {\it
  detecting} all the novel words \cite{Ding13b:ref}. In this paper we
complete the characterization by showing that an {\it affine
  independence} condition on the mixing weights is necessary and
sufficient for consistently {\it estimating} a separable topic matrix.
These conditions are satisfied by practical choices of topic priors
such as the Dirichlet distribution \cite{LDA:ref}.
All these necessary conditions are information-theoretic and
algorithm-independent, i.e., they are irrespective of the specific
statistics of the observations or the algorithms that are used.
The provable statistical and computational efficiency guarantees of
our proposed algorithm hold true under these necessary and sufficient
conditions.
%

% organization of thie paper
% [revisit this paragraph at the end]
The rest of this paper is organized as follows. We review related work
on topic modeling as well as the separability property in various
domains in Sec.~\ref{section:related_work}. We introduce the
separability property on $\bm\beta$, the simplicial and affine
independence conditions on mixing weights, and the extreme point
geometry that motivates our approach in
Sec.~\ref{section:ideal_geometry}. We then discuss how the solid
angle can be used to identify robust extreme points to deal with a
finite number of samples (words per document) in
Sec.~\ref{section:finite_geometry}. We describe our overall
algorithm and sketch its analysis in
Sec.~\ref{section:algorithm}. We demonstrate the performance of our
approach in Sec.~\ref{section:experiments} on various synthetic and
real-world examples.
%
% *************PI: remember to say this:
Proofs of all 
%new 
results appear in the appendices.
%
%%%%%%%%%%%%%%%%%%%%%%%%%%%
\section{Related Work}
\label{section:related_work}
The idea of modeling text documents as mixtures of a few semantic
topics was first proposed in \cite{pLSI:ref} where the mixing weights
were assumed to be deterministic.  Latent Dirichlet Allocation (LDA)
in the seminal work of \cite{LDA:ref} extended this to a probabilistic
setting by modeling topic mixing weights using Dirichlet priors. This
setting has been further extended to include other topic priors such
as the log-normal prior in the Correlated Topic Model
\cite{Blei07:ref}. LDA models and their derivatives have been
successful on a wide range of problems in terms of achieving good
empirical performance \cite{Blei2012Review:ref,MMLVM14:ref}.

% To estimate and inference the topic models,
The prevailing approaches for estimation and inference problems in
topic modeling are based on MAP or ML estimation
\cite{Blei2012Review:ref}.
However, the computation of posterior distributions conditioned on
observations $\mathbf{X}$ is intractable \cite{LDA:ref}. Moreover, the
MAP estimation objective is non-convex and has been shown to be
$\mathcal{NP}$-hard \cite{Sontag11:ref,A12:ref}.
Therefore various approximation and heuristic strategies have been
employed. These approaches fall into two major categories -- sampling
approaches and optimization approaches.
Most sampling approaches are based on Markov Chain Monte Carlo (MCMC)
algorithms that seek to generate (approximately) independent samples
from a Markov Chain that is carefully designed to ensure that the
sample distribution converges to the true posterior
\cite{Griffiths04Gibbs:ref,Wallach09:ref}.
% A conjugate prior on topic weights are favorable to achieve efficient Gibbs Sampling.
%
Optimization approaches are typically based on the so-called
Variational-Bayes methods. These methods optimize the parameters of a
simpler parametric distribution so that it is close to the true
posterior in terms of KL divergence
\cite{LDA:ref,wainwright2008graphical}. Expectation-\-Maximization-\-type
algorithms are typically used in these methods.  In practice, while
both Variational-Bayes and MCMC algorithms have similar performance,
Variational-Bayes is typically faster than MCMC
\cite{hoffman2010online,Blei2012Review:ref}.

Nonnegative Matrix Factorization (NMF) is an alternative approach for
topic estimation. NMF-based methods exploit the fact that both the
topic matrix $\bm{\beta}$ and the mixing weights are nonnegative and
attempt to decompose the empirical observation matrix $\mathbf{X}$
into a product of a nonnegative topic matrix $\bm\beta$ and the matrix
of mixing weights by minimizing a cost function of the form
\cite{nmfLS:ref,NMFbook:ref,recht2012factoring,hoffman2010online}
\begin{equation*}
%
%\min 
\sum_{m=1}^{M} d(\mathbf{X}^{m}, \bm\beta\bm\theta^{m}) +\lambda
\psi(\bm\beta, \bm\theta^{1},\ldots,\bm\theta^{M}),
\end{equation*}
where $d( , )$ is some measure of closeness and $\psi$ is a
regularization term which enforces desirable properties, e.g.,
sparsity, on $\bm\beta$ and the mixing weights.  The NMF problem,
however, is also known to be non-convex and $\mathcal{NP}$-hard
\cite{Vava09:ref} in general. Sub-optimal strategies such as
alternating minimization, greedy gradient descent, and heuristics are
used in practice \cite{NMFbook:ref}.

In contrast to the above approaches, a new approach has recently
emerged which 
%emerging theme 
is based on imposing additional structure on the model parameters
\cite{A12:ref,Arora2:ref,DDP:ref,Ding14:ref,anandkumar2014tensor,Kumar13:ref}.
These approaches show that the topic discovery problem lends itself to
provably consistent and polynomial-time solutions by making
assumptions about the {\it structure} of the topic matrix $\bm\beta$
and the distribution of the mixing weights. In this category of
approaches are methods based on a tensor decomposition of the moments
of $\mathbf{X}$ \cite{Anan13:ref, anandkumar2014tensor}.
The algorithm in \cite{Anan13:ref} uses second order empirical moments
and is shown to be asymptotically consistent when the topic matrix
$\bm\beta$ has a special sparsity structure.
The algorithm in \cite{anandkumar2014tensor} uses the third order
tensor of observations. It is, however, strongly tied to the specific
structure of the Dirichlet prior on the mixing weights and requires
knowledge of the concentration parameters of the Dirichlet
distribution \cite{anandkumar2014tensor}.
Furthermore, in practice these approaches are computationally
intensive and require some initial coarse dimensionality reduction,
gradient descent speedups, and GPU acceleration to process large-scale
text corpora like the NYT dataset \cite{anandkumar2014tensor}.
%Furthermore, these approaches for large scale text corpus are
%computationally intensive and require gradient descent speedups and
%GPU acceleration \cite{anandkumar2014tensor}.

%
Our work falls into the family of approaches that exploit the
separability property of $\bm\beta$ and its geometric implications
\cite{A12:ref,Arora2:ref,DDP:ref,Ding14:ref,Kumar13:ref,Awasthi15VB:ref,bansal2014provable}.
An asymptotically consistent polynomial-time topic estimation
algorithm was first proposed in \cite{A12:ref}.
However, this method requires solving $W$ linear programs, each with
$W$ variables and is computationally impractical.
Subsequent work improved the computational efficiency
\cite{recht2012factoring, Kumar13:ref}, but theoretical guarantees of
asymptotic consistency (when $N$ fixed, and the number of documents
$M\rightarrow \infty$) are unclear.
Algorithms in \cite{Arora2:ref} and \cite{DDP:ref} are both practical
and provably consistent. Each requires a stronger and slightly
different technical condition on the topic mixing weights than
\cite{A12:ref}.
Specifically, \cite{Arora2:ref} imposes a full-rank condition on the
second-order correlation matrix of the mixing weights and proposes a
Gram-Schmidt procedure to identify the extreme points. Similarly,
\cite{DDP:ref} imposes a diagonal-dominance condition on the same
second-order correlation matrix and proposes a random projections
based approach.
These approaches are tied to the specific conditions imposed and they
both fail to detect all the novel words and estimate topics when the
imposed conditions (which are sufficient but not necessary for
consistent novel word detection or topic estimation) fail to hold in
some examples \cite{Ding14:ref}.
The random projections based algorithm proposed in \cite{Ding14:ref}
is both practical and provably consistent. Furthermore, it requires
fewer constraints on the topic mixing weights.

We note that the separability property has been exploited in other
recent work as well \cite{bansal2014provable,Awasthi15VB:ref}.  In
\cite{bansal2014provable}, a singular value decomposition based
approach is proposed for topic estimation. In \cite{Awasthi15VB:ref},
it is shown that the standard Variational-Bayes approximation can be
asymptotically consistent if $\bm\beta$ is separable. However, the
additional constraints proposed 
%in \cite{bansal2014provable} 
%% IS THIS THE CORRECT REFERENCE TO CITE HERE? SEEMS LIKE IT SHOULD BE
%% AWASTHI15VB
essentially boil down to the 
%condition 
requirement that each document contain predominantly only one topic.
In addition to assuming the existence of such ``pure'' documents,
\cite{Awasthi15VB:ref} also requires a strict initialization. It is
thus unclear how this can be achieved using only the observations
$\mathbf{X}$.

The separability property has been re-discovered and exploited in the
literature across a number of different fields and has found
application in several problems.
% different fields.
%
To the best of our knowledge, this concept was first introduced as the
{\it Pure Pixel Index} assumption in the Hyperspectral Image unmixing
problem \cite{boardman1993automating:ref}. This work assumes the
existence of pixels in a hyper-spectral image containing predominantly
one species.
Separability has also been studied in the NMF literature in the
context of ensuring the uniqueness of NMF \cite{Donhunique:ref}.
Subsequent work has led to the development of NMF algorithms that
exploit separability
%in different aspects in the context of NMF
\cite{recht2012factoring,Gillis2014:ref}.
The uniqueness and correctness results in this line of work has
primarily focused on the noiseless case.
We finally note that separability has also been recently exploited in
the problem of learning multiple ranking preferences from pairwise
comparisons for personal recommendation systems and information
retrieval \cite{Farias09:ref,topicRank2:ref} and has led to provably
consistent and efficient estimation algorithms.
%
%
%******************************************************************
\section{Topic Separability, Necessary and Sufficient Conditions, and the Geometric Intuitions}
\label{section:ideal_geometry}

In this section, we unravel the key ideas that motivate our
algorithmic approach by focusing on the ideal case where there is no
``sampling-noise'', i.e., each document is infinitely long ($N =
\infty$). In the next section, we will turn to the finite $N$ case.
%
%This section formally introduces the key structural property, namely
%topic separability, and the information-theoretic necessary and
%sufficuent conditions on the topic mixing weights that guarantee the
%uniqueness of topic recovery.
%
We recall that $\bm\beta$ and $\mathbf{X}$ denote the $W\times K$
topic matrix and the $W\times M$ empirical word counts/frequency
matrix respectively. Also, $M, W$, and $K$ denote, respectively, the
number of documents, the vocabulary size, and the number of topics.
For convenience, we group the document-specific mixing weights, the
$\bm\theta^{m}$'s, into a $K \times M$ weight matrix $\bm\theta =
\left[\bm\theta^{1},\ldots, \bm\theta^{M}\right]$ and the
document-specific distributions, the $\mathbf{A}^{m}$'s, into a $W
\times M$ document distribution matrix $\mathbf{A}=
\left[\mathbf{A}^{1},\ldots, \mathbf{A}^{M}\right]$. The generative
procedure that describes a topic model then implies that $\mathbf{A} =
\bm\beta \bm\theta$. In the ideal case considered in this section ($N
= \infty$), the empirical word {\it frequency} matrix $\mathbf{X} =
\mathbf{A}$.
{\bf Notation:} A vector $\mathbf{a}$ without specification will
denote a column-vector, $\mathbf{1}$ the all-ones column vector of
suitable size, $\mathbf{X}^{i}$ the $i$-th column vector and
$\mathbf{X}_j$ the $j$-th row vector of matrix $\mathbf{X}$, and
$\bar{\mathbf{B}}$ a suitably row-normalized version (described later)
of a nonnegative matrix $\mathbf{B}$. Also, $[n] :=\lbrace 1,\ldots,
n \rbrace$.
\subsection{Key Structural Property: Topic Separability}
\label{sec:subsetction:separability}
We first introduce separability as a key structural property of a
topic matrix $\bm\beta$. Formally,
\begin{definition}
\label{definiton:separability}
{\bf(Separability)}
A topic matrix ${\bm \beta}\in\rR^{W\times K}$ is {separable} if for
each topic $k$, there is some word $i$ such that ${\bm \beta}_{i,k}>0$
and ${\bm\beta}_{i,l}=0$, $\forall ~ l\neq k$.
\end{definition}
Topic separability implies that each topic contains word(s) which
appear only in that topic. We refer to these words as the {\bf novel
  words} of the $K$ topics.
%
%\vglue -3ex
\begin{figure}[!htb]
\centering
\includegraphics[width=3.5in]{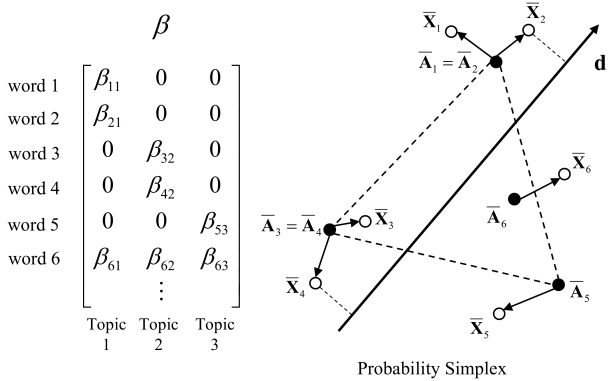}
%
% where an .eps filename suffix will be assumed under latex, 
% and a .pdf suffix will be assumed for pdflatex; or what has been declared
% via \DeclareGraphicsExtensions.
%
\vglue -3ex
\caption{An example of separable topic matrix $\bm\beta$ (left) and
  the underlying geometric structure (right) of the row space of the
  normalized document distribution matrix $\bar{\mathbf{A}}$. Note:
  the word ordering is only for visualization and has no bearing on
  separability. Solid circles represent {\bf rows} of
  $\bar{\mathbf{A}}$. Empty circles represent {\bf rows} of
  $\bar{\mathbf{X}}$ when $N$ is finite (in the ideal case,
  $\bar{\mathbf{A}} = \bar{\mathbf{X}}$). Projections of
  $\bar{\mathbf{A}}_w$'s (resp.~$\bar{\mathbf{X}}_w$'s) along a random
  isotropic direction $\mathbf{d}$ can be used to identify novel
  words.}
\label{fig:separable_and_extreme}
\vglue -3ex
\end{figure}
Figure~\ref{fig:separable_and_extreme} shows an example of a separable
$\bm\beta$ with $K=3$ topics. Words $1$ and $2$ are novel to topic
$1$, words $3$ and $4$ to topic $2$, and word $5$ to topic $3$. Other
words that appear in multiple topics are called non-novel words (e.g.,
word $6$). Identifying the novel words for $K$ distinct topics is the
key step of our proposed approach.

We note that separability has been empirically observed to be
approximately satisfied by topic estimates produced by Variational-Bayes and
MCMC based algorithms \cite{Arora2:ref,Ding14:ref,Awasthi15VB:ref}.
%
%Yet it might appear to be restrictive and is only an assumption of
%convenience.
%
More fundamentally, in very recent work \cite{Ding15HighProb:ref}, it
has been shown that topic separability is an inevitable consequence of
having a relatively small number of topics in a very large vocabulary
(high-dimensionality).
In particular, when the $K$ columns (topics) of $\bm\beta$ are
independently sampled from a Dirichlet distribution (on a
$(W-1)$-dimensional probability simplex), the resulting topic matrix
$\bm\beta$ will be (approximately) separable with probability tending
to $1$ as $W$ scales to infinity sufficiently faster than $K$. A
Dirichlet prior on $\bm\beta$ is widely-used in smoothed settings of
topic modeling \cite{Blei2012Review:ref}.
%\cite{LDA:ref}.
% high probability if the topics are iid sampled from a Dirichlet distribution and $W$ scales sufficiently faster than $K$ \cite{Ding15HighProb:ref}. 
%

As we will discuss next in Sec.~\ref{sec:subsection:geometric-RP}, the
topic separability property combined with additional conditions on the
second-order statistics of the mixing weights leads to an intuitively
appealing geometric property that can be exploited to develop a
provably consistent and efficient topic estimation algorithm.
%
%We note that the separability first occurs in early work on hyperspecral imaging analysis \cite{boardman1993automating:ref}, and has been introduced in \cite{Donhunique:ref} to guarantee the uniqueness of NMF \cite{Gillis2014:ref}.  
%
%%%%%%%%%%%%%%%%%%%%%%%%%%%%%%%%
\subsection{Conditions on the Topic Mixing Weights}
\label{sec:subsection:theta}
\begin{figure*}[!t]
\centering
{
\begin{tabular}{ccccc}
$
\begin{pmatrix}
1&0&0\\
0&1&0\\
0&0&1\\
0&0&1\\
&\ldots &
\end{pmatrix}
$
& 
$ \begin{pmatrix}
\leftarrow & \bm{\theta}_1& \rightarrow \\
\leftarrow & \bm{\theta}_2 & \rightarrow \\
\leftarrow & 0.5\bm{\theta}_1 +0.5\bm{\theta}_2& \rightarrow 
\end{pmatrix} $
&
=
&
$ \begin{pmatrix}
1&0&0\\
0&1&0\\
0&0&1\\
0.5&0.5&0\\
&\ldots &
\end{pmatrix} $
&
$ \begin{pmatrix}
\leftarrow & \bm{\theta}_1& \rightarrow \\
\leftarrow & \bm{\theta}_2 & \rightarrow \\
\leftarrow & 0.5\bm{\theta}_1 +0.5\bm{\theta}_2& \rightarrow 
\end{pmatrix} $ \\

$\bm{\beta}^{(1)} $& $\bm{\theta}$ & & $\bm{\beta}^{(2)} $ & $\bm{\theta}$
\end{tabular}
}
\vglue -0ex
\caption{Example showing that topic separability {\bf alone} does not
  guarantee a unique solution to the problem of estimating $\bm\beta$
  from $\mathbf{X}$.  Here, $\bm \beta_1 \bm \theta = \bm \beta_2 \bm
  \theta = \mathbf{A}$ is a document distribution matrix that is
  consistent with two different topic matrices $\bm\beta^{(1)}$ and
  $\bm\beta^{(2)}$ that are both separable. }
\label{fig:simplicial_counterexample}
\vglue -1ex
\end{figure*}
Topic separability alone does not guarantee that there will be a
unique $\bm\beta$ that is consistent with all the observations
$\mathbf{X}$. This is illustrated in
Fig.~\ref{fig:simplicial_counterexample} \cite{Ding13b:ref}.
Therefore, in an effort to develop provably consistent topic
estimation algorithms, a number of different conditions have been
imposed on the topic mixing weights $\bm\theta$ in the literature
\cite{A12:ref,Arora2:ref,DDP:ref,Ding14:ref,Kumar13:ref}. Complementing
the work in \cite{Ding13b:ref} which identifies necessary and
sufficient conditions for consistent {\it detection} of novel words,
in this paper we identify necessary and sufficient conditions for
consistent {\it estimation} of a separable topic matrix.
Our necessity results are {\it information-theoretic and
  algorithm-independent} in nature, meaning that they are independent
of any specific statistics of the observations and the algorithms
used.
%
%We note that 
The novel words and the topics can only be identified up to a
permutation and this is accounted for in our results.
% this section we first identify the necessary and sufficient
% conditions for $(1)$ novel word discovery and $(2)$ separable topic
% estimation.

Let $\mathbf{a} := \mathbb{E}(\bm\theta^{m})$ and $\mathbf{R} :=
\mathbb{E}(\bm\theta^{m}\bm\theta^{m\top})$ be the $K\times 1$
expectation vector and the $K\times K$ correlation matrix of the
weight prior $\Pr(\bm\alpha)$. Without loss of generality, we can
assume that the elements of $\mathbf{a}$ are strictly positive since
otherwise some topic(s) will not appear in the corpus. A key quantity
is $\bar{\mathbf{R}} :=
\diag(\mathbf{a})^{-1}\mathbf{R}\diag(\mathbf{a})^{-1}$ which may be
viewed as a ``normalized'' second-moment matrix of the weight
vector. The following conditions are central to our results.
% play a central role in our results.
%
%%%%%%%%%%%%%%%%%%%%%%%%%%%%%%%%%
\begin{condition}
\label{def:simplicial} {\bf(Simplicial Condition)}
A matrix $\mathbf{B}$ is (row-wise) $\gamma_s$-simplicial if any
row-vector of $\mathbf{B}$ is at a Euclidean distance of at least
$\gamma_s >0$ from the convex hull of the remaining row-vectors.
A topic model is $\gamma_s$-simplicial if its normalized second-moment
$\bar{{\mathbf R}}$ is $\gamma_s$-simplicial.
\end{condition}
%
%%%%%%%%%%%%%%%%%%%%%%%%%%%%%%%%%
%
\begin{condition}
\label{def:affine-independent}
{\bf(Affine-Independence)} A matrix $\mathbf{B}$ is (row-wise)
$\gamma_a$-affine-independent if
%
%$\Vert {\bm\lambda}^{\mathsf{T}} \mathbf{B} \Vert_2 \geq \gamma_a
%\Vert \bm\lambda \Vert_2 > 0$, 
%
$ \min_{\bm\lambda} \Vert \sum_{k=1}^{K} \lambda_{k} \mathbf{B}_k
\Vert_2 / \Vert \bm\lambda \Vert_2 \geq \gamma_a >0$, where
$\mathbf{B}_k$ is the $k$-th row of $\mathbf{B}$ and the minimum is
over all $\bm\lambda\in\mathbb{R}^{K}$ such that $\bm\lambda \neq
\mathbf{0}$ and $\sum_{k=1}^{K}\lambda_{k} = 0$.
% the only $\lambda_1,\ldots,\lambda_K\in\mathbb{R}$ that satisfying $\sum \lambda_k = 0$ and $\sum \lambda_k \mathbf{x}_k  = \mathbf{0}$ are $\lambda_1= \ldots = \lambda_K = 0$.
%
A topic model is $\gamma_a$-affine-independent if its normalized
second-moment $\bar{\mathbf{R}}$ is $\gamma_a$-affine-independent.
\end{condition}
\noindent Here, $\gamma_s$ and $\gamma_a$ are called the simplicial
and affine-independence constants respectively. They are condition
numbers which measure the degree to which the conditions that they are
respectively associated with hold. The larger that these condition
numbers are, the easier it is to estimate the topic matrix.
Going forward, we will say that a matrix is simplicial (resp.~affine
independent) if it is $\gamma_s$-simplicial
(resp.~$\gamma_a$-affine-independent) for some $\gamma_s > 0$
(resp.~$\gamma_a > 0$).
The simplicial condition was first proposed in \cite{A12:ref} and then
further investigated in \cite{Ding13b:ref}.
% Geometrically, it requires the $K$ rows of $\bar{\mathbf{R}}$ to be extreme points of the convex hull they themselves formed. 
%
This paper is the first to identify affine-independence
%is first identified in this paper 
as both {\it necessary and sufficient} for consistent separable topic
estimation.
Before we discuss their geometric implications, we point out that
affine-independence is stronger than the simplicial condition:
\begin{proposition}
\label{prop:simplicial-vs-affine}
$\bar{\mathbf{R}}$ is $\gamma_a$-affine-independent $\Rightarrow$
$\bar{\mathbf{R}}$ is at least $\gamma_a$-simplicial. The reverse
implication is false in general.
\end{proposition}
{\noindent\bf The Simplicial Condition is both Necessary and
  Sufficient for Novel Word Detection:} We first focus on detecting
all the novel words of the $K$ distinct topics. For this task, the
simplicial condition is an algorithm-independent,
information-theoretic necessary condition. Formally,
\begin{lemma}
\label{lem:simplicial-necessary}{(Simplicial Condition is Necessary
  for Novel Word Detection \cite[Lemma~1]{Ding13b:ref})}
Let $\bm\beta$ be separable and $W > K$. If there exists an algorithm
that can consistently identify all novel words of all $K$ topics from
$\mathbf{X}$, then $\bar{\mathbf{R}}$ is simplicial.
\end{lemma}
The key insight behind this result is that when $\bar{\mathbf{R}}$ is
non-simplicial, we can construct two distinct separable topic matrices
with different sets of novel words which induce the same distribution
on the empirical observations $\mathbf{X}$.
Geometrically, the simplicial condition guarantees that the $K$ rows
of $\bar{\mathbf{R}}$ will be extreme points of the convex hull that
they themselves form. Therefore, if $\bar{\mathbf{R}}$ is not
simplicial, there will exist at least one redundant topic which is
just a convex combination of the other topics.

It turns out that $\bar{\mathbf{R}}$ being simplicial is also
sufficient for consistent novel word detection. This is a direct
consequence of the consistency guarantees of our approach as outlined
in Theorem~\ref{thm:novel-word-detection}.

{\noindent\bf Affine-Independence is Necessary and Sufficient for
  Separable Topic Estimation:} We now focus on estimating a separable
topic matrix $\bm\beta$, which is a stronger requirement than
detecting novel words. It naturally requires conditions that are
stronger than the simplicial condition. Affine-independence turns out
to be an algorithm-independent, information-theoretic necessary
condition. Formally,
\begin{lemma}
\label{lem:affine-necessary}{(Affine-Independence is Necessary for Separable Topic Estimation)}
Let $\bm\beta$ be separable with $W \geq 2+K$. If there exists an
algorithm that can consistently estimate $\bm\beta$ from $\mathbf{X}$,
then its normalized second-moment $\bar{\mathbf{R}}$ is
affine-independent.
\end{lemma}
Similar to Lemma~\ref{lem:simplicial-necessary}, if $\bar{\mathbf{R}}$
is not affine-independent, we can construct two distinct and separable
topic matrices that induce the same distribution on the observation
which makes consistent topic estimation impossible.
Geometrically, every point in a convex set can be decomposed {\it
  uniquely} as a convex combination of its extreme points, if, and
only if, the extreme points are affine-independent.
Hence, if $\bar{\mathbf{R}}$ is not affine-independent, a non-novel
word can be assigned to different subsets of topics.
%
% decomposing a point in a convex hull as convex combination of its extreme points is not unique unless the extreme points are affine-independent. 
%
%The proof of this result appears in the appendix.
%

The sufficiency of the affine-independence condition in separable
topic estimation is again a direct consequence of the consistency
guarantees of our approach as in
Theorems~\ref{thm:novel-word-detection}
and~\ref{thm:topic-estimation}.  We note that since
affine-independence implies the simplicial condition
(Proposition~\ref{prop:simplicial-vs-affine}), affine-independence is
sufficient for novel word detection as well. 

{\noindent\bf Connection to Other Conditions on the Mixing Weights:}
We briefly discuss other conditions on the mixing weights $\bm\theta$
that have been exploited in the literature. In
\cite{Arora2:ref,Kumar13:ref}, $\mathbf{R}$ (equivalently
$\bar{\mathbf{R}}$) is assumed to have full-rank (with minimum
eigenvalue $\gamma_r>0$). In \cite{DDP:ref}, $\bar{\mathbf{R}}$ is
assumed to be diagonal-dominant, i.e., $\forall i,j, i\neq j,
\bar{\mathbf{R}}_{i,i} - \bar{\mathbf{R}}_{i,j} \geq \gamma_d
>0$. They are both sufficient conditions for detecting all the novel
words of all distinct topics. The constants $\gamma_r$ and $\gamma_d$
are condition numbers which measure the degree to which the full-rank
and diagonal-dominance conditions hold respectively. They are
counterparts of $\gamma_s$ and $\gamma_a$ and like them, the larger
they are, the easier it is to consistently detect the novel words and
estimate $\bm\beta$.
The relationships between these conditions are summarized in
Proposition~\ref{prop:all-conditions} and illustrated in
Fig.~\ref{fig:conditions_theta}.
\begin{figure}[!htb]
\centering
\includegraphics[width=2.5in]{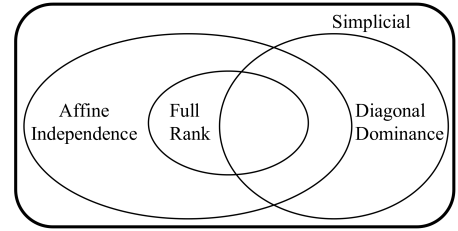}
%
% where an .eps filename suffix will be assumed under latex, 
% and a .pdf suffix will be assumed for pdflatex; or what has been declared
% via \DeclareGraphicsExtensions.
%
\vglue -1ex
\caption{Relationships between Simplicial, Affine-Independence, Full
  Rank, and Diagonal Dominance conditions on the normalized
  second-moment $\bar{\mathbf{R}}$.}
\label{fig:conditions_theta}
\vglue -3ex
\end{figure}
%
%
%%%%%%%%%%%%%%%%%%%%%%%%%%%%%%
\begin{proposition}
\label{prop:all-conditions}
Let $\bar{\mathbf{R}}$ be the normalized second-moment of the topic
prior. Then,
\begin{enumerate}
\item $\bar{\mathbf{R}}$ is full rank with minimum eigenvalue
  $\gamma_r$ $\Rightarrow$ $\bar{\mathbf{R}}$ is at least
  $\gamma_r$-affine-independent $\Rightarrow$ $\bar{\mathbf{R}}$ is at
  least $\gamma_r$-simplicial.
\item $\bar{\mathbf{R}}$ is $\gamma_d$-\-diagonal-\-dominant
  $\Rightarrow$ $\bar{\mathbf{R}}$ is at least $\gamma_d$-simplicial.
\item $\bar{\mathbf{R}}$ being diagonal-dominant neither implies nor
  is implied by $\bar{\mathbf{R}}$ being affine-independent (or
  full-rank).
\end{enumerate}
\end{proposition}
We note that in our earlier work \cite{Ding14:ref}, the provable
guarantees for estimating the separable topic matrix require
$\bar{\mathbf{R}}$ to have full rank. The analysis in this paper
provably extends the guarantees to the affine-independence condition.
%
%
%%%%%%%%%%%%%%%%%%%%%%%%%%%%
\subsection{Geometric Implications and Random Projections Based Algorithm}
\label{sec:subsection:geometric-RP}
We now demonstrate the geometric implications of topic separability 
%on $\bm\beta$ 
combined with the simplicial/ affine-independence condition on the
topic mixing weights. To highlight the key ideas we focus on the ideal
case where $N = \infty$. Then, the empirical document word-frequency
matrix $\mathbf{X} = \mathbf{A} = \bm\beta \bm\theta$.
%we first consider the ideal case where $N\rightarrow \infty$. In this
%limit, by suitably normalizing columns of $\mathbf{X}$ to convert them
%to empirical document word-frequencies instead of counts, it converge
%to the ideal document distributions $\mathbf{A} = \bm\beta \bm\theta$.

{\noindent\bf Novel Words are Extreme Points:} 
To expose the underlying geometry, we normalize the rows of
$\mathbf{A}$ and $\bm\theta$ to obtain row-stochastic matrices
$\bar{\mathbf{A}} := \diag(\mathbf{A}\mathbf{1})^{-1}\mathbf{A}$ and
$\bar{\bm\theta} := \diag(\bm\theta\mathbf{1})^{-1}\bm\theta$. Then
since $\mathbf{A} = \bm\beta \bm\theta$, we have $\bar{\mathbf{A}} =
\bar{\bm\beta} \bar{\mathbf{\theta}}$ where $\bar{\bm\beta} :=
\diag(\mathbf{A}\mathbf{1})^{-1}\bm\beta \diag(\bm\theta\mathbf{1})$
is a row-normalized ``topic matrix'' which is both row-stochastic and
separable with the same sets of novel words as $\bm\beta$.

Now consider the row vectors of $\bar{\mathbf{A}}$ and
$\bar{\bm\theta}$.
%
% First, it can be shown that if $\bar{{\mathbf{R}}}$ is simplicial
%(cf.~Condition~\ref{def:simplicial}) then, with probability one, no
%row of $\bar{\bm\theta}$ will be in the convex hull of the others
%\cite{Ding13b:ref}.
{\color{edit}First, it can be shown that if $\bar{{\mathbf{R}}}$ is simplicial
(cf.~Condition~\ref{def:simplicial}) then, with high probability, no
row of $\bar{\bm\theta}$ will be in the convex hull of the others (see Appendix~\ref{sec:append:thetasimplicial}).}
Next, the separability property ensures that if $w$ is a novel word of
topic $k$, then $\bar{\beta}_{wk}=1$ and $\bar{\beta}_{wj}=0\ \forall
j\neq k$ so that $\bar{\mathbf{A}}_{w} = \bar{\bm\theta}_k$.
Revisiting the example in Fig.~\ref{fig:separable_and_extreme}, the
rows of $\bar{\mathbf{A}}$ which correspond to novel words, e.g.,
words $1$ through $5$, are all row-vectors of $\bar{\bm\theta}$ and
together form a convex hull of $K$ extreme points. For example,
$\bar{\mathbf{A}}_{1} = \bar{\mathbf{A}}_{2} = \bar{\bm\theta}_1$ and
$\bar{\mathbf{A}}_{3} = \bar{\mathbf{A}}_{4} = \bar{\bm\theta}_2$.
If, however, $w$ is a non-novel word, then $\bar{\mathbf{A}}_w =
\sum_{k}\bar{\beta}_{wk}\bar{\bm\theta}_{k}$ lives inside the convex
hull of the rows of $\bar{\bm\theta}$.
In Fig.~\ref{fig:separable_and_extreme}, row $\bar{\mathbf{A}}_{6}$
which corresponds to non-novel word $6$, is inside the convex hull of
$\bar{\bm\theta}_1, \bar{\bm\theta}_2, \bar{\bm\theta}_3$.
In summary, the novel words can be detected as extreme points of all
the row-vectors of $\bar{\mathbf{A}}$. Also, multiple novel words of
the same topic correspond to the same extreme point (e.g.,
$\bar{\mathbf{A}}_1 = \bar{\mathbf{A}}_2 =
\bar{\bm\theta}_1$). Formally,
%%
%\begin{lemma}
%%
%\label{lem:extreme-A}
%%
%{(Novel Words are Extreme Points)} Let $\bar{\mathbf{R}}$ be
%simplicial and ${\bm \beta}$ be separable. Then, a word $i$ is novel
%if, and only if, the $i$-th row of $\bar{\mathbf{A}}$ is an extreme
%point of the convex hull spanned by all the rows of
%$\bar{\mathbf{A}}$.
%%
%\end{lemma}
%%
%
{\color{edit}
\begin{lemma}
\label{lem:extreme-A}
Let $\bar{\mathbf{R}}$ be $\gamma_s$ simplicial and $\bm\beta$ be separable. Then, with probability at least $1 - 2K\exp(-c_1 M) -\exp(-c_2 M)$, 
% a word $i$ is novel if, and only if, the $i$-th row of $\bar{\mathbf{A}}$ is an extreme point of the convex hull spanned by all the rows of $\bar{\mathbf{A}}$. 
the $i$-th row of $\bar{\mathbf{A}}$ is an extreme point of the convex hull spanned by all the rows of $\bar{\mathbf{A}}$ if, and only if, word $i$ is novel. 
Here the constant $c_1 := \gamma_s^2 a_{\min}^{4}/4\lambda_{\max}$ and $c_2 := \gamma_s^{4} a_{\min}^{4}/2\lambda_{\max}^{2}$. The model parameters are defined as follows. $a_{\min}$ is the minimum element of $\mathbf{a}$. $\lambda_{\max}$ is the maximum singular-value of $\bar{\mathbf{R}}$. 
\end{lemma}
}
To see how identifying novel words can help us estimate $\bm\beta$,
recall that the row-vectors of $\bar{\mathbf{A}}$ corresponding to
novel words coincide with the rows of $\bar{\bm\theta}$. Thus
$\bar{\bm\theta}$ is known once one novel word for each topic is
known. Also, for all words $w$, $\bar{\mathbf{A}}_w =
\sum_{k}\bar{\beta}_{wk}\bar{\bm\theta}_{k}$.  Thus, if we can {\it
  uniquely} decompose $\bar{\mathbf{A}}_w$ as a convex combination of
the extreme points, then the coefficients of the decomposition will
give us
%$\bar{\bm\beta}_{w}$, 
the $w$-th row of $\bar{\bm\beta}$. A unique decomposition exists 
%with certainty
{\color{edit}with high probability} 
when $\bar{\mathbf{R}}$ is affine-independent and can be
found by solving a constrained linear regression problem. This gives
us $\bar{\bm\beta}$. Finally, noting that
$\diag(\mathbf{A}\mathbf{1})\bar{\bm\beta} = \bm\beta
\diag(\bm\theta\mathbf{1})$, $\bm\beta$ can be recovered by suitably
renormalizing rows and then columns of $\bar{\bm\beta}$. To sum up,
%%
%\begin{lemma}
%%
%\label{lem:topic-regression}
%%
%Let $\mathbf{A}$ and one novel word for each distinct topic be
%given. If $\bar{\mathbf{R}}$ is affine-independent, then $\bm\beta$
%can be recovered uniquely via constrained linear regression.
%%
%\end{lemma}
%%
{\color{edit}
\begin{lemma}
%
%Let $\mathbf{A}$ and one novel words for each distinct topic be given. 
%
Let $\mathbf{A}$ and one novel word per distinct topic be given.
If $\bar{\mathbf{R}}$ is $\gamma_a$ affine-independent, then, with probability at least $1 - 2K\exp(-c_1 M) -\exp(-c_2 M)$, $\bm\beta$ can be recovered uniquely via constrained linear regression. Here the constant $c_1 := \gamma_a^2 a_{\min}^{4}/4\lambda_{\max}$ and $c_2 := \gamma_a^{4} a_{\min}^{4}/2\lambda_{\max}^{2}$. The model parameters are defined as follows. $a_{\min}$ is the minimum element of $\mathbf{a}$. $\lambda_{\max}$ is the maximum singular-value of $\bar{\mathbf{R}}$.
\label{lem:topic-regression}
\end{lemma}
}
Lemmas~\ref{lem:extreme-A} and~\ref{lem:topic-regression} together
provide a geometric approach for learning $\bm\beta$ from $\mathbf{A}$
(equivalently $\bar{\mathbf{A}}$):
$(1)$ Find extreme points of rows of $\bar{\mathbf{A}}$. Cluster the
rows of $\bar{\mathbf{A}}$ that correspond to the same extreme point
into the same group.
$(2)$ Express the remaining rows of $\bar{\mathbf{A}}$ as convex
combinations of the $K$ distinct extreme points.
$(3)$ Renormalize $\bar{\bm\beta}$ to obtain $\bm\beta$.

{\noindent\bf Detecting Extreme Points using Random Projections:} A
key contribution of our approach is an efficient random projections
based algorithm to detect novel words as extreme points.
The idea is illustrated in Fig.~\ref{fig:separable_and_extreme}: if we
project every point of a convex body onto an isotropically distributed
random direction $\mathbf{d}$, the maximum (or minimum) projection
value must correspond to one of the extreme points with probability
$1$.
On the other hand, the non-novel words will not have the maximum
projection value along any random direction.
Therefore, by repeatedly projecting all the points onto a few
isotropically distributed random directions, we can detect all the
extreme points with very high probability as the number of random
directions increase. An explicit bound on the number of projections
needed appears in Theorem~\ref{thm:novel-word-detection}.
% In addition to its efficiency, it turns out that this random projection based approach is naturally amenable to online or distributed settings in which the documents are observed in a stream or stored in distributed servers. 
% 
% {\noindent\bf A Provably Consistent Approach for Separable NMF:} 
% As a byproduct, the approach discussed above provides an alternative to the Non-negative Matrix Factorization with a separable structure. 
% Meanwhile, the affine-independence (simplicial) condition is also the necessary and sufficient condition for Separable NMF.
% \cite{Donhunique:ref}

{\noindent\bf Finite $N$ in Practice:}
The geometric intuition discussed above was based on the row-vectors
of $\bar{\mathbf{A}}$. When $N = \infty$, $\bar{\mathbf{A}} =
\bar{\mathbf{X}}$ the matrix of row-normalized empirical
word-frequencies of all documents. If $N$ is finite but very large,
$\bar{\mathbf{A}}$ can be well-approximated by $\bar{\mathbf{X}}$
thanks to the law of large numbers.
However, in real-word text corpora, $N\ll W$ (e.g., $N = 298$ while $W
= 14,943$ in the NYT dataset).
Therefore, the row-vectors of $\bar{\mathbf{X}}$ are significantly
perturbed away from the ideal rows of $\bar{\mathbf{A}}$ as
illustrated in Fig.~\ref{fig:separable_and_extreme}. We discuss the
effect of small $N$ and how we address the accompanying issues next.
%
% Therefore, unless the number of word per document, i.e., $N\rightarrow\infty$ for each document. However, in typical text corpus, $N$ is limited and small compared to the size of vocabulary $W$. For instance in the NYT dataset \cite{UCIdataset:ref} to be considered in Sec.~\ref{section:experiments}, $M = 300,000$, $W = 14,943$ and on average $N = 298$ per document. We discuss its consequences and how we deal with these issues in the next section.
%
%
%%%%%%%%%%%%%%%%%%%%%%%%%%%%%%%%
\section{Topic Geometry with Finite Samples: Word Co-occurrence Matrix
  Representation, Solid Angle, and Random Projections based approach}
\label{section:finite_geometry}
The extreme point geometry sketched in
Sec.~\ref{sec:subsection:geometric-RP} is perturbed when $N$ is small
as highlighted in Fig.~\ref{fig:separable_and_extreme}. Specifically,
the rows of the empirical word-frequency matrix $\mathbf{X}$ deviate
from the rows of $\mathbf{A}$. This creates several problems: $(1)$
points in the convex hull corresponding to non-novel words may also
become ``outlier'' extreme points (e.g., $\bar{\mathbf{X}}_6$ in
Fig.~\ref{fig:separable_and_extreme}); $(2)$ some extreme points that
correspond to novel words may no longer be extreme (e.g.,
$\bar{\mathbf{X}}_3$ in Fig.~\ref{fig:separable_and_extreme}); $(3)$
multiple novel words corresponding to the same extreme point may
become multiple distinct extreme points (e.g., $\bar{\mathbf{X}}_1$
and $\bar{\mathbf{X}}_2$ in Fig.~\ref{fig:separable_and_extreme}).
%
%These issues, however, do not vanish when $N$ is small regardless of
%how $M$ scales.
%
Unfortunately, these issues do not vanish as $M$ increases with $N$
fixed -- a regime which captures the characteristics of typical
benchmark datasets -- because the dimensionality of the rows (equal to
$M$) also increases. There is no ``averaging'' effect to smoothen-out
the sampling noise.

Our solution is to seek a new representation, a statistic of
$\mathbf{X}$, which can not only smoothen out the sampling noise of
individual documents, but also preserve the same extreme point
geometry induced by the separability and affine independence
conditions.
%
%i.e., some statistics of the empirical observation $\mathbf{X}$, that
%can preserve the same extreme point geometry and can be consistently
%estimated from observation by exploiting the fact that $M$ is very
%large.
%
%In another word, we would like the new representation to smooth out
%the sampling noise of individual documents.
%
In addition, we also develop an extreme point robustness measure that
naturally arises within our random projections based framework. This
robustness measure can be used to detect and exclude the ``outlier''
extreme points.
%
% We therefore seek for a new representation, i.e., some statistics of the empirical observation $\mathbf{X}$, which can preserve the same extreme point geometry and can be consistently estimated. We focus on the case where $M$ is large while $N$ is fixed as motivated by the benchmark datasets. 
%
%
\subsection{Normalized Word Co-occurrence Matrix Representation}
\label{sec:subsection:co-occurrence-mat}
We construct a suitably normalized word co-occurrence matrix from
$\mathbf{X}$ as our new representation. The co-occurrence matrix
converges almost surely to an ideal statistic as $M\rightarrow \infty$
for any fixed $N \geq 2$. Simultaneously, in the asymptotic limit, the
original novel words continue to correspond to extreme points in the
new representation and overall extreme point geometry is preserved.

The new representation is (conceptually) constructed as follows. First
randomly divide all the words in each document into two equal-sized
independent halves and obtain two $W\times K$ empirical word-frequency
matrices $\mathbf{X}$ and $\mathbf{X}^{\prime}$ each containing $N/2$
words. Then normalize their rows like in
Sec.~\ref{sec:subsection:geometric-RP} to obtain $\bar{\mathbf{X}}$
and $\bar{\mathbf{X}}^{\prime}$ which are row-stochastic. The
empirical word co-occurrence matrix of size $W\times W$ is then given
by
\begin{equation}\label{eqa:word-co-occurrence-def}
\widehat{\mathbf{E}} := M \bar{\mathbf{X}}^{\prime} \bar{\mathbf{X}}^{\top}
\end{equation}

We note that in our random projection based approach,
$\widehat{\mathbf{E}}$ is not {\it explicitly} constructed by
multiplying $\bar{\mathbf{X}}^{\prime}$ and $\bar{\mathbf{X}}$.
Instead, we keep $\bar{\mathbf{X}}^{\prime}$ and $\bar{\mathbf{X}}$
and exploit their sparsity properties to reduce the computational
complexity of all subsequent processing.
% instead exploit the sparsity in $\mathbf{X}$ to reduce the computation complexity.

{\noindent\bf Asymptotic Consistency:} The first nice property of the
word co-occurrence representation is its asymptotic consistency when
$N$ is fixed. As the number of documents $M\rightarrow\infty$, the
empirical $\widehat{\mathbf{E}}$ converges, almost surely, to an ideal
word co-occurrence matrix $\mathbf{E}$ of size $W\times W$. Formally,
\begin{lemma}\label{lem:second-order-convergence}(\cite[Lemma~2]{topicRank2:ref})
Let $\widehat{\mathbf{E}}$ be the empirical word co-occurrence matrix
defined in Eq.~\eqref{eqa:word-co-occurrence-def}. Then,
\begin{equation}
\label{eqa:word-co-occurrence-limit}
\widehat{\mathbf{E}} \xrightarrow[\mbox{almost surely}]{M
  \rightarrow\infty} \bar{\bm{\beta}}\bar{\mathbf{R}}
\bar{\bm{\beta}}^{\top} =: \mathbf{E}
\end{equation}
where 
$\bar{\bm{\beta}} :=
\diag^{-1}(\bm{\beta}\mathbf{a})\bm{\beta}\diag(\mathbf{a})$ and
$\bar{\mathbf{R}} :=
\diag^{-1}(\mathbf{a})\mathbf{R}\diag^{-1}(\mathbf{a})$. 
%is the normalized second-moment of the topic prior.
% $\bar{\mathbf{R}} = \diag^{-1}(\mathbf{a})\mathbf{R}\diag^{-1}(\mathbf{a})$, 
% and $\mathbf{a}$ and $\mathbf{R}$ are, respectively, the $K\times 1$ expectation and $K\times K$ correlation matrix of the topic prior. 
Furthermore, if $\eta := \min_{1\leq i\leq W}(\bm{\beta}\mathbf{a})_i
>0$, then
$\Pr(\Vert \widehat{\mathbf{E}} - \mathbf{E}\Vert_{\infty} \geq
\epsilon) \leq$ $8W^2\exp(-\epsilon^2 \eta^4 MN /20)$.
%\hbox{$\Pr(\Vert \widehat{\mathbf{E}} - \mathbf{E}\Vert_{\infty} \geq
%  \epsilon) \leq 8W^2\exp(-\epsilon^2 \eta^4 MN /20) $}
%
\end{lemma}
Here $\bar{\mathbf{R}}$ is the same normalized second-moment of the
topic priors as defined in Sec.~\ref{section:ideal_geometry} and
$\bar{\bm\beta}$ is a row-normalized version of $\bm\beta$. 
We make note of the abuse of notion for $\bar{\bm\beta}$ which was
defined in Sec.~\ref{sec:subsection:geometric-RP}. It can be shown
that the $\bar{\bm\beta}$ defined in
Lemma~\ref{lem:second-order-convergence} is the limit of the one
defined in Sec.~\ref{sec:subsection:geometric-RP} as
$M\rightarrow\infty$.
The convergence result in Lemma~\ref{lem:second-order-convergence}
shows that the word co-occurrence representation $\mathbf{E}$ can be
consistently estimated by $\widehat{\mathbf{E}}$ as $M\rightarrow
\infty$ and the deviation vanishes exponentially in $M$ which is large
in typical benchmark datasets.

{\noindent\bf Novel Words are Extreme Points:} 
Another reason for using this word co-occurrence representation is that
it preserves the extreme point geometry.
Consider the ideal word co-occurrence matrix $\mathbf{E} =
\bar{\bm\beta} (\bar{\mathbf{R}} \bar{\bm{\beta}}^{\top})$. It is
straightforward to show that if $\bar{\bm\beta}$ is separable and
$\bar{\mathbf{R}}$ is simplicial then $(\bar{\mathbf{R}}
\bar{\bm{\beta}}^{\top})$ is also simplicial. Using these facts it is
possible to establish the following counterpart of
Lemma~\ref{lem:extreme-A} for $\mathbf{E}$:
\begin{lemma}
\label{lem:extreme-E}
{(Novel Words are Extreme Points \cite[Lemma~1]{Ding14:ref})}
Let $\bar{\mathbf{R}}$ be simplicial and ${\bm \beta}$ be
separable. Then, a word $i$ is novel if, and only if, the $i$-th row
of $\mathbf{E}$ is an extreme point of the convex hull spanned by all
the rows of $\mathbf{E}$.
\end{lemma}
\noindent In another words, the novel words correspond to the extreme
points of all the row-vectors of the ideal word co-occurrence matrix
$\mathbf{E}$. Consider the example in Fig.~\ref{fig:geometry-word-co}
which is based on the same topic matrix $\bm\beta$ as in
Fig.~\ref{fig:separable_and_extreme}. Here,
$\mathbf{E}_1=\mathbf{E}_2, \mathbf{E}_3=\mathbf{E}_4$, and
$\mathbf{E}_5$ are $K=3$ distinct extreme points of all row-vectors of
$\mathbf{E}$ and $\mathbf{E}_6$, which corresponds to a non-novel
word, is inside the convex hull.
%
%\vglue -3ex
%
\begin{figure}[!hb]
\centering
\vglue -4ex
\includegraphics[width=3.5in]{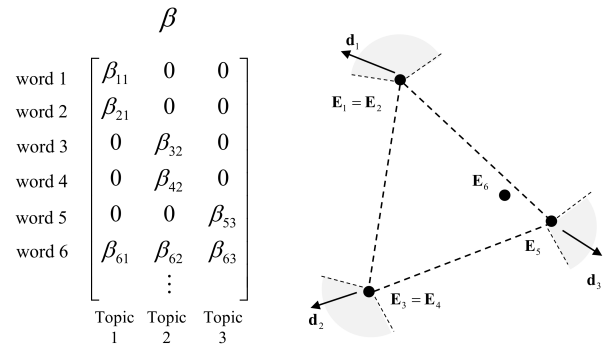}
%
% where an .eps filename suffix will be assumed under latex, 
% and a .pdf suffix will be assumed for pdflatex; or what has been declared
% via \DeclareGraphicsExtensions.
%
\vglue -5ex
\caption{An example of separable topic matrix $\bm\beta$ (left) and
  the underlying geometric structure (right) in the word co-occurrence
  representation. Note: the word ordering is only for visualization
  and has no bearing on separability. The example topic matrix
  $\bm\beta$ is the same as in
  Fig.~\ref{fig:separable_and_extreme}. Solid circles represent the
  {\bf rows} of $\mathbf{E}$. The shaded regions depict the solid
  angles subtended by each extreme
  point. $\mathbf{d}_1,\mathbf{d}_2,\mathbf{d}_3$ are isotropic random
  directions along which each extreme point has maximum projection
  value. They can be used to estimate the solid angles.}
\label{fig:geometry-word-co}
\end{figure}
%
%\vglue -1ex
%

Once the novel words are detected as extreme points, we can follow the
same procedure as in Lemma~\ref{lem:topic-regression} and express each
row $\mathbf{E}_w$ of $\mathbf{E}$ as a unique convex combination of
the $K$ extreme rows of $\mathbf{E}$ or equivalently the rows of
$(\bar{\mathbf{R}} \bar{\bm{\beta}}^{\top})$. The weights of the
convex combination are the $\bar{\beta}_{wk}$'s. We can then apply the
same row and column renormalization to obtain $\bm\beta$. The
following result is the counterpart of
Lemma~\ref{lem:topic-regression} for $\mathbf{E}$:
\begin{lemma}
\label{lem:topic-regression-E}
Let $\mathbf{E}$ and one novel word for each distinct topic be
given. If $\bar{\mathbf{R}}$ is affine-independent, then $\bm\beta$
can be recovered uniquely via constrained linear regression.
\end{lemma}
One can follow the same steps as in the proof of
Lemma~\ref{lem:topic-regression}. The only additional step is to check
that $\bar{\mathbf R} \bar{\bm \beta}^{\top} = \left[
  \bar{\mathbf{R}}, \bar{\mathbf R}\mathbf{B} \right] $ is
affine-independent if $\bar{\mathbf{R}}$ is affine-independent.

We note that the finite sampling noise perturbation
$\widehat{\mathbf{E}}-\mathbf{E}$ is still not $0$ but vanishes as
$M\rightarrow\infty$ (in contrast to the $\bar{\mathbf{X}}$
representation in Sec.~\ref{sec:subsection:geometric-RP}).
However, there is still a possibility of observing ``outlier'' extreme
points if a non-novel word lies on the facet of the convex hull of the
rows of $\mathbf{E}$.
We next introduce an extreme point robustness measure based on a
certain {\it solid angle} that naturally arises in our random
projections based approach, and discuss how it can be used to detect
and distinguish between ``true'' novel words and such ``outlier''
extreme points.
% It turns out that the random projection based algorithm is robust in the presence of small perturbation. 
%
%
%%%%%%%%%%%%%%%%%%%%%%%%%%%%%%
\subsection{Solid Angle Extreme Point Robustness Measure}
\label{sec:subsection:solid-angle}
To handle the impact of a small but nonzero perturbation $\Vert
\widehat{\mathbf{E}} - \mathbf{E}\Vert_{\infty}$, we develop an
extreme point ``robustness'' measure. This is necessary for not only
applying our approach to real-world data but also to establish finite
sample complexity bounds.
Intuitively, a robustness measure should be able to distinguish
between the ``true'' extreme points (row vectors that are novel words)
and the ``outlier'' extreme points (row vectors of non-novel words
that become extreme points due to the nonzero perturbation).
Towards this goal, we leverage a key geometric quantity, namely, the
{\it Normalized Solid Angle} subtended by the convex hull of the rows
of $\mathbf{E}$ at an extreme point.
To visualize this quantity, we revisit our running example in
Fig.~\ref{fig:geometry-word-co} and indicate the solid angles attached
to each extreme point by the shaded regions.
It turns out that this geometric quantity naturally arises in the
context of random projections that was discussed earlier. To see this
connection, in Fig.~\ref{fig:geometry-word-co} observe that the shaded
region attached to any extreme point coincides precisely with the set
of directions along which its projection is larger (taking sign into
account) than that of any other point (whether extreme or not). For
example, in Fig.~\ref{fig:geometry-word-co} the projection of
$\mathbf{E}_1= \mathbf{E}_2$ along $\mathbf{d}_1$ is larger than that
of any other point. Thus, the solid angle attached to a point
$\mathbf{E}_i$ (whether extreme or not) can be formally defined as the
set of directions $\{\mathbf{d}: \forall j: {\mathbf E}_j \neq
{\mathbf E}_1, \langle {\mathbf E}_i, {\mathbf d} \rangle > \langle
{\mathbf E}_j, {\mathbf d} \rangle\}$. This set is nonempty only for
extreme points.
The solid angle defined above is a set. To derive a scalar robustness
measure from this set and tie it to the idea of random projections, we
adopt a statistical perspective and define the normalized solid angle
of a point as the {\it probability} that the point will have the
maximum projection value along an isotropically distributed random
direction. Concretely, for the $i$-th word (row vector), the
normalized solid angle $q_i$ is defined as
\begin{equation}
\label{eqa:solidangle-ideal}
q_i := \Pr(\forall j: {\mathbf E}_j \neq {\mathbf E}_i, \langle
{\mathbf E}_i, {\mathbf d} \rangle > \langle {\mathbf E}_j, {\mathbf
  d} \rangle)
\end{equation}
where ${\mathbf d}$ is drawn from an isotropic distribution in
$\mathbb{R}^{W}$ such as the spherical Gaussian.
The condition $ {\mathbf E}_i \neq {\mathbf E}_j$ in
Eq.~\eqref{eqa:solidangle-ideal} is introduced to exclude the multiple
novel words of the same topic that correspond to the same extreme
point. For instance, in Fig.~\ref{fig:geometry-word-co} $\mathbf{E}_1
= \mathbf{E}_2$, Hence, for $q_1$, $j=2$ is excluded.
To make it practical to handle finite sample estimation noise we
replace the condition $\mathbf{E}_j\neq \mathbf{E}_i$ by the condition
$\Vert\mathbf{E}_i - \mathbf{E}_j\Vert\geq \zeta$ for some suitably
defined $\zeta$.

As illustrated in Fig.~\ref{fig:geometry-word-co}, the solid angle for
all the extreme points are strictly positive given $\bar{\mathbf{R}}$
is $\gamma_s$-simplicial. On the other hand, for $i$ that is
non-novel, the corresponding solid angle $q_i$ is zero by
definition. Hence the extreme point geometry in
Lemma~\ref{lem:extreme-E} can be re-expressed in term of solid angles
as follows:
\begin{lemma}
\label{lem:extreme-solid-angle}
{(Novel Words have Positive Solid Angles)}
Let $\bar{\mathbf{R}}$ be simplicial and $\bm{\beta}$ be
separable. Then, word $i$ is a novel word if, and only if, $q_i > 0$.
\end{lemma}
We denote the smallest solid angle among the $K$ distinct extreme
points by $q_{\wedge} > 0$. This is a robust condition number of the
convex hull formed by the rows of $\mathbf{E}$ and is related to the
simplicial constant $\gamma_s$ of $\bar{\mathbf{R}}$.

In a real-world dataset we have access to only an empirical estimate
$\widehat{\mathbf{E}}$ of the ideal word co-occurrence matrix
$\mathbf{E}$. If we replace $\mathbf{E}$
%the ideal word co-occurrence matrix,
with $\widehat{\mathbf{E}}$, 
%its empirical version,
then the resulting empirical solid angle estimate $\widehat{q}_i$ will
be very close to the ideal $q_i$ if $\widehat{\mathbf{E}}$ is close
enough to $\mathbf{E}$.
Then, the solid angles of ``outlier'' extreme points will be close to
$0$ while they will be bounded away from zero for the ``true'' extreme
points.
One can then hope to correctly identify all $K$ extreme points by {\it
  rank-ordering} all empirical solid angle estimates and selecting the
$K$ distinct row-vectors that have the largest solid angles. This
forms the basis of our proposed algorithm.
The problem now boils down to efficiently estimating the solid angles
and establishing the asymptotic convergence of the estimates as $M
\rightarrow \infty$.
We next discuss how random projections can be used to achieve these
goals.
%
%
%%%%%%%%%%%%%%%%%%%%%%%%%
\subsection{Efficient Solid Angle Estimation via Random Projections}
\label{sec:subsection:solid-angle-rp}
The definition of the normalized solid angle in
Eq.~\eqref{eqa:solidangle-ideal} motivates an efficient algorithm
based on {\it random projections} to estimate it. For convenience, we
first rewrite Eq.~\eqref{eqa:solidangle-ideal} as
\begin{eqnarray}
\label{eqa:solidangle-ideal-2}
q_i &=&\mathbb{E}\Biggl[ \mathbb{I}\lbrace \forall j:
\Vert\mathbf{E}_j - \mathbf{E}_i\Vert\geq \zeta,\ \mathbf{E}_i
\mathbf{d} \geq \mathbf{E}_j \mathbf{d} \rbrace \Biggr]
\end{eqnarray}
and then propose to estimate it by
\begin{multline}
\label{eqa:solid-approx}
\hat{q}_i = \frac{1}{\numofproj} \sum\limits_{r=1}^{\numofproj}
\mathbb{I}(\forall j: \widehat{E}_{i,i} + \widehat{E}_{j,j} -
2\widehat{E}_{i,j} \geq \zeta/2, \\
 \widehat{\mathbf E}_i {\mathbf d}^{r} > \widehat{\mathbf E}_j
         {\mathbf d}^{r} )
\end{multline}
where $\mathbf{d}^1,\ldots,\mathbf{d}^{P}\in\mathbf{R}^{W\times 1}$
are $P$ iid directions drawn from an isotropic distribution in
$\mathbf{R}^{W}$.
Algorithmically, by Eq.~\eqref{eqa:solid-approx}, we approximate the
solid angle $q_i$ at the $i$-th word (row-vector) by first projecting
all the row-vectors onto $P$ iid isotropic random directions and then
calculating the fraction of times each row-vector achieves the maximum
projection value.
It turns out that the condition $\widehat{E}_{i,i} + \widehat{E}_{j,j}
- 2\widehat{E}_{i,j} \geq \zeta/2 $ is equivalent to
$\Vert\mathbf{E}_i - \mathbf{E}_j\Vert\geq \zeta$ in terms of its
ability to exclude multiple novel words from the same topic and is
adopted for its simplicity.
\footnote{We abuse the symbol $\zeta$ by using it to indicate
  different thresholds in these conditions.}

This procedure of taking random projections followed by calculating
%the frequencies of different words being a maximizer in
the number of times a word is a maximizer via
Eq.~\eqref{eqa:solid-approx} provides a consistent estimate of the
solid angle in Eq.~\eqref{eqa:solidangle-ideal} as $M\rightarrow
\infty$ and the number of projections $P$ increases. The high-level
idea is simple: as $P$ increases, the empirical average in
Eq.~\ref{eqa:solid-approx} converges to the corresponding expectation.
Simultaneously, as $M$ increases, $\widehat{\mathbf{E}}
\xrightarrow{a.s.} \mathbf{E}$.
Overall, the approximation $\widehat{q}_i$ proposed in
Eq~\eqref{eqa:solid-approx} using random projections converges to
$q_i$.
%
% Overall, we can use a few iid random projections to achieve a {\it
%consistent and efficient} estimation of the solid angles.
%

This random projections based approach is also computationally
efficient 
%in approximating the solid angles
for the following reasons.
First, it enables us to avoid the explicit construction of the
$W\times W$ dimensional matrix $\widehat{\mathbf{E}}$:
Recall that each column of $\mathbf{X}$ and $\mathbf{X}^{\prime}$ has
no more than $N\ll W$ nonzero entries. Hence $\mathbf{X}$ and
$\mathbf{X}^{\prime}$ are both sparse. Since
$\widehat{\mathbf{E}}\mathbf{d} = M \bar{\mathbf{X}}^{\prime}
(\bar{\mathbf{X}}^{\top} \mathbf{d})$, the projection can be
calculated using two sparse matrix-vector multiplications.
Second, it turns out that the number of projections $P$ needed to
guarantee consistency is small. In fact in
Theorem~\ref{thm:novel-word-detection} we provide a sufficient upper
bound for $P$ which is a polynomial function of $\log(W)$,
$\log(1/\delta)$ and other model parameters, where $\delta$ is the
probability that the algorithm fails to detect all the distinct novel
words.

{\noindent\bf Parallelization, Distributed and Online Settings:}
Another advantage of the proposed random projections based approach is
that it can be {\it parallelized} and is naturally amenable to {\it
  online} or {\it distributed} settings.
This is based on the following observation that each projection has an
additive structure:
\begin{equation*}
\label{eqa:distributed-solidangle-motivation}
\widehat{\mathbf{E}}\mathbf{d}^{r} = M \bar{\mathbf{X}}^{\prime}
\bar{\mathbf{X}}^{\top} \mathbf{d}^{r} = M \sum_{m=1}^{M}
\bar{\mathbf{X}}^{m\prime} \bar{\mathbf{X}}^{m\top} \mathbf{d}^{r}.
\end{equation*}
The $P$ projections can also be computed independently. Therefore,
\begin{itemize}
\item In a {\it distributed} setting in which the documents are stored
  on distributed servers, we can first share the same random
  directions across servers and then aggregate the projection
  values. The communication cost is only the ``partial'' projection
  values and is therefore insignificant \cite{Ding14:ref} and does not
  scale as the number of observations $N,M$ increases.
\item In an {\it online} setting in which the documents are streamed
  in an online fashion \cite{hoffman2010online}, we only need to keep
  all the projection values and update the projection values (hence
  the empirical solid angle estimates) when new documents arrive.
\end{itemize}
The additive and independent structure guarantees that the statistical
efficiency of these variations are the same as the centralized
``batch'' implementation.
%
% Therefore, all the projection values (hence the novel words) can be estimated without losing statistical accuracy if the documents $m=1,\ldots, M$ are observed in a streaming fashion . 
%Similarly, if the documents are stored on distributed servers, we can share the same random directions across servers and aggregate the projection values  at an insignificant communication cost of transmitting the projection values \cite{Ding14:ref}. 
%
For the rest of this paper, we only focus on the centralized version.
% and establish provable statistical and computation efficiency results. 

%
{\noindent\bf Outline of Overall Approach:}
Our overall approach can be summarized as follows. $(1)$ Estimate the
empirical solid angles using $P$ iid isotropic random directions as in
Eq.~\ref{eqa:solid-approx}. $(2)$ Select the $K$ words with distinct
word co-occurrence patterns (rows) that have the largest empirical
solid angles. $(3)$ Estimate the topic matrix using constrained linear
regression as in Lemma~\ref{lem:topic-regression}. We will discuss the
details of our overall approach in the next section and establish
guarantees for its computational and statistical efficiency.
\section{Algorithm and Analysis}
\label{section:algorithm}
Algorithm~\ref{alg:text:highlevel} describes the main steps of our
overall random projectons based algorithm which we call RP. The two
main steps, novel word detection and topic matrix estimation are
outlined in Algorithms~\ref{alg:text:rp} and ~\ref{alg:text:esttopic1}
respectively.
Algorithm~\ref{alg:text:rp} outlines the random projection and
rank-ordering steps. Algorithm~\ref{alg:text:esttopic1} describes the
constrained linear regression and the renormalization steps in a
combined way.
\begin{algorithm}[!htb]
%
%\caption{Overall-Approach}
\caption{RP}
\label{alg:text:highlevel}
\begin{algorithmic}[1]
\REQUIRE Text documents $\bar{\mathbf X}$, $\bar{\mathbf X}^{\prime}
(W\times M)$; Number of topics $K$; Number of iid random projections
$P$; Tolerance parameters $\zeta,\epsilon >0$.
\ENSURE Estimate of the topic matrix $\widehat{\bm{\beta}} (W\times
K)$.
\STATE Set of Novel Words
$\mathcal{I}\leftarrow$NovelWordDetect($\bar{\mathbf X},\bar{\mathbf
  X}^{\prime}, K, \numofproj, \zeta$)
\STATE $\hat{\bm{\beta}} \leftarrow$EstimateTopics($\mathcal{I},
\bar{\mathbf X},\bar{\mathbf X}^{\prime}, \epsilon$)
\end{algorithmic}
\end{algorithm}
%
%%%%%%%%%%%%%%%%%%%%%%%%%%%%%%%%%%%%%
%
%%%%%%%%%%%%%%%%%%%%%%%%%%%%%%%%%%%%%
\begin{algorithm}[!htb]
\caption{NovelWordDetect (via Random Projections)}
\label{alg:text:rp}
\begin{algorithmic}[1]
\REQUIRE $\bar{\mathbf X},\bar{\mathbf X}^{\prime}$; Number of topics
$K$; Number of projections $P$; Tolerance $\zeta$;
\ENSURE The set of all novel words of $K$ distinct topics
$\mathcal{I}$.
\STATE $\hat{q}_i\leftarrow 0,~\forall i=1,\ldots,W$,
~~$\widehat{\mathbf{E}}\leftarrow M\bar{\mathbf
  X}^{\prime}\bar{\mathbf{X}}^{\top}$.
\FORALL {$r=1,\ldots, P$}
\STATE Sample $\mathbf{d}^{r} \in\mathbb{R}^{W}$ from an isotropic
prior.
\STATE $\mathbf{v} \leftarrow M\bar{\mathbf
  X}^{\prime}\bar{\mathbf{X}}^{\top}\mathbf{d}^{r}$
\STATE $i^{*} \leftarrow \arg\max_{1\leq i\leq W} \mathbf{v}_{i}$,
~~$\hat{q}_{i^*} \leftarrow \hat{q}_{i^*} + 1/P$
\STATE $\hat{J}_{i^*} \leftarrow \{ j : \widehat{E}_{i^*,i^*} +
\widehat{E}_{j,j} - 2\widehat{E}_{i^*,j} \geq \zeta/2 \}$
\FORALL {$k \in \hat{J}^c_{i^*}$}
\STATE $\hat{J}_{k} \leftarrow \{ j : \widehat{E}_{k,k} +
\widehat{E}_{j,j} - 2\widehat{E}_{k,j} \geq \zeta/2 \}$
\IF {$\{\forall j \in \hat{J}_k, v_{k} > v_{j} \}$}
\STATE $\hat{q}_k \leftarrow \hat{q}_k + 1/P$
\ENDIF
\ENDFOR
\ENDFOR
\STATE $\mathcal{I}\leftarrow\emptyset$, $k\leftarrow 0$, $j\leftarrow
1$
\WHILE {$k<K$}
\STATE $i\leftarrow$ index of the $j^{th}$ largest value of
$\{\hat{q}_1,\ldots,\hat{q}_W\}$.
\IF {$\{\forall p \in \mathcal{I}, \widehat{E}_{p,p} +
  \widehat{E}_{i,i} - 2\widehat{E}_{i,p} \geq \zeta/2 \}$}
		\STATE $\mathcal{I} \leftarrow {\mathcal{I}} \cup \{ i
                \}$, $k \leftarrow k + 1$
	\ENDIF
	\STATE $j \leftarrow j + 1$
\ENDWHILE
\STATE {\bf Return} $\mathcal{I}$.
\end{algorithmic}
\end{algorithm}
%
%%%%%%%%%%%%%%%%%%%%%%%%%%%%%%%%%%%%%
%
%%%%%%%%%%%%%%%%%%%%%%%%%%%%%%%%%%%%%
\begin{algorithm}[!htb]
\caption{EstimateTopics}
\label{alg:text:esttopic1}
\begin{algorithmic}[1]
\REQUIRE $\mathcal{I} = \{i_1,\ldots, i_K\}$ set of novel words, one
for each of the $K$ topics; $\widehat{\mathbf{E}}$; precision
parameter $\epsilon$
\ENSURE $\widehat{{\bm \beta}}$, which is the estimate of the ${\bm
  \beta}$ matrix
\STATE $\widehat{\mathbf{E}}^{*}_w =\left[
  \widehat{\mathbf{E}}_{w,i_1}, \ldots, \widehat{\mathbf{E}}_{w,i_K}
  \right]$
\STATE ${\mathbf Y}=(\widehat{\mathbf E}_{i_1}^{*\top}, \ldots,
\widehat{\mathbf E}_{i_K}^{*\top})^{\top}$
\FORALL {$i = 1, \ldots, W$}
\STATE Solve $\mathbf{b}^* := \argmin_{\mathbf{b}}
\Vert\widehat{\mathbf E}^{*}_i - {\mathbf b} {\mathbf Y} \Vert^2$
\STATE subject to $b_j \geq 0, \sum_{j=1}^{K} b_j = 1$
\STATE using precision $\epsilon$ for the stopping-criterion.
\STATE $\widehat{\bm \beta}_i \leftarrow (\frac{1}{M} {\mathbf X}_i
       {\mathbf 1})\mathbf{b}^*$
\ENDFOR 
\STATE $\widehat{\bm\beta} \leftarrow$column normalize $\widehat{\bm
  \beta}$
\end{algorithmic}
\end{algorithm}
%%%%%%%%%%%%%%%%%%%%%%%%%%%%%%%%%%%%%

{\noindent\bf Computational Efficiency:}
We first summarize the computational efficiency of
Algorithm~\ref{alg:text:highlevel}:
\begin{theorem}
\label{thm:computation}
Let the number of novel words for each topic be a constant relative to
$M,W,N$. Then, the running time of Algorithm~\ref{alg:text:highlevel}
is $\mathcal{O}(M N P + WP + W K^3)$.
\end{theorem}
This efficiency is achieved by exploiting the sparsity of $\mathbf{X}$
and the property that there are only a small number of novel words in
a typical vocabulary. A detailed analysis of the computational
complexity is presented in the appendix. Here we point out that in
order to upper bound the computation time of the linear regression in
Algorithm~ \ref{alg:text:esttopic1} we used $\mathbf{O}(W K^3)$ for
$W$ matrix inversions, one for each of the words in the vocabulary.
In practice, a
%standard
gradient descent implementation can be used 
%one can use
for the constrained linear regression which is much more efficient.
%but to establish computation bound for this step is not the main focus
%of this paper.
We also note that these $W$ optimization problems are decoupled
%independent
given the set of detected novel words. Therefore, they can be
parallelized in a straightforward manner \cite{Ding14:ref}.

{\noindent\bf Asymptotic Consistency and Statistical Efficiency:}
We now summarize the asymptotic consistency and sample complexity
bounds for Algorithm~\ref{alg:text:highlevel}. The analysis is a
combination of the consistency of the novel word detection step
(Algorithm~\ref{alg:text:rp}) and the topic estimation step
(Algorithm~\ref{alg:text:esttopic1}). We state the results for both of
these steps.
First, for detecting all the novel words of the $K$ distinct topics,
we have the following result:
\begin{theorem}
\label{thm:novel-word-detection}
Let topic matrix $\bm{\beta}$ be separable and $\bar{\mathbf{R}}$ be
$\gamma$-simplicial. If the projection directions are iid sampled from
any isotropic distribution, then Algorithm~\ref{alg:text:rp} can
identify all the novel words of the $K$ distinct topics as
$M,P\rightarrow \infty$.
Furthermore, $\forall \delta \geq 0$, if 
\begin{equation}
\label{eqa:complexity-bound-1}
M \geq  20 \frac{\log(2W/\delta)}{ N \rho^2 \eta^4} ~\text{and}~ P \geq 8 \frac{\log(2W/\delta)}{q_{\wedge}^2}
\end{equation}
\noindent then Algorithm~\ref{alg:text:rp} fails with probability at
most $\delta$.
The model parameters are defined as follows. $\rho =\min\{
\frac{d}{8}, \frac{\pi d_2 q_{\wedge}}{4 W^{1.5}} \}$ where
$d=(1-b)^{2}\gamma^{2}/\lambda_{\max}$, $d_2 \triangleq (1-b)\gamma$,
$\lambda_{\max}$ is the maximum eigenvalue of $\bar{\mathbf{R}}$,
$b=\max_{j\in\mathcal{C}_0,k} \bar{\beta}_{j,k}$, and
$\mathcal{C}_{0}$ is the set of non-novel words. Finally, $q_{\wedge}$
is the minimum solid angle of the extreme points of the convex hull of
the rows of $\mathbf{E}$.
\end{theorem}
The detailed proof is presented in the appendix.
The results in Eq.~\eqref{eqa:complexity-bound-1} provide a sufficient
finite sample complexity bound for novel word detection. The bound is
{\it polynomial} with respect to $M, W, K, N$, $\log(\delta)$ and
other model parameters. The number of projections $P$ that impacts the
computational complexity scales as $\log(W)/q_{\wedge}^{2}$ in this
sufficient bound where $q_{\wedge}$ can be upper bounded by $1/K$. In
practice, we have found that setting $P = \mathcal{O}(K)$ is a good
choice \cite{Ding14:ref}.

We note that the result in Theorem~\ref{thm:novel-word-detection} only
requires the simplicial condition which is the {\it minimum} condition
required for consistent novel word detection
(Lemma~\ref{lem:simplicial-necessary}). This theorem holds true if the
topic prior $\bar{\mathbf{R}}$ satisfies stronger conditions such as
affine-independence.
We also point out that our proof in this paper holds for {\it any
  isotropic distribution} on the random projection directions
$\mathbf{d}^{1}, \ldots, \mathbf{d}^{P}$. The previous result in
\cite{Ding14:ref}, however, only applies to some specific isotopic
distributions such as the Spherical Gaussian or the uniform
distribution in a unit ball. In practice, we use Spherical Gaussian
since sampling from such prior is simple and requires only
$\mathcal{O}(W)$ time for generating each random direction.

Next, given the successful detection of the set of novel words for all
topics, we have the following result for the accurate estimation of
the separable topic matrix $\bm\beta$:
\begin{theorem}
\label{thm:topic-estimation}
Let topic matrix $\bm{\beta}$ be separable and $\bar{\mathbf{R}}$ be
$\gamma_a$-affine-independent.
Given the successful detection of novel words for all $K$ distinct
topics, the output of Algorithm~\ref{alg:text:esttopic1} $\widehat{\bm
  \beta} \xrightarrow{p} {\bm \beta}$ element-wise (up to a column
permutation). Specifically, if
\begin{equation}
\label{eqa:complexity-bound-2}
M \geq \frac{2560 W^2 K \log({W^{4}K}/{\delta})}{N \gamma_a^{2}
  a_{\min}^{2} \eta^4 \epsilon^2 }
\end{equation}
then $\forall i, k$, $\widehat{\beta}_{i,k}$ will be $\epsilon$ close
to $\beta_{i,k}$ with probability at least $1 - \delta$, for any $0<
\epsilon < 1$. $\eta$ is the same as in
Theorem~\ref{thm:novel-word-detection}. $a_{\min}$ is the minimum
value in $\mathbf{a}$.
\end{theorem}
We note that the sufficient sample complexity bound in
Eq.~\eqref{eqa:complexity-bound-2} is again polynomial in terms of all
the model parameters. Here we only require $\bar{\mathbf{R}}$ to be
affine-independent.
Combining Theorem~\ref{thm:novel-word-detection} and
Theorem~\ref{thm:topic-estimation} gives the consistency and sample
complexity bounds of our overall approach in
Algorithm~\ref{alg:text:highlevel}.
%
%
%***********************************************************
\section{Experimental Results}
\label{section:experiments}
In this section, we present experimental results on both synthetic and
real world datasets. We report different performance measures that
have been commonly used in the topic modeling literature.
When the ground truth is available (Sec.~\ref{sec:subsection:semi}),
we use the $\ell_1$ {\it reconstruction error} between the ground
truth topics and the estimates after proper topic alignment.
For the real-world text corpus in Sec.~\ref{sec:subsection:realdata},
we report the {\it held-out probability}, which is a standard measure
used in the topic modeling literature.
We also {\it qualitatively} (semantically) compare the topics
extracted by the different approaches using the top probable words for
each topic.
%
%%%%%%%%%%%%%%%%%%%%%%%%%%%%%%%%%%%%%
%
\subsection{Semi-synthetic text corpus}
\label{sec:subsection:semi}
%
%For the results of this section, 
In order to validate our proposed algorithm, we generate
``semi-synthetic'' text corpora by sampling from a synthetic, yet
realistic, ground truth topic model.
% configuration. 
To ensure that the semi-synthetic data is similar to real-world data,
in terms of dimensionality, sparsity, and other characteristics, we
use the following generative procedure adapted from \cite{Arora2:ref,
  Ding14:ref}.

We first train an LDA model (with $K=100$) on a real-world dataset
using a standard Gibbs Sampling method with default parameters (as
described in \cite{Griffiths04Gibbs:ref,McCallumMALLET}) to obtain a
topic matrix $\bm\beta_0$ of size $W\times K$.
The real-world dataset that we use to generate our synthetic data is
derived from a New York Times (NYT) articles dataset
\cite{UCIdataset:ref}.
The original vocabulary is first pruned based on document
frequencies. Specifically, as is standard practice, only words that
appear in more than $500$ documents are retained. Thereafter, again as
per standard practice, the words in the so-called stop-word list are
deleted as recommended in \cite{stopword:ref}. After these steps, $M =
300,000$, $W = 14,943$, and the average document length $N = 298$.
We then generate semi-synthetic datasets, for various values of $M$,
by fixing $N = 300$ and using $\bm{\beta}_0$ and a Dirichlet topic
prior. As suggested in \cite{Griffiths04Gibbs:ref} and used in
\cite{Arora2:ref,Ding14:ref}, we use symmetric hyper-parameters
($0.03$) for the Dirichlet topic prior.

The $W \times K$ topic matrix $\bm{\beta}_0$ may not be separable. To
enforce separability, we create a new {\it separable} $(W+K)\times K$
dimensional topic matrix $\bm{\beta}_{\text{sep}}$ by inserting $K$
synthetic novel words (one per topic) having suitable probabilities in
each topic. Specifically, $\bm{\beta}_{\text{sep}}$ is constructed by
transforming $\bm\beta_0$ as follows. First, for each synthetic novel
word in $\bm{\beta}_{\text{sep}}$, the value of the sole nonzero entry
in its row is set to the probability of the most probable word in the
topic (column) of $\bm\beta_0$ for which it is a novel word. Then the
resulting $(W+K)\times K$ dimensional nonnegative matrix is
renormalized column-wise to make it column-stochastic.
Finally, we generate semi-synthetic datasets, for various values of
$M$, by fixing $N = 300$ and using $\bm{\beta}_{\text{sep}}$ and the
same symmetric Dirichlet topic prior used for $\bm\beta_0$.

We use the name {\it Semi-Syn} to refer to datasets that are generated
using $\bm{\beta}_0$ and the name {\it Semi-Syn$+$Novel} for datasets
generated using $\bm{\beta}_{\text{sep}}$.

%%
%The real-world dataset that we use to generate our synthetic data is
%derived from a New York Times (NYT) articles dataset
%\cite{UCIdataset:ref}.
%%
%The original vocabulary is pruned based on document frequencies and
%then, as is standard practice, the words in the so-called stop-word
%list are deleted \cite{stopword:ref}. After these steps $M = 300,000$,
%$W = 14,943$, and the average document length $N = 298$. An LDA model
%with $K=100$ is trained with Gibbs Sampling
%\cite{Griffiths04Gibbs:ref,McCallumMALLET} using the default
%parameters.
%%
%We then generate the semi-synthetic datasets for various $M$ by fixing
%$N = 300$ and using a Dirichlet topic prior with symmetric
%hyper-parameters ($0.03$) as suggested in \cite{Griffiths04Gibbs:ref}
%and used in \cite{Arora2:ref,Ding14:ref}.
%%

In our proposed random projections based algorithm, which we call RP,
we set $\numofproj = 150\times K$, $\zeta = 0.05$, and $\epsilon =
10^-4$. We compare RP against the provably efficient algorithm
RecoverL2 in \cite{Arora2:ref} and the standard Gibbs Sampling based
LDA algorithm (denoted by Gibbs) in
\cite{Griffiths04Gibbs:ref,McCallumMALLET}. In order to measure the
performance of different algorithms in our experiments based on
semi-synthetic data, we compute the $\ell_1$ norm of the {\it
  reconstruction error} between $\widehat{\bm\beta}$ and $\bm{\beta}$.
Since all column permutations of a given topic matrix correspond to
the same topic model (for a corresponding permutation of the topic
mixing weights), we use a bipartite graph matching algorithm to
optimally match the columns of $\widehat{\bm\beta}$ with those of
$\bm{\beta}$ (based on minimizing the sum of $\ell_1$ distances
between all pairs of matching columns) before computing the $\ell_1$
norm of the reconstruction error between $\widehat{\bm\beta}$ and
$\bm{\beta}$.

The results on both {\it Semi-Syn$+$Novel} NYT and {\it Semi-Syn} NYT
are summarized in Fig.~\ref{plot_synNYT:stat} for all three algorithms
for various choices of the number of documents $M$. We note that in
these figures the $\ell_1$ norm of the error has been normalized by
the number of topics ($K=100$).
\begin{figure}[!hbt]
\centering
\begin{minipage}[b]{1\linewidth}
\centering
\includegraphics[scale=0.6]{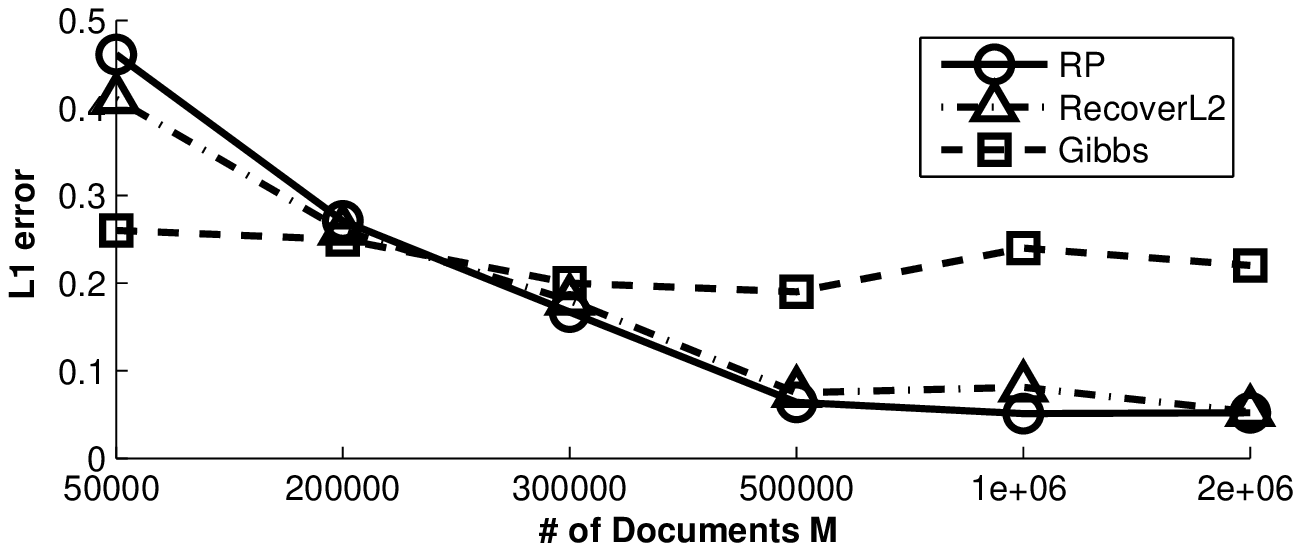}
\end{minipage}
%\quad
\begin{minipage}[b]{1\linewidth}
\centering
\includegraphics[scale=0.6]{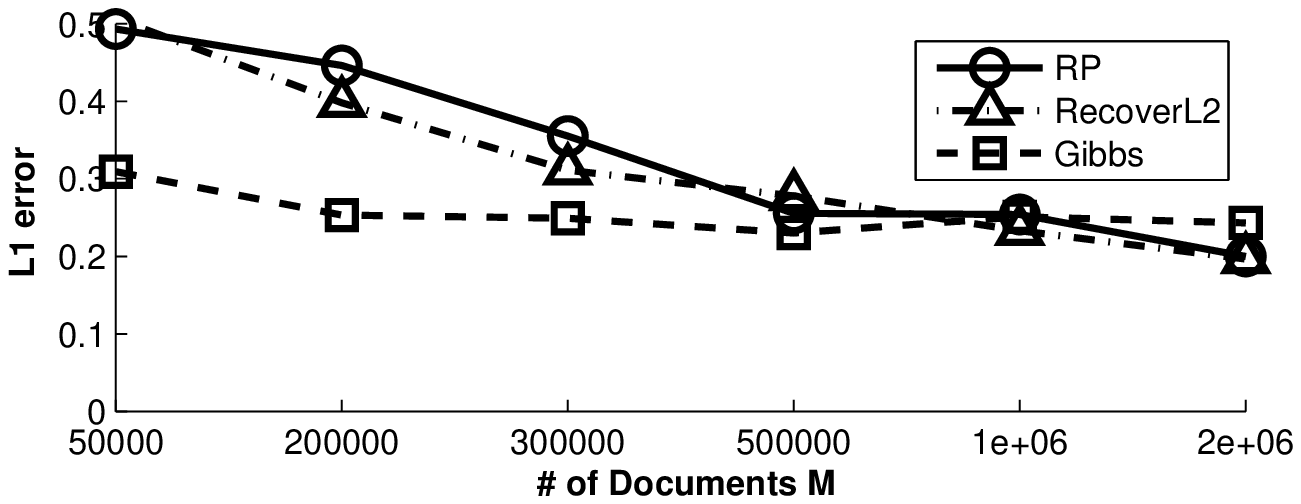}
\end{minipage}
\caption{$\ell_1$ norm of the
%reconstruction
error in estimating the topic matrix $\bm{\beta}$ for various $M$
($K=100$): (Top) {\it Semi-Syn$+$Novel} NYT; (Bottom) {\it Semi-Syn}
NYT. RP is the proposed algorithm, RecoverL2 is a provably efficient
algorithm from \cite{Arora2:ref}, and Gibbs is the Gibbs Sampling
approximation algorithm in \cite{Griffiths04Gibbs:ref}. In RP,
$\numofproj = 150K$, $\zeta = 0.05$, and $\epsilon = 10^-4$.}
\label{plot_synNYT:stat}
\end{figure}

As Fig.~\ref{plot_synNYT:stat} shows, when the separability condition
is strictly satisfied ({\it Semi-Syn$+$Novel} ), the reconstruction
error of RP converges to 0 as $M$ becomes large and outperforms the
approximation-based Gibbs.
When the separability condition is not strictly satisfied ({\it
  Semi-Syn}), the reconstruction error of RP is comparable to Gibbs (a
practical benchmark).

{\noindent\it Solid Angle and Model Selection:}~ 
%For visualization purposes 
In our proposed algorithm RP, the number of topics $K$ (the
model-order) needs to be specified. When $K$ is unavailable, it needs
to be estimated from the data. Although not the focus of this work,
Algorithm~\ref{alg:text:rp}, which identifies novel words by sorting
and clustering the estimated solid angles of words, can be suitably
modified to estimate $K$. 

Indeed, in the ideal scenario where there is no sampling noise ($M =
\infty, \widehat{\mathbf{E}} = \mathbf{E}$, and $\forall i, \hat{q}_i
= q_i$), only novel words have positive solid angles ($\hat{q}_i$'s)
and the rows of $\widehat{\mathbf{E}}$ corresponding to the novel words
of the same topic are identical, i.e., the distance between the rows
is zero or, equivalently, they are within a neighborhood of size zero
of each other.  Thus, the number of distinct neighborhoods of size zero
among the non-zero solid angle words equals $K$. 

In the nonideal case $M$ is finite. If $M$ is sufficiently large, one
can expect that the estimated solid angles of non-novel words will not
all be zero. They are, however, likely to be much smaller than those
of novel words. Thus to reliably estimate $K$ one should not only
exclude words with exactly zero solid angle estimates, but also those
above some nonzero threshold. When $M$ is finite, the the rows of
$\widehat{\mathbf{E}}$ corresponding to the novel words of the same
topic are unlikely to be identical, but if $M$ is sufficiently large
they are likely to be close to each other. Thus, if the threshold
$\zeta$ in Algorithm~\ref{alg:text:rp}, which determines the size of
the neighborhood for clustering all novel words belonging to the same
topic, is made sufficiently small, then each neighborhood will have
only novel words belonging to the same topic.

With the two modifications discussed above, the number of distinct
neighborhoods of a suitably nonzero size (determined by $\zeta > 0$)
among the words whose solid angle estimates are larger than some
threshold $\tau > 0$ will provide an estimate of $K$. The values of
$\tau$ and $\zeta$ should, in principle, decrease to zero as $M$
increases to infinity. Leaving the task of unraveling the dependence
of $\tau$ and $\zeta$ on $M$ to future work, here we only provide a
brief empirical validation on both the {\it Semi-Syn$+$Novel} and {\it
  Semi-Syn} NYT datasets. We set $M=2,000,000$ so that the
reconstruction error has essentially converged (see
Fig.~\ref{plot_synNYT:stat}), and consider different choices of the
threshold $\zeta$.
%that influences size of the neighborhood . 
%We then rank order the words based on the empirical solid angle and
%run Algorithm~\ref{alg:text:rp} until the novel word selection
%procedure terminates.
%

We run Algorithm~\ref{alg:text:rp} with $K=100$, $P=150\times K$, and
a new line of code: 16': ({\bf if} $\{\hat{q}_i = 0\}$, {\bf break});
inserted between lines 16 and 17 (this corresponds to $\tau = 0$).
The input hyperparameter $K=100$ is not the actual number of estimated
topics. It should be interpreted as specifying an upper bound on the
number of topics. The value of (little) $k$ when
Algorithm~\ref{alg:text:rp} terminates (see lines 14--21) provides an
estimate of the number of topics.
\begin{figure*}[!hbt]
\centering
%
%\begin{minipage}[b]{1\linewidth}
\begin{minipage}[b]{1\textwidth}
\centering
\includegraphics[scale=0.6]{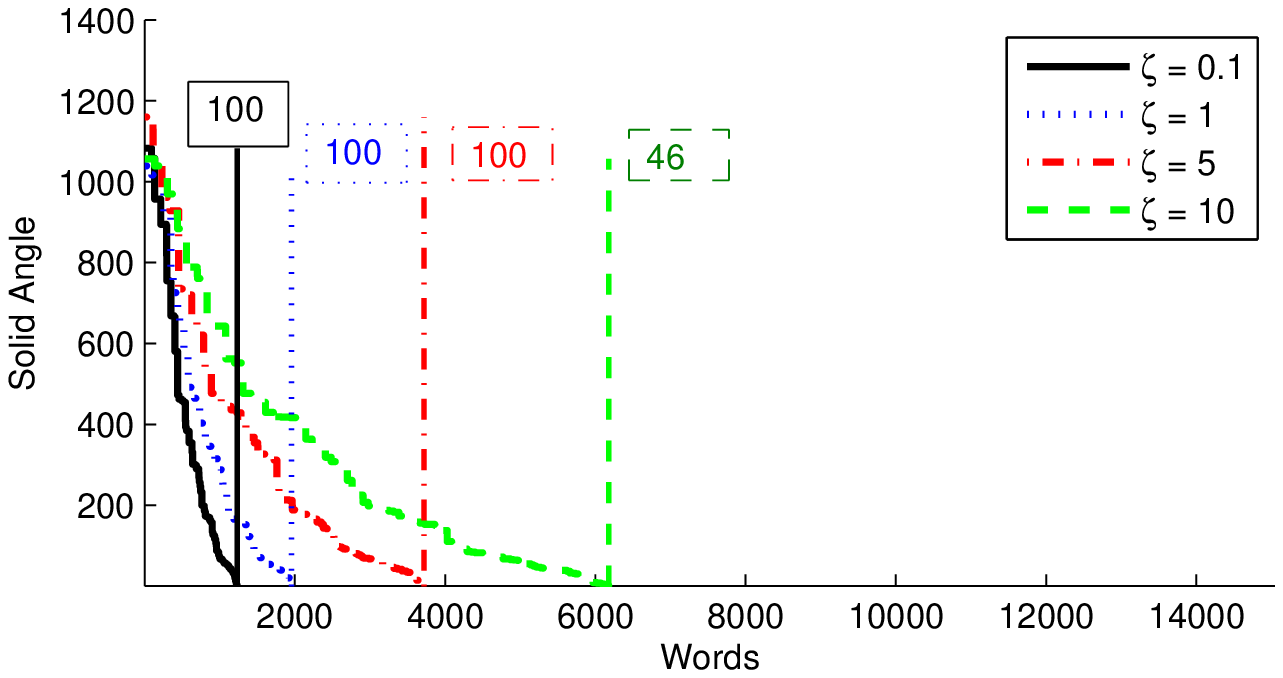}
\includegraphics[scale=0.6]{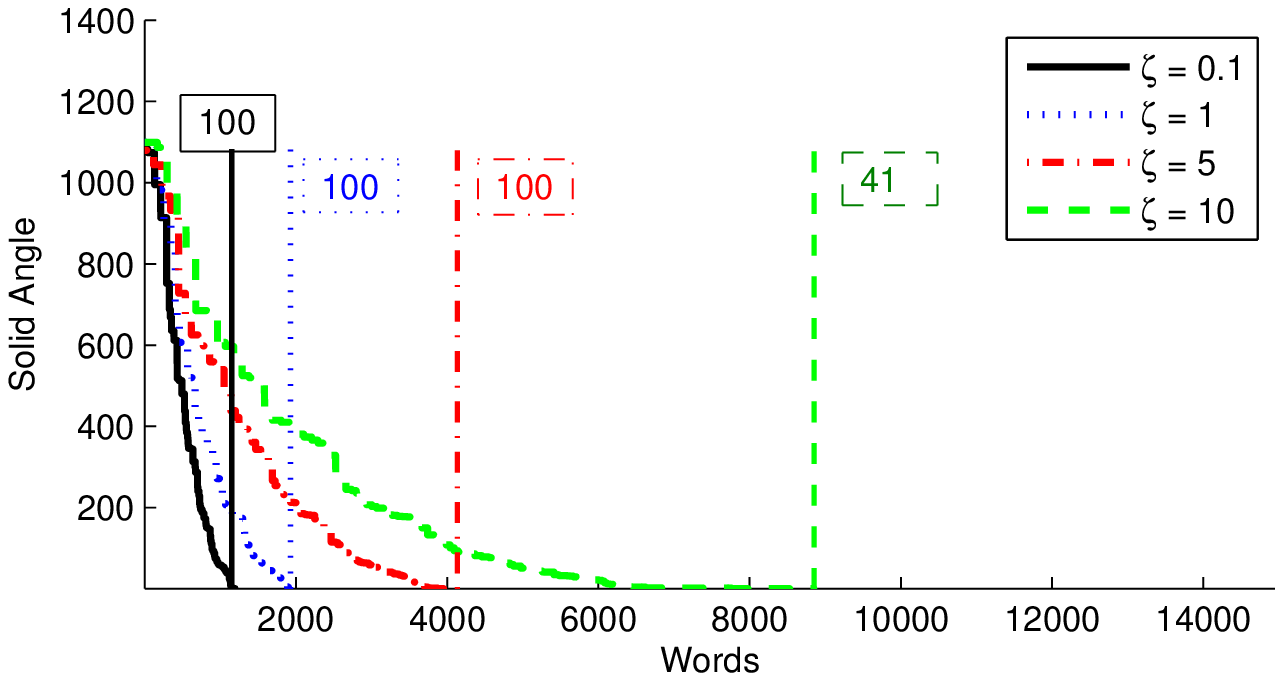}
\end{minipage}
\caption{Solid-angles (in descending order) of all $14943+100$ words
  in the {\it Semi-Syn$+$Sep} NYT dataset (left) and all $14943$ words
  in the {\it Semi-Syn} NYT dataset (right) estimated (for different
  values of $\zeta$) by Algorithm~\ref{alg:text:rp} with $K=100$,
  $P=150\times K$, $M=2,000,000$, and a new line of code: 16': ({\bf
    if} $\{\hat{q}_i = 0\}$, {\bf break}); inserted between lines 16
  and 17.
  The values of $j$ and (little) $k$ when Algorithm~\ref{alg:text:rp}
  terminates are indicated, respectively, by the position of the
  vertical dashed line and the rectangular box next to it for
  different $\zeta$.
%Illustration of novel word selection in terms of sorting and
%clustering of solid angles for the {\it Semi-Syn$+$Sep} NYT
%dataset. For this example $W=14943$ and $K=100$. For each choice of
%clustering parameter $\zeta$, we rank the estimated solid angles of
%all extreme points in descending order. The termination points are
%highlighted along with the value of $K$ at the termination point.
}
\label{plot_solidangle}
%\label{plot_solidangle_semisyn_sep}
%
\end{figure*}
%
%
%%
%\begin{figure}[!hbt]
%%
%\centering
%%
%\begin{minipage}[b]{1\linewidth}
%%
%\centering
%%
%\includegraphics[scale=0.6]{Figures/solidangle_semisyn}
%%
%\end{minipage}
%%
%\caption{An illustration of sorting and clustering based novel word
%  selection in our approach for the {\it Semi-Syn} NYT dataset. For
%  this example $W=14943$ and $K=100$. For each choice of clustering
%  parameter $\zeta$, we rank the estimated solid angles of all extreme
%  points in descending order. The stopping point are highlighted along
%  with the value of $K$ at the stopping point.}
%%
%%\label{plot_solidangle_semisyn}
%%
%\end{figure}
%%
%%

%Figure~\ref{plot_solidangle_semisyn_sep}
%and~\ref{plot_solidangle_semisyn}
Figure~\ref{plot_solidangle} illustrates how the solid angles of all
words, sorted in descending order, decay for different choices of
$\zeta$ and how they can be used to detect the novel words and
estimate the value of $K$. We note that in both the semi-synthetic
datasets, for a wide range of values of $\zeta$ (0.1--5), the modified
Algorithm~\ref{alg:text:rp} correctly estimates the value of $K$ as
$100$. When $\zeta$ is large (e.g., $\zeta=10$ in
Fig.~\ref{plot_solidangle}),
%Fig.~\ref{plot_solidangle_semisyn_sep}
%and~\ref{plot_solidangle_semisyn}), 
many interior points would be declared as novel words and multiple
ideal novel words would be grouped into one cluster resulting. This
causes $K$ to be underestimated (46 and 41 in
Fig.~\ref{plot_solidangle}).

%
%%%%%%%%%%%%%%%%%%%%%%%%%%%
\subsection{Real-world data}
\label{sec:subsection:realdata}
We now describe results on the actual real-world NYT dataset that was
used in Sec.~\ref{sec:subsection:semi} to construct the semi-synthetic
datasets.
%In this dataset, we have $M=300,000;~W=14,943; N\approx 300$
%\cite{UCIdataset:ref}.
%%
%The vocabulary size $W=14,943$ is pruned from the raw observation
%based on document frequencies. Specifically, as is standard practice,
%we keep only words that appear in more than $500$ documents. We then
%delete the stop-word list as recommended in \cite{stopword:ref}.
%
Since ground truth topics are unavailable, we measure performance
using the so-called {\it predictive held-out log-probability}.
This is a standard measure which is typically used to evaluate how
well a learned topic model fits real-world data.
To calculate this for each of the three topic estimation methods
(Gibbs~\cite{Griffiths04Gibbs:ref,McCallumMALLET},
RecoverL2~\cite{Arora2:ref}, and RP), we first randomly select
$60,000$ documents to test the goodness of fit and use the remaining
$240,000$ documents to produce an estimate $\widehat{\bm\beta}$ of the
topic matrix.
Next we assume a Dirichlet prior on the topics and estimate its
concentration hyper-parameter ${\bm\alpha}$. In Gibbs, this estimate
$\widehat{\bm\alpha}$ is a byproduct of the algorithm. In RecoverL2
and RP this can be estimated from $\widehat{\bm\beta}$ and
$\mathbf{X}$
%%%% Insert explanation and reference:%%%
. 
We then calculate the probability of observing the test documents given
the learned topic model $\widehat{\bm\beta}$ and $\widehat{\bm\alpha}$: 
\[
\log\Pr(\mathbf{X}_{\text{test}} \vert \widehat{\bm\beta},
\widehat{\bm\alpha})
\]
Since an exact evaluation of this predictive log-likelihood is
intractable in general, we calculate it using the MCMC based
approximation proposed in \cite{Wallach09:ref} which is now a standard
approximation tool \cite{McCallumMALLET}.
For RP, we use $P=150\times K$, $\zeta = 0.05$, and $\epsilon = 10^-4$
as in Sec.~\ref{sec:subsection:semi}. We report the held-out log
probability, normalized by the total number of words in the test
documents, averaged across 5 training/testing splits. The results are
summarized in Table~\ref{table:nyt-heldout}.
\begin{table}[!hbt]
\caption{\small Normalized held-out log probability of RP, RecoverL2,
  and Gibbs Sampling on NYT test data. The Mean$\pm$STD's are
  calculated from $5$ different random training-testing splits.}
\label{table:nyt-heldout}
\centering
\begin{tabular}{|c|p{2cm}|p{2cm}|p{2cm}|}
\hline
K & RecoverL2 & Gibbs & RP \\
\hline
50 & -8.22$\pm$0.56 & -7.42$\pm$0.45 & -8.54$\pm$0.52 \\
100 & -7.63$\pm$0.52 & -7.50$\pm$0.47 & -7.45$\pm$0.51 \\
150 & -8.03$\pm$0.38 & -7.31$\pm$0.41 & -7.84$\pm$0.48\\
200 & -7.85$\pm$0.40 & -7.34$\pm$0.44 & -7.69$\pm$0.42 \\
\hline
\end{tabular}
\end{table}
%
%We compare RP against RecoverL2\cite{Arora2:ref} and
%Gibbs\cite{Griffiths04Gibbs:ref,McCallumMALLET}. 
As shown in Table~\ref{table:nyt-heldout}, Gibbs has the best
descriptive power for new documents. RP and RecoverL2 have similar,
but somewhat lower values than Gibbs. This may be attributed to
missing novel words that appear only in the test set and are crucial
to the success of RecoverL2 and RP. Specifically, in real-world
examples, there is a model-mismatch as a result of which the data
likelihoods of RP and RecoverL2 suffer.
Finally, we {\it qualitatively} access the topics produced by our RP
algorithm. We show some example topics extracted by RP trained on the
{\it entire} NYT dataset of $M = 300,000$ documents in
Table~\ref{table:nyt-topicwords}
\footnote{The zzz prefix in the NYT vocabulary is used to annotate
  certain special named entities. For example, {zzz\textunderscore
    nfl} annotates NFL. }
\begin{table}[!htb]
\centering
\caption{\small Examples of topics estimated by RP on NYT}
\begin{tabular}{|>{\small}m{0.15\linewidth}|>{\footnotesize}m{0.75\linewidth}|}
\hline
Topic label &  Words in decreasing order of estimated probabilities  \\
\hline \hline
``weather''&    weather wind    air     storm   rain    cold    \\
\hline
``feeling'' &   feeling sense   love    character       heart   emotion \\
\hline
``election''    & election      zzz\textunderscore florida      ballot  vote    zzz\textunderscore al\textunderscore gore       recount \\
\hline
``game''        & yard  game    team    season  play    zzz\textunderscore nfl          \\
\hline
\end{tabular}
\label{table:nyt-topicwords}
\end{table}
For each topic, its most frequent words are listed. As can be seen,
the estimated topics do form recognizable themes that can be assigned
meaningful labels. The full list of all $K=100$ topics estimated on
the NYT dataset can be found in \cite{DDP:ref}.
\section{Conclusion and Discussion}
This paper proposed a provably consistent and efficient algorithm for
topic discovery. We considered a natural structural property -- topic
separability -- on the topic matrix and exploited its geometric
implications.
We resolved the necessary and sufficient conditions that can guarantee
consistent novel words detection as well as separable topic
estimation. We then proposed a random projections based algorithm that
has not only provably polynomial statistical and computational
complexity but also state-of-the-art performance on semi-synthetic and
real-world datasets.
%

%
% Distributed and online variation
% {\noindent\it Efficient Distributed Topic Discovery:}
%
While we focused on the standard {centralized} batch implementation in
this paper, it turns out that our random projections based scheme is
naturally amenable to an efficient {distributed implementation} which
is of interest when the documents are stored on a network of
distributed servers. 
This is because the iid isotropic projection directions can be
precomputed and shared across document servers, and counts,
projections, and co-occurrence matrix computations have an additive
structure which allows partial computations to be performed at each
document server locally and then aggregated at a fusion center with
only a small communication cost.
%
%additive structure of the word co-occurrence matrix representation
%and the projection values makes it possible to approximate the
%empirical solid angles using only ``partial'' projection values that
%are computed locally and transmitted to some fusion center. Hence it
%requires insignificant communication cost.
%
% to achieve the same statistical efficiency of the centralized
% version by insignificant communication.
It turns out that the distributed implementation can provably match
the polynomial computational and statistical efficiency guarantees of
its centralized counterpart. As a consequence, it provides a provably
efficient alternative to the distributed topic estimation problem
which has been tackled using variations of MCMC or Variational-Bayes
in the literature
\cite{yao2009parallel,hoffman2010online,newman2009distributed,asyDisLDA:ref}
This is appealing for modern web-scale databases, e.g., those
generated by Twitter Streaming.
%as evidenced by, for instance, Twitter Streaming.
A comprehensive theoretical and empirical investigation of the
distributed variation of our algorithm can be found in
\cite{Ding14:ref}.
% 

%
% {\noindent\it Model Selection, Inference, and Prediction:}
%%
%In this paper, we only considered the estimation problem in topic
%modeling, i.e., to estimate the model parameter $\bm\beta$.
%%
%It remains open questions to establish theoretical guarantees for
%other important tasks in topic modeling including model selection,
%inference, and prediction.
%%
%The goal of model selection problem is to determine the right number
%of topic $K$. In the context of a separable topic model, the
%robustness measure of solid angle can be used as a statistic to
%determine the number of true extreme points as the number of topic
%$K$. This would be carefully investigated in the future work.
%%
%We further note that the goal of inference problem is to fit the
%document specific topic weights $\bm\theta^{m}$ and the prediction
%problem is to calculate the probability of a new words. It is not
%clear yet how separability can be helpful in establishing theoretical
%guarantees in these problems.
%
%

{\noindent\bf Separability of general measures:}~
We defined and studied the notion of separability for a $W \times K$
topic matrix $\bm\beta$ which is a finite collection of $K$
probability distributions over a finite set (of size $W$). It turns
out that we can extend the notion separability to a finite collection
of measures over a measurable space. This necessitates making a small
technical modification to the definition of separability to
accommodate the possibility of only having ``novel subsets'' that have
zero measure.
We also show that our generalized definition of separability is
equivalent to the so-called {\bf irreducibility} property of a finite
collection of measures that has recently been studied in the context
of mixture models to establish conditions for the identifiability of
the mixing components \cite{blanchard2014:ref, scott2015:ref}.
Consider a collection of $K$ measures $\nu_1,\ldots, \nu_K$ over a
measurable space $(\mathcal{X},\mathcal{F})$, where $\mathcal{X}$ is a
set and $\mathcal{F}$ is a $\sigma$-algebra over $\mathcal{X}$. We
define the generalized notion of separability for measures as follows.
\begin{definition}
\label{def:separability-measure}
{\bf (Separability)} A collection of $K$ measures $\nu_1,\ldots,
\nu_K$ over a measurable space $(\mathcal{X},\mathcal{F})$ is
separable if for all $k = 1,\ldots, K$,
\begin{equation}
\inf\limits_{A\in\mathcal{F}:~\nu_{k}>0} \max\limits_{j:~j\neq k}\frac{\nu_{j}(A)}{\nu_{k}(A)} = 0.
\end{equation}
\end{definition}
Separability requires that for each measure $\nu_k$, there exists a
sequence of measurable sets $A_{n}^{(k)}$, of nonzero measure with
respect to $\nu_k$, such that, for all $j\neq k$, the ratios
$\nu_{j}(A_{n}^{(k)})/\nu_{k}(A_{n}^{(k)})$ vanish
asymptotically. Intuitively, this means that for each measure there
exists a sequence of nonzero-measure measurable subsets that are
asymptotically ``novel'' for that measure. When $\mathcal{X}$ is a
finite set as in topic modeling, this reduces to the existence of
novel words as in Definition~\ref{definiton:separability} and
$A_{n}^{(k)}$ are simply the sets of novel words for topic $k$.
The separability property just defined is equivalent to the so-called
irreducibility property. Informally, a collection of measures is
irreducible if {\it only nonnegative linear combinations of them can
  produce a measure}. Formally,
\begin{definition}
\label{def:irreducible}
{\bf (Irreducibility)}
A collection of $K$ measures $\nu_1,\ldots, \nu_K$ over a measurable
space $(\mathcal{X},\mathcal{F})$ is irreducible if the following
condition holds: 
If $\forall A \in \mathcal{F}$, $\sum_{k=1}^{K}c_k\nu_k(A) \geq 0$,
then for all $k=1,\ldots, K$, $c_k\geq 0$.
%
% $\forall \bm\lambda\in\mathbf{R}^{K}$, if \sum_{k=1}^{K}c_k\nu_k$ is
%a measure, then $c_k\geq 0$ or all $k=1,\ldots, K$.
%
\end{definition} 
For a collection of nonzero measures,\footnote{A measure $\nu$ is
  nonzero if there exists at least one measurable set $A$ for which
  $\nu(A) > 0$.} these two properties are equivalent. Formally,
\begin{lemma}
\label{lem:separable-equiv-irreducible}
A collection of nonzero measures $\nu_1,\ldots, \nu_K$ over a
measurable space $(\mathcal{X},\mathcal{F})$ is irreducible if and
only if it is separable.
In particular, a topic matrix $\bm\beta$ is irreducible if and only if
it is separable.
\end{lemma}
The proof appears in Appendix~\ref{sec:irreducibility-proof}.
%

% {\noindent\it Other Mixed Membership Models:}
%
Topic models like LDA discussed in this paper belong to the larger
family of Mixed Membership Latent Variable Models \cite{MMLVM14:ref}
which have been successfully employed in a variety of problems that
include text analysis, genetic analysis, network community detection,
and ranking and preference discovery. The structure-leveraging
approach proposed in this paper can be potentially extended to this
larger family of models. Some initial steps in this direction for rank
and preference data are explored in \cite{topicRank2:ref}.
%
%viewed as an alternative solution to learn this general family of
%latent variable models. We note that recent work have made initial
%steps on several Mixed Membership models for rank and preference data
%\cite{topicRank2:ref}, network data \cite{anandkumar2014tensor}, etc.
%
%An interesting future direction of this paper is to identify suitable
%mixed membership models in which the separability property is natural.
%
%It is also a further direction to identify problems and the separable
%structural property in other mixed membership models.
%

Finally, in this entire paper, the topic matrix is assumed to be
separable. While {\it exact} separability may be an idealization, as
shown in \cite{Ding15HighProb:ref}, approximate separability is both
theoretically inevitable and practically encountered when $W \gg
K$. Extending the results of this work to approximately separable
topic matrices is an interesting direction for future work. Some steps
in this direction are explored in \cite{topicRank3:ref} in the context
of learning mixed membership Mallows models for rankings.

\section*{Acknowledgment}
This article is based upon work supported by the U.S. AFOSR under
award number \# FA9550-10-1-0458 (subaward \# A1795) and the U.S. NSF
under award numbers \# 1527618 and \# 1218992. The views and
conclusions contained in this article are those of the authors and
should not be interpreted as necessarily representing the official
policies, either expressed or implied, of the agencies.

% or

% add this to the beginning so that the long equations can be break into different pages.
\allowdisplaybreaks
\appendix
%
%
%
%
%%%%%%%%%%%%%%%%%%%%%%%%%%%%%%%%%%%%%%%%%
\subsection{Proof of Lemma~\ref{lem:simplicial-necessary}}
\begin{proof}
The proof is by contradiction. We will show that if $\bar{{\mathbf R}}$ is non-simplicial, we can construct two topic matrices $\bm{\beta}^{(1)}$ and $\bm{\beta}^{(2)}$ whose sets of novel words are not identical and yet $\mathbf{X}$ has the same distribution under both models. The difference between constructed $\bm{\beta}^{(1)}$ and $\bm{\beta}^{(2)}$ is not a result of column permutation. This will imply the impossibility of consistent novel word detection.

Suppose $\bar{{\mathbf R}}$ is non-simplicial. Then we can assume, without loss of generality, that its first row is within the convex hull of the remaining rows, i.e., $\bar{{\mathbf R}}_1 = \sum_{j=2}^{K}c_j \bar{{\mathbf R}}_j$, where $\bar{{\mathbf R}}_j$ denotes the $j$-th row of $\bar{{\mathbf R}}$, and $c_2,\ldots,c_K \geq 0$, $~\sum_{j=2}^{K}c_j =1$ are convex combination weights. Compactly, $\mathbf{e}^{\top}\bar{{\mathbf R}} \mathbf{e} = 0$ where $\mathbf{e} := \left[ -1, c_2, \ldots, c_K \right]^{\top}$.
Recalling that $\bar{{\mathbf R}} = \diag(\mathbf{a})^{-1}\mathbf{R}\diag(\mathbf{a})^{-1}$, where $\mathbf{a}$ is a positive vector and $\mathbf{R} = \eE(\bm{\theta}^m {\bm{\theta}}^{m\top} )$ by definition, we have
\begin{eqnarray*}
0 &= & \mathbf{e}^{\top}\bar{{\mathbf R}} \mathbf{e} 
= (\diag(\mathbf{a})^{-1}{\mathbf{e}})^{\top} \eE(\bm{\theta}^m {\bm{\theta}}^{m\top}) (\diag(\mathbf{a})^{-1}{\mathbf{e}})\\
&=& \eE ( \Vert \bm{\theta}^{m\top}
\diag(\mathbf{a})^{-1}{\mathbf{e}} \Vert_2^2),
\end{eqnarray*}
which implies that $ \bm{\theta}^{m\top} \diag(\mathbf{a})^{-1}{\mathbf{e}}  \stackrel{a.s.} = 0$.
%
%, or equivalently, $\bm{\theta}_{i1}/a_1 \stackrel{a.s.}=
%\sum\limits_{j=2}^{K} \bm{\theta}_{ij}c_j/a_j$.
%
From this it follows that if we define two nonnegative row vectors $\mathbf{b}_1 := b\left[ a_1^{-1}, 0,\ldots,0 \right]$ and $\mathbf{b}_2 = b \left[(1-\alpha) a_1^{-1}, \alpha c_2  a_2^{-1},\ldots, \alpha c_K a_K^{-1}\right]$, where $b > 0, 0 < \alpha < 1$ are constants, then $\mathbf{b}_1 \bm{\theta}^{m } \stackrel{a.s.} = \mathbf{b}_2 \bm{\theta}^{m }$ for any distribution on $\bm\theta^{m}$.

Now we construct two separable topic matrices $\bm{\beta}^{(1)}$ and $\bm{\beta}^{(2)}$ as follows. Let $\mathbf{b}_1$ be the first row and $\mathbf{b}_2$ be the second in $\bm\beta^{(1)}$. Let $\mathbf{b}_2$ be the first row and $\mathbf{b}_1$ the second in $\bm\beta^{(2)}$. Let $\mathbf{B}\in\rR^{W-2 \times K}$ be a valid separable topic matrix. Set the remaining $(W-2)$ rows of both $\bm{\beta}^{(1)}$ and $\bm{\beta}^{(2)}$ to be $\mathbf{B}(I_K - \diag(\mathbf{b}_1+\mathbf{b}_2))$. We can choose $b$ to be small enough to ensure that each element of $(\mathbf{b}_1+\mathbf{b}_2)$ is strictly less than $1$. This will ensure that $\bm{\beta}^{(1)}$ and $\bm{\beta}^{(2)}$ are column-stochastic and therefore valid separable topic matrices. Observe that $\mathbf{b}_2$ has at lease two nonzero components. Thus, word 1 is novel for $\bm\beta^{(1)}$ but non-novel for $\bm\beta^{(2)}$.

By construction, $\bm{\beta}^{(1)}\bm{\theta} \stackrel{a.s.} =
\bm{\beta}^{(2)} \bm{\theta}$, i.e., the distribution of $\mathbf{X}$ conditioned on $\bm{\theta}$ is the same for both models. Marginalizing over $\bm{\theta}$, the distribution of $\mathbf{X}$ under each topic matrix is the same. Thus no algorithm can consistently distinguish between $\bm{\beta}^{(1)}$ and $\bm{\beta}^{(2)}$ based on $\mathbf{X}$. 
\end{proof} 
\subsection{Proof of Lemma~\ref{lem:affine-necessary}}
\begin{proof}
The proof is by contradiction. Suppose that $\bar{\mathbf{R}}$ is not
affine-independent. Then there exists a $\bm\lambda 
%:= [\lambda_1,\ldots, \lambda_K]^{\top} 
\neq \mathbf{0}$
%which are not all zero, and $\sum_{k}\lambda_{k} = 0$, 
with $\mathbf{1}^\top \bm\lambda = 0$ such that
%$\sum_{k}\lambda_{k}\bar{\mathbf{R}}_k = \mathbf{0}$.
%%
$\bm\lambda^{\top}\bar{\mathbf{R}} = \mathbf{0}$ so that
$\bm\lambda^{\top}\bar{\mathbf{R}}\bm\lambda = 0$. Recalling that
$\bar{\mathbf{R}} =
\diag(\mathbf{a})^{-1}\mathbf{R}\diag(\mathbf{a})^{-1}$, we have,
\begin{eqnarray*}
0 &=& \bm{\lambda}^{\top} \bar{\mathbf{R}} \bm{\lambda} 
= (\diag(\mathbf{a})^{-1}\bm{\lambda} )^{\top} \eE(\bm{\theta}^m
\bm{\theta}^{m\top}) (\diag(\mathbf{a})^{-1}\bm{\lambda}) \\
& =& \eE ( \Vert \bm{\theta}^{m\top}
\diag(\mathbf{a})^{-1}\bm{\lambda} \Vert^2),
\end{eqnarray*}
which implies that $ \bm{\theta}^{m\top}
\diag(\mathbf{a})^{-1}\bm{\lambda} \stackrel{a.s.} = 0$.
%
%, or equivalently, $\bm{\theta}_{i1}/a_1 \stackrel{a.s.}=
%\sum\limits_{j=2}^{K} \bm{\theta}_{ij}c_j/a_j$.
%
Since $\bm\lambda \neq \mathbf{0}$, we can assume, without loss of
generality, that the first $t$ elements of $\bm\lambda$,
%
%Noting that $\lambda_k$'s are not all zero and they sum up to zero,
%then, without loss of generality,
%
$\lambda_1,\ldots,\lambda_t  > 0$, the next $s$ elements of $\lambda$, $\lambda_{t+1},\ldots,\lambda_{t+s} <0$, and the remaining elements are $0$ for some $s,t: s>0, t>0, s+t\leq K$.
%$t>1$, $s>t$, and $s\leq K$. 
%
Therefore, if we define two nonnegative and nonzero row vectors
$\mathbf{b}_1 := b\left[ \lambda_{1} a_1^{-1}, \ldots, \lambda_{t}
  a_t^{-1} 0,\ldots,0 \right]$ and $\mathbf{b}_2 := -b
\left[0,\ldots,0, \lambda_{t+1} a_{t+1}^{-1},\ldots, \lambda_{s}
  a_{s}^{-1}, 0,\ldots, 0\right]$, where $b > 0$ is a constant, then $
\mathbf{b}_1 \bm{\theta}^m \stackrel{a.s.} = \mathbf{b}_2
\bm{\theta}^m $.

Now we construct two topic matrices $\bm{\beta}^{(1)}$ and
$\bm{\beta}^{(2)}$ as follows.  Let $\mathbf{b}_1$ be the first row
and $\mathbf{b}_2$ the second in $\bm\beta_1$. Let $\mathbf{b}_2$ be
the first row and $\mathbf{b}_1$ the second in $\bm\beta_2$.  Let
$\mathbf{B}\in\rR^{W-2 \times K}$ be a valid topic matrix and assume
that it is {\bf separable}. Set the remaining $(W-2)$ rows of both
$\bm{\beta}^{(1)}$ and $\bm{\beta}^{(2)}$ to be $\mathbf{B}(I_K -
\diag(\mathbf{b}_1+\mathbf{b}_2))$. We can choose $b$ to be small
enough to ensure that each element of $(\mathbf{b}_1+\mathbf{b}_2)$ is
strictly less than $1$. This will ensure that $\bm{\beta}^{(1)}$ and
$\bm{\beta}^{(2)}$ are column-stochastic and therefore valid topic
matrices.  We note that the supports of $\mathbf{b}_1$ and
$\mathbf{b}_2$ are disjoint and both are non-empty. They appear in
distinct topics.

By construction, $\bm{\beta}^{(1)}\bm{\theta} \stackrel{a.s.}
=\bm{\beta}^{(2)} \bm{\theta} \Rightarrow$ the distribution of the
observation $\mathbf{X}$ conditioned on $\bm{\theta}$ is the same for
both models. Marginalizing over $\bm{\theta}$, the distributions of
$\mathbf{X}$ under the topic matrices are the same. Thus no algorithm
can distinguish between $\bm{\beta}_1$ and $\bm{\beta}_2$ based on
$\mathbf{X}$.
\end{proof} 
%
%
%%%%%%%%%%%%%%%%%%%%%%%
%
%%%%%%%%%%%%%%%%%%%%%%%
\subsection{Proof of Proposition~\ref{prop:simplicial-vs-affine} and Proposition~\ref{prop:all-conditions}}
Proposition~\ref{prop:simplicial-vs-affine} and
Proposition~\ref{prop:all-conditions} summarizes the relationships
between the full-rank, affine-independence, simplicial, and
diagonal-dominance conditions. Here we consider all the pairwise
implication separately.

{\noindent (1) $\bar{\mathbf{R}}$ is $\gamma_a$-affine-independent
  $\Rightarrow$ $\bar{\mathbf{R}}$ is at least $\gamma_a$-simplicial.}
\begin{proof}
By definition of affine independence, $\Vert \sum_{k=1}^{K}
\lambda_{k} \bar{\mathbf{R}}_k \Vert_2$ $\geq \gamma_a \Vert
\bm\lambda \Vert_2$ $>0$ for all $\bm\lambda\in\mathbb{R}^{K}$ such
that $\sum_{k=1}^{K}\lambda_{k} = 0$ and $\bm\lambda \neq
\mathbf{0}$. If for each $i \in [K]$ we set $\lambda_k = 1$ for $k =
i$ and choose $\lambda_k \leq 0,\ \forall k \neq i$ then (i) $\Vert
\bm\lambda \Vert_2 \geq 1$, (ii) $\{-\lambda_k, k \neq i\}$ are convex
weights, i.e., they are nonnegative and sum to $1$, and (iii)
$\sum_{k=1}^{K} \lambda_{k} \bar{\mathbf{R}}_k = \bar{\mathbf{R}}_i -
\sum_{k\neq i} (-\lambda_{k}) \bar{\mathbf{R}}_k$.
%
%We consider $\bm\lambda = (1, -c_2,\ldots, -c_K)$ where
%$c_2,\ldots,c_K$ are convex weights, i.e., $c_2,\ldots, c_K >0$,
%$\sum_{j=2}^{K} c_j = 1$. Hence $\mathbf{1}^{\top}\bm\lambda =0$ and
%$\bm\lambda \neq \mathbf{0}$. And $\bm\lambda > 1$.
%%
Therefore, for all $i \in [K]$, $\Vert \bar{\mathbf{R}}_i -
\sum_{k\neq i} (-\lambda_{k}) \bar{\mathbf{R}}_k \Vert_2 \geq \gamma_a >
0$ which proves that $\bar{\mathbf{R}}$ is at least
$\gamma_a$-simplicial.

For the reverse implication, consider 
$$\bar{\mathbf{R}} = \begin{bmatrix}
    1 & 0 & 0.5 & 0.5 \\
    0 & 1 & 0.5 & 0.5 \\
    0.5 & 0.5 & 1 & 0 \\
    0.5 & 0.5 & 0 & 1
\end{bmatrix}.
$$
It is simplicial but is not affine independent (the $1, 1, -1, -1$ combination of the 4 rows would be $\mathbf{0}$).
\end{proof}

{\noindent (2) $\bar{\mathbf{R}}$ is full rank with minimum eigenvalue
  $\gamma_r$ $\Rightarrow$ $\bar{\mathbf{R}}$ is at least
  $\gamma_r$-affine-independent.}
\begin{proof}
The Rayleigh-quotient characterization of the minimum eigenvalue of a
symmetric, positive-definite matrix $\bar{\mathbf{R}}$ gives
$\min_{\bm\lambda \neq \mathbf{0}} \Vert \bm\lambda^\top \mathbf{R}
%\sum_{k=1}^{K} \lambda_{k} \bar{\mathbf{R}}_k 
\Vert_2 / \Vert \bm\lambda \Vert_2 = \gamma_r >0$. Therefore,
$\min_{\bm\lambda \neq \mathbf{0}, \mathbf{1}^\top\bm\lambda = 0}
\Vert \bm\lambda^\top \mathbf{R} \Vert_2 / \Vert \bm\lambda \Vert_2
\geq \gamma_r >0$.
%restricting the combination weights $\bm\lambda$ to the set that
%$\mathbf{1}^{\top}\bm\lambda = 0$ also satisfy the above inequality.
%It is straightforward to
One can construct examples that contradict the reverse implication:
$$\bar{\mathbf{R}} = \begin{bmatrix}
    1 & 0 & 1 \\
    0 & 1 & 1 \\
    1 & 1 & 2 
\end{bmatrix}.
$$
which is affine independent, but not linear independent.
\end{proof}

{\noindent (3) $\bar{\mathbf{R}}$ is $\gamma_d$-\-diagonal-\-dominant
  $\Rightarrow$ $\bar{\mathbf{R}}$ is at least $\gamma_d$-simplicial.}
\begin{proof}
Noting that $\bar{\mathbf{R}}_{i,i}- \bar{\mathbf{R}}_{i,j} \geq
\gamma_d > 0$ for all $i,j$, then the distance of the first row of
$\bar{\mathbf{R}}$, $\bar{\mathbf{R}}_1$, to any convex combination of
the remaining rows, $\sum\limits_{j=2}^{K} c_{j}\bar{\mathbf{R}}_j$,
where $c_2,\ldots, c_K$ are convex combination weights, can be lower
bounded by,
$ \Vert \bar{\mathbf{R}}_1 - \sum\limits_{j=2}^{K} c_{j}
\bar{\mathbf{R}}_j \Vert_2 \geq$ $ \vert \bar{\mathbf{R}}_{1,1} -
\sum\limits_{j=2}^{K} c_{j} \bar{\mathbf{R}}_{j,1} \vert = $ $\vert
\sum\limits_{j=2}^{K} c_{j}(\bar{\mathbf{R}}_{1,1} -
\bar{\mathbf{R}}_{j,1}) \vert \geq$ $\gamma_d > 0$.
%
%\begin{align*}
%%
%\Vert \bar{\mathbf{R}}_1 - \sum\limits_{j=2}^{K} c_{j}
%\bar{\mathbf{R}}_j \Vert_2 & \geq \vert \bar{\mathbf{R}}_{1,1} -
%\sum\limits_{j=2}^{K} c_{j} \bar{\mathbf{R}}_{j,1} \vert \\
%%
%& = \vert \sum\limits_{j=2}^{K} c_{j}(\bar{\mathbf{R}}_{1,1} -
%\bar{\mathbf{R}}_{j,1}) \vert \geq \gamma_d >0
%\end{align*}
Therefore, $\bar{\mathbf{R}}$ is at least $\gamma_d$-simplicial.
It is straightforward to construct examples that contradict the reverse
implication: 
$$\bar{\mathbf{R}} = \begin{bmatrix}
    1 & 0 & 1 \\
    0 & 1 & 1 \\
    1 & 1 & 2 
\end{bmatrix}.
$$ which is affine independent, hence simplicial, but not
diagonal-dominant.
\end{proof}

{\noindent (4) $\bar{\mathbf{R}}$ being diagonal-dominant neither
  implies nor is implied by $\bar{\mathbf{R}}$ being
  affine-independent.}

\begin{proof}
Consider the following two examples: 
$$\bar{\mathbf{R}} = \begin{bmatrix}
    1 & 0 & 1 \\
    0 & 1 & 1 \\
    1 & 1 & 2 
\end{bmatrix}.
$$ 
and 
$$\bar{\mathbf{R}} = \begin{bmatrix}
    1 & 0 & 0.5 & 0.5 \\
    0 & 1 & 0.5 & 0.5 \\
    0.5 & 0.5 & 1 & 0 \\
    0.5 & 0.5 & 0 & 1
\end{bmatrix}.
$$
They are the examples for the two sides of this assertion. 
\end{proof}
\subsection{Proof of Lemma~\ref{lem:extreme-A}}
\label{sec:append:thetasimplicial}
% We prove Lemma~\ref{lem:extreme-A} that summarizes the extreme point
%geometry in the row-vector space of $\bar{\mathbf{A}}$. For
%convenience, we define $\mathcal{C}_k$ to be the set of novel words of
%topic $k$, $k=1,\ldots, K$. We also define $\mathcal{C}_0$ as the set
%of all non-novel words.
%
%{\noindent \color{red} Updated Lemma~\ref{lem:extreme-A}}: Let $\bar{\mathbf{R}}$ be $\gamma_s$ simplicial and $\bm\beta$ be separable. Then, with probability at least $1 - 2K\exp(-c_1 M) -\exp(-c_2 M)$, a word $i$ is novel if, and only if, the $i$-th row of $\bar{\mathbf{A}}$ is an extreme point of the convex hull spanned by all the rows of $\bar{\mathbf{A}}$. Here the constant $c_1 := \gamma_s^2 a_{\min}^{4}/4\lambda_{\max}$ and $c_2 := \gamma_s^{4} a_{\min}^{4}/2\lambda_{\max}^{2}$. The model parameters are defined as follows. $a_{\min}$ is the minimum element of $\mathbf{a}$. $\lambda_{\max}$ is the maximum singular-value of $\bar{\mathbf{R}}$. 
%
\begin{proof}
Recall that$\bar{\mathbf{A}}=\bar{\bm{\beta}}\bar{\bm{\theta}}$ where $\bar{\mathbf{A}}$ and $\bar{\bm{\theta}}$ are row-normalized version of $\mathbf{A}$ and $\bm{\theta}$, $\bar{\bm\beta} :=
\diag(\mathbf{A}\mathbf{1})^{-1}\bm\beta \diag(\bm\theta\mathbf{1})$. $\bar{\bm\beta}$ is row-stochastic and is separable if $\bm\beta$ is separable.  
If $w$ is a novel word of topic $k$, $\bar{\beta}_{wk}=1$ and
$\bar{\beta}_{wj}=0,\ \forall j\neq k$. We have then
$\bar{\mathbf{A}}_w = \bar{\bm\theta}_k$. If $w$ is a non-novel word,
$\bar{\mathbf{A}}_w = \sum_{k}\bar{\beta}_{wk}\bar{\bm\theta}_{k}$ is
a convex combination of the rows of $\bar{\bm\theta}$.

We next prove that {\color{edit} if $\bar{\mathbf{R}}$ is $\gamma_s$-simplicial with some constant $\gamma_s > 0$, then, the random matrix $\bar{\bm\theta}$ is also simplicial with high probability, i.e., for any $\mathbf{c}\in\rR^{K}$ such that $c_k=1, c_j\leq 0, j\neq k, \sum_{j\neq k} -c_j=1, k\in[K]$, the $M$-dimensional vector $\mathbf{c}^{\top}\bar{\bm\theta}$ is not all-zero with high probability.} In another words, we need to show that {\color{edit}the maximum absolute value of the $M$ entries in $\mathbf{c}^{\top}\bar{\bm\theta}$ is strictly positive}. Noting that the $m$-th entry of $\mathbf{c}^{\top}\bar{\bm\theta}$ (scaled by $M$) is  
\begin{align*}
M \mathbf{c}^{\top}\bar{\bm\theta}^m = & \mathbf{c}^{\top}\diag(\mathbf{a})^{-1}\bm\theta^m \\
&+ \mathbf{c}^{\top}(\diag({\sum_{d}\bm\theta^{d}/M})^{-1} -\diag(\mathbf{a})^{-1})\bm\theta^m
\end{align*}
\noindent the absolute value can be lower bounded as follows,
\begin{align}
\nonumber \vert M \mathbf{c}^{\top}\bar{\bm\theta}^m \vert \geq & \vert \mathbf{c}^{\top}\diag(\mathbf{a})^{-1}\bm\theta^m  \vert \\
-&\vert \mathbf{c}^{\top}(\diag({\sum_{d}\bm\theta^{d}/M})^{-1} -\diag(\mathbf{a})^{-1})\bm\theta^m \vert 
\label{eqa:simplicial_decomp-2}
\end{align}
The key ideas are: {\color{edit}$(i)$ as $M$ increases, the second term in Eq.~\eqref{eqa:simplicial_decomp-2} converges to $0$, and $(ii)$ the maximum of the first term in Eq.~\eqref{eqa:simplicial_decomp-2} among $m=1,\ldots, M$ is strictly above zero with high probability}. 
For $(i)$, recall that $\mathbf{a} =\eE(\bm\theta^m)$ and $0\leq \theta_{k}^{m}\leq 1$, by Hoeffding's lemma $\forall t > 0$,
$$
\Pr( \Vert \sum_{d}\bm\theta^{d}/M -\mathbf{a} \Vert_{\infty} \geq t )\leq 2K\exp(-2Mt^2)
$$ 
%
%The factor $K$ comes from the union bound for each dimension. 
Also note that $\forall 0<\epsilon <1$,
\begin{align*}
 &\Vert \sum_{d}\bm\theta^{d}/M -\mathbf{a} \Vert_{\infty} \leq \epsilon a_{\min}^2/2 \\
\Rightarrow & \Vert (\diag({\sum_{d}\bm\theta^{d}/M})^{-1} -\diag(\mathbf{a})^{-1})\Vert_\infty \leq \epsilon \\
\Rightarrow & \vert \mathbf{c}^{\top}(\diag({\sum_{d}\bm\theta^{d}/M})^{-1} -\diag(\mathbf{a})^{-1})\bm\theta^m \vert \leq \epsilon
\end{align*}
where $a_{\min}$ is the minimum entry of $\mathbf{a}$. The last inequality is true since $\sum_{k=1}^{K}\theta_{k}^{m}=1$. In sum, we have
\begin{align}
\label{eqa:second-term-simplicial}
\nonumber &\Pr(\vert \mathbf{c}^{\top}(\diag({\sum_{d}\bm\theta^{d}/M})^{-1} -\diag(\mathbf{a})^{-1})\bm\theta^m \vert > \epsilon )\\
&\leq 2K\exp(- M \epsilon^2 a_{\min}^{4}/2)
\end{align}

For $(ii)$, recall that $\bar{\mathbf{R}}$ is $\gamma_s$-simplicial and $\Vert \mathbf{c}^{\top}\bar{\mathbf{R}} \Vert\geq \gamma_s$. Therefore, $\mathbf{c}^{\top}\bar{\mathbf{R}}\mathbf{c} = \mathbf{c}^{\top}\bar{\mathbf{R}}\bar{\mathbf{R}}^{\dagger}\bar{\mathbf{R}}\mathbf{c} \geq \frac{\gamma_s^{2}}{\lambda_{\texttt{max}}}$ where $\lambda_{\texttt{max}}$ is the maximum singular value of $\bar{\mathbf{R}}$. Noting that $\bar{\mathbf{R}} = \diag(\mathbf{a})^{-1}\eE(\bm{\theta}^m\bm{\theta}^{m\top})\diag(\mathbf{a})^{-1}$, we get
\begin{align}
\eE(\vert \mathbf{c}^{\top} \diag(\mathbf{a})^{-1} \bm\theta^m \vert^{2} ) \geq \frac{\gamma_s^{2}}{\lambda_{\texttt{max}}}
\label{eqa:expectation-simplicial}
\end{align}
For convenience, let $x_m := \vert \mathbf{c}^{\top} \diag(\mathbf{a})^{-1} \bm\theta^m \vert^2\leq 1/a_{\texttt{min}}^{2}$. 
Then, by Hoeffding's lemma,  
$$
\Pr(\eE(x_m) - \sum_{m=1}^{M} x_{m}/M \geq \frac{\gamma_{s}^{2}}{2\lambda_{\text{max}}}) \leq \exp(- M \gamma_s^{4} a_{\texttt{min}}^{4} /2\lambda_{\text{max}}^2)
$$
Combining Eq.~\eqref{eqa:expectation-simplicial} we get
$$
\Pr(\sum_{m=1}^{M} x_{m}/M \leq \frac{\gamma_{s}^{2}}{2\lambda_{\text{max}}}) \leq \exp(- M \gamma_s^{4} a_{\texttt{min}}^{4} /2\lambda_{\text{max}}^2)
$$
Hence 
\begin{align}
\Pr( \max_{m=1}^{M} x_m \leq \frac{\gamma_{s}^{2}}{2\lambda_{\text{max}}}) \leq \exp(- M \gamma_s^{4} a_{\texttt{min}}^{4} /2\lambda_{\text{max}}^2)
\label{eqa:first-term-simplicial}
\end{align}
i.e., the maximum absolute value of the first term in Eq.~\eqref{eqa:simplicial_decomp-2} is greater than $\gamma_s/\sqrt{2\lambda_{\max}}$ with high probability. 

To sum up, if we set $\epsilon = \gamma_s/\sqrt{2\lambda_{\max}}$ in Eq.~\eqref{eqa:second-term-simplicial}, we get 
\begin{align*}
\Pr(\max_{m=1}^{M} \vert \mathbf{c}^{\top}\bar{\bm\theta}^m \vert = 0) \leq & \Pr( \max_{m=1}^{M} x_m \leq \frac{\gamma_{s}^{2}}{2\lambda_{\text{max}}})  \\
& + \Pr(\vert \mathbf{c}^{\top}(\diag(\frac{\mathbf{a}}{\sum_{d}\bm\theta_{d}}) -I)\bm\theta^m  \vert > \epsilon)\\
\leq & \exp(- M \gamma_s^{4} a_{\texttt{min}}^{4} /2\lambda_{\text{max}}^{2}) \\
& + 2K\exp(-M \gamma_s^{2}a_{\min}^{4}/4\lambda_{max})
\end{align*}

To summarize, the probability that $\bar{\theta}$ is not simplicial is at most $\exp(- M \gamma_s^{4} a_{\texttt{min}}^{4} /2\lambda_{\text{max}}^{2}) + 2K\exp(-M \gamma_s^{2}a_{\min}^{4}/4\lambda_{max})$. This converges to $0$ exponentially fast as $M \rightarrow \infty$. Therefore, with high probability, all the row-vectors of $\bar{\bm\theta}$ are extreme points of the convex hull they form and this concludes our proof. 
\end{proof}
%
%
%%%%%%%%%%%%%%%%%%%%%%%%
\subsection{Proof of Lemma~\ref{lem:topic-regression}}
%
%{\noindent \color{red} Updated Lemma~\ref{lem:topic-regression}}: Let $\mathbf{A}$ and one novel words for each distinct topic be given. If $\bar{\mathbf{R}}$ is $\gamma_a$ affine-independent, then, with probability at least $1 - 2K\exp(-c_1 M) -\exp(-c_2 M)$, $\bm\beta$ can be recovered uniquely via constrained linear regression. Here the constant $c_1 := \gamma_a^2 a_{\min}^{4}/4\lambda_{\max}$ and $c_2 := \gamma_a^{4} a_{\min}^{4}/2\lambda_{\max}^{2}$. The model parameters are defined as follows. $a_{\min}$ is the minimum element of $\mathbf{a}$. $\lambda_{\max}$ is the maximum singular-value of $\bar{\mathbf{R}}$. 
%
\begin{proof}
{\color{edit} We first show that if $\bar{\mathbf{R}}$ is $\gamma_a$ affine-independent, $\bar{\bm\theta}$ is also affine-independent with high probability, i.e., $\forall \mathbf{c}\in\rR^{K}$ such that $\mathbf{c}\neq \mathbf{0}, \sum_{k}c_k = 0$, $\mathbf{c}^{\top}\bar{\bm\theta}$ is not all-zero vector with high probability.} Our proof is similar to that of Lemma~\ref{lem:extreme-A}.
We first re-write the $m$-th entry of $\mathbf{c}^{\top}\bar{\bm\theta}$ (with some scaling) as,
\begin{align*}
M \mathbf{c}^{\top}\bar{\bm\theta}^m = & \mathbf{c}^{\top}\diag(\mathbf{a})^{-1}\bm\theta^m \\
&+ \mathbf{c}^{\top}(\diag({\sum_{d}\bm\theta^{d}/M})^{-1} -\diag(\mathbf{a})^{-1})\bm\theta^m
\end{align*}
and lower bound its absolute value by 
\begin{align}
\nonumber \vert M\mathbf{c}^{\top}\bar{\bm\theta}^m \vert \geq & \vert \mathbf{c}^{\top}\diag(\mathbf{a})^{-1}\bm\theta^m\vert \\
-& \vert \mathbf{c}^{\top}(\diag({\sum_{d}\bm\theta^{d}/M})^{-1} -\diag(\mathbf{a})^{-1})\bm\theta^m\vert 
\label{eqa:affine_decomp-2}
\end{align}
We will then show that: {\color{edit}$(i)$ as $M$ increases, the second term in Eq.~\eqref{eqa:affine_decomp-2} converges to $0$, and $(ii)$ the maximum of the first term in Eq.~\eqref{eqa:affine_decomp-2} among $M$ iid samples is strictly above zero with high probability.}
For $(i)$, by the Cauchy-Schwartz inequality 
\begin{align*}
&\vert	\mathbf{c}^{\top}(\diag({\sum_{d}\bm\theta^{d}/M})^{-1} -\diag(\mathbf{a})^{-1})\bm\theta^m \vert \\
 \leq & \Vert \mathbf{c} \Vert_{2} \Vert (\diag({\sum_{d}\bm\theta^{d}/M})^{-1} -\diag(\mathbf{a})^{-1})\bm\theta^m \Vert_2 \\
 \leq & \Vert \mathbf{c} \Vert_2 \Vert (\diag({\sum_{d}\bm\theta^{d}/M})^{-1} -\diag(\mathbf{a})^{-1})  \Vert_{\infty} 
\end{align*}
Here the last inequality is true since $\theta_{k}^{m}\leq 1, \sum_{k}\theta_{k}^{m} = 1$. Similar to Eq.~\eqref{eqa:second-term-simplicial}, we have, 
\begin{align}
\nonumber &\Pr(\vert (\diag({\sum_{d}\bm\theta^{d}/M})^{-1} -\diag(\mathbf{a})^{-1})\bm\theta^m \vert \geq \Vert \mathbf{c} \Vert_2 \epsilon ) \\
\leq & 2K\exp(-M\epsilon^2 a_{\text{min}}^{4}/4)
\label{eqa:second-term-affine}
\end{align}
for any $0 < \epsilon < 1$, $a_{\text{min}}$ is the minimum entry of $\mathbf{a}$. For $(ii)$, recall that by definition, $\Vert \mathbf{c}^{\top}\bar{\mathbf{R}} \Vert_2 \geq \gamma_{a}\Vert \mathbf{c} \Vert_2$. Hence $\mathbf{c}^{\top}\bar{\mathbf{R}}\mathbf{c} \geq \gamma_{a}^{2} \Vert \mathbf{c} \Vert_2^{2}/\lambda_{\max}$. Therefore, by the construction of $\bar{\mathbf{R}}$, we have, 
\begin{align}
\eE(\vert \mathbf{c}^{\top}\diag(\mathbf{a})^{-1}\bm\theta^{m} \vert^2/ \Vert \mathbf{c}\Vert_2^{2}) \geq  \frac{\gamma_a^{2}}{\lambda_{\max}}
\end{align}

For convenience, let $x_m := \vert \mathbf{c}^{\top}\diag(\mathbf{a})^{-1}\bm\theta^{m} \vert^2/ \Vert \mathbf{c}\Vert_2^{2} \leq 1/a_{\min}^{2}$. Following the same procedure as in Eq.~\eqref{eqa:first-term-simplicial}, we have,
\begin{equation}
\Pr(\max_{m=1}^{M} x_m \leq \frac{\gamma_a^{2}}{2\lambda_{\max}})\leq \exp(-M\gamma_a^{4}a_{\min}^{4}/2\lambda_{\max}^2)
\end{equation}
Therefore, if we set in Eq.~\eqref{eqa:second-term-affine} $\epsilon = \gamma/\sqrt{2\lambda_{\max}}$, we get, 
\begin{align*}
\Pr(\max_{m=1}^{M} \vert \mathbf{c}^{\top}\bar{\bm\theta}^{m} \vert \leq 0 )\leq & \exp(-M\gamma_a^{4} a_{\min}^{4}/2\lambda_{\max}^2) \\
& + 2K\exp(-M \gamma_a^2 a_{\min}^4/4\lambda_{\max} )
\end{align*}
In summary, if $\bar{\mathbf{R}}$ is $\gamma_a$ affine-independent, $\bar{\bm\theta}$ is also
affine-independent with high probability.

{\color{edit} Now we turn to prove Lemma~\ref{lem:topic-regression}.} By Lemma~\ref{lem:extreme-A}, detecting $K$ distinct novel words for $K$ topics is equivalent to knowing $\bar{\bm\theta}$ up to a row permutation.
Noting that $\bar{\mathbf{A}}_w =
\sum_{k}\bar{\beta}_{wk}\bar{\bm\theta}_{k}$.
it follows that $\bar{\beta}_{wk}, k=1,\ldots, K$ is one optimal
solution to the following constrained optimization problem:
\begin{align*}
\min ~~\Vert \bar{\mathbf{A}}_w - \sum_{k=1}^{K} b_{k}
\bar{\bm\theta}_{k} \Vert^2 ~ \text{s.t} ~~ b_{k} \geq 0,
\sum_{k=1}^{K} b_{k} = 1
\end{align*}

{\color{edit} Since $\bar{\bm\theta}$ is affine-independent with high probability, therefore, this optimal solution is unique with high probability.} If this is not true, then there would exist two distinct solutions $b_1^{1},\ldots,b_{K}^{1}$ and $b_1^{2},\ldots,b_{K}^{2}$ such that $\bar{\mathbf{A}}_w = \sum_{k=1}^{K} b_{k}^{1} \bar{\bm\theta}_{k} = \sum_{k=1}^{K} b_{k}^{2} \bar{\bm\theta}_{k}$. $\sum b_{k}^{1} = \sum
b_{k}^{2} = 1$. We would then obtain
\begin{align*}
\sum_{k=1}^{K} (b_{k}^{1} - b_{k}^{2}) \bar{\bm\theta}_{k} = \mathbf{0} 
\end{align*}
where the coefficients $b_{k}^{1} - b_{k}^{2}$ are not all zero and
$\sum_k b_{k}^{1} - b_{k}^{2} = 0 $. This would contradict the
affine-independence definition.

Finally, we check the renormalization steps.  Recall that since
$\diag(\mathbf{A}\mathbf{1})\bar{\bm\beta} = \bm\beta
\diag(\bm\theta\mathbf{1})$, $\diag(\mathbf{A}\mathbf{1})$ can be
directly obtained from the observations. So we can first renormalize
the rows of $\bar{\bm\beta}$. Removing $\diag(\bm\theta\mathbf{1})$ is
then simply a column renormalization operation (recall that $\bm\beta$
is column-stochastic). It is not necessary to know the exact the value
of $\diag(\bm\theta\mathbf{1})$.

To sum up, by solving a constrained linear regression followed by
suitable row renormalization, we can obtain a unique solution which is
the ground truth topic matrix. This concludes the proof of
Lemma~\ref{lem:topic-regression}.
\end{proof}
%
%
%%%%%%%%%%%%%%%%%%%%%%%%%%%%%%
\subsection{Proof of Lemma~\ref{lem:second-order-convergence}}
Lemma~\ref{lem:second-order-convergence} establishes the second order co-occurrence estimator in Eq.~\eqref{eqa:word-co-occurrence-def}. 
We first provide a generic method to establish the explicit convergence bound for a function $\psi(\mathbf{X})$ of $d$ random variables $X_1,\ldots, X_d$, then apply it to establish Lemma~\ref{lem:second-order-convergence}
\begin{proposition}
\label{prop:convergence}
Let $\mathbf{X} = \left[X_1,\ldots, X_d \right]$ be $d$ random variables and $\mathbf{a}=\left[a_1,\ldots, a_d \right]$ be positive constants. 
Let $\mathcal{E}:=\bigcup\limits_{i \in\mathcal{I}} \lbrace \vert X_i - a_i\vert \geq \delta_i \rbrace$ for some constants $ \delta_i >0$,
and $\psi(\mathbf{X})$ be a continuously differentiable function in
$\mathcal{C} := \mathcal{E}^{c} $.
If for $i=1,\ldots, d$, $\Pr(\vert X_i - a_i \vert \geq \epsilon) \leq f_i(\epsilon)$ are the individual convergence rates  and $\max\limits_{X\in\mathcal{C}}\vert \partial_i \psi(\mathbf{X})\vert \leq C_i$, then, 
\begin{equation*}
\Pr(\vert \psi(\mathbf{X}) - \psi(\mathbf{a})\vert \geq \epsilon)  \leq \sum\limits_{i}f_{i}(\delta_i ) + \sum\limits_{i=1} f_{i}(\frac{\epsilon}{d C_i})
\end{equation*}
\end{proposition}
\begin{proof}
Since $\psi(\mathbf{X})$ is continuously differentiable in
$\mathcal{C}$, $\forall \mathbf{X}\in\mathbf{C}, \exists \lambda \in
(0,1)$ such that
\begin{equation*}
\psi(\mathbf{X}) - \psi(\mathbf{a}) = \nabla^{\top}\psi((1-\lambda)\mathbf{a} + \lambda\mathbf{X})\cdot (\mathbf{X}-\mathbf{a})
\end{equation*}
Therefore, 
\begin{align*}
&\Pr(\vert \psi(\mathbf{X}) - \psi(\mathbf{a})\vert \geq \epsilon) \\
\leq & \Pr(\mathbf{X}\in\mathcal{E}) + \\
& \Pr(\sum\limits_{i=1}^{d}\vert \partial_{i}\psi((1-\lambda)\mathbf{a} + \lambda\mathbf{X})	\vert \vert X_i - a_i \vert\geq \epsilon \vert \mathbf{X}\in\mathcal{C}	)\\
\leq & \sum\limits_{i\in\mathcal{I}} \Pr(\vert X_i - a_i\vert \geq \delta_i) + \\
& \sum\limits_{i=1}^{d}\Pr(\max\limits_{\mathbf{x}\in\mathcal{C}}\vert \partial_i \psi(\mathbf{x}) \vert\vert X_i - a_i\vert \geq \epsilon/d) \\ 
= & \sum\limits_{i\in\mathcal{I}}f_{i}(\delta_i ) + \sum\limits_{i=1} f_{i}(\frac{\epsilon}{d C_i})
\end{align*}
\end{proof}
Now we turn to prove Lemma~\ref{lem:second-order-convergence}. Recall that $\bar{\mathbf{X}}$ and $\bar{\mathbf{X}}^{\prime}$ are obtained from $\mathbf{X}$ by first splitting each user's comparisons into two independent halves and then re-scaling the rows to make them row-stochastic hence
$\bar{\mathbf{X}} = \diag^{-1}(\mathbf{X}\mathbf{1})\mathbf{X}$. 
Also recall that $\bar{\bm\beta} = \diag^{-1}(\bm\beta \mathbf{a}) \bm\beta
\diag(\mathbf{a})$, $\bar{\mathbf{R}}
=\diag^{-1}(\mathbf{a})\mathbf{R}\diag^{-1}(\mathbf{a})$, and $
\bar{\bm\beta}$ is row stochastic. 
For any $1\leq i,j\leq W$,
\begin{align*}
\widehat{E}_{i,j} & = M \frac{1}{\sum\limits_{m=1}^{M} X_{i,m}^{\prime}} (\sum\limits_{m=1}^{M} X_{i,m}^{\prime} X_{j,m} ) \frac{1}{\sum\limits_{m=1}^{M} X_{i,m}} \\
& = \frac{1/M \sum\limits_{m=1}^{M}(X_{i,m}^{\prime} X_{j,m})}{(1/M \sum\limits_{m=1}^{M}X_{i,m}^{\prime} )(1/M \sum\limits_{m=1}^{M} X_{j,m})} \\
& = \frac{\frac{1}{MN^2} \sum\limits_{m=1,n=1,n'=1}^{M,N,N} \mathbb{I}(w_{m,n}=i) \mathbb{I}(w_{m,n'}^{\prime}=j)}{\frac{1}{MN} \sum\limits_{m=1, n=1}^{M,N} \mathbb{I}(w_{m,n} =i)  \frac{1}{MN} \sum\limits_{m=1, n=1}^{M,N} \mathbb{I}(w_{m,n}^{\prime} =i)}\\
 & : = \frac{F_{i,j}(M,N)}{G_{i}(M,N) H_{j}(M,N)}
\end{align*}
From the Strong Law of Large Numbers and the generative topic modeling procedure,
\begin{align*}
& F_{i,j}(M,N) \xrightarrow{a.s.} \eE (\mathbb{I}(w_{m,n}=i) \mathbb{I}(w_{m,n'}^{\prime}=j))\\
&~~~~~~~~~~~~ = (\bm\beta\mathbf{R}\bm\beta^{\top})_{i,j} :=p_{i,j} \\
& G_{i}(M,N) \xrightarrow{a.s.} \eE (\mathbb{I}(w_{m,n}^{\prime}=i) ) = (\bm\beta\mathbf{a})_{i} :=p_{i} \\
& H_{i}(M,N) \xrightarrow{a.s.} \eE (\mathbb{I}(w_{m,n}=j) ) = (\bm\beta\mathbf{a})_{j} :=p_{j}
\end{align*}
and $\frac{(\bm\beta\mathbf{R}\bm\beta^{\top})_{i,j}}{(\bm\beta
  \mathbf{a})_{i} (\bm\beta \mathbf{a})_{j}} = \mathbf{E}_{i,j}$ by
definition.
Using McDiarmid's inequality, we obtain
\begin{align*}
& \Pr(\vert F_{i,j} - p_{i,j} \vert \geq \epsilon )\leq 2\exp(-\epsilon^2MN) \\ 
& \Pr(\vert G_{i} - p_{i} \vert \geq \epsilon )\leq 2\exp(-2\epsilon^2MN) \\
& \Pr(\vert H_{j} - p_{j} \vert \geq \epsilon )\leq 2\exp(-2\epsilon^2MN) 
\end{align*}
In order to calculate $\Pr\lbrace \vert \frac{F_{i,j}}{G_{i} H_{j}} -
\frac{p_{i,j}}{p_{i} p_{j}} \vert \geq \epsilon \rbrace$, we apply the
results from Proposition~\ref{prop:convergence}.
Let $\psi(x_1, x_2, x_3) = \frac{x_1}{x_2 x_3}$ with $x_1,x_2,x_3 >0$,
and $a_1 = p_{i,j}$, $a_2 = p_i$, $a_3= p_j$.
Let $\mathcal{I}=\{2,3\}$, $\delta_2 = \gamma p_i$, and $\delta_3 =
\gamma p_j$.
Then $\vert \partial_1\psi \vert = \frac{1}{x_2 x_3} $, $\vert
\partial_2\psi \vert = \frac{x_1}{x_2^{2} x_3} $, and $\vert
\partial_3\psi \vert = \frac{x_1}{x_2 x_3^{2}} $.
If $F_{i,j} = x_1$, $G_i = x_2$, and $H_j = x_3$, then $F_{i,j} \leq
G_i$, $F_{i,j}\leq H_j$. Then note that
\begin{align*}
C_1 &= \max_{\mathcal{C}} \vert \partial_1\psi \vert =  \max_{\mathcal{C}}  \frac{1}{G_i H_j} \leq \frac{1}{(1-\gamma)^{2}p_i p_j} \\
C_2 &= \max_{\mathcal{C}} \vert \partial_2\psi \vert =  \max_{\mathcal{C}}  \frac{F_{i,j}}{G_i^{2} H_j} \leq   \max_{\mathcal{C}}  \frac{1}{G_i H_j} \leq \frac{1}{(1-\gamma)^{2}p_i p_j} \\
C_3 &= \max_{\mathcal{C}} \vert \partial_3\psi \vert =  \max_{\mathcal{C}}  \frac{F_{i,j}}{G_i H_j^{2}} \leq   \max_{\mathcal{C}}  \frac{1}{G_i H_j}\leq \frac{1}{(1-\gamma)^{2}p_i p_j}
\end{align*}
By applying Proposition~\ref{prop:convergence}, we get
\begin{align*}
& \Pr\lbrace \vert \frac{F_{i,j}}{G_{i} H_{j}} - \frac{p_{i,j}}{p_{i} p_{j}} \vert \geq \epsilon \rbrace \\
\leq & \exp(-2\gamma^{2}p_i^{2} MN) + \exp(-2\gamma^{2}p_j^{2} MN) \\
& + 2\exp (-\epsilon^{2} (1-\gamma)^{4}(p_i p_j)^{2} MN/9)\\
& + 4 \exp (-2\epsilon^{2} (1-\gamma)^{4}(p_i p_j)^{2} MN/9) \\ 
\leq & 2\exp(-2\gamma^{2}\eta^{2}MN) + 6\exp(-\epsilon^{2} (1-\gamma)^{4}\eta^{4} MN/9)
\end{align*}
where $\eta = \min_{1\leq i \leq W} p_i$. 
There are many strategies for optimizing the free parameter
$\gamma$. We set $2\gamma^{2} = \frac{(1-\gamma)^{4}}{9}$ and solve
for $\gamma$ to obtain
\begin{align*}
 \Pr\lbrace \vert \frac{F_{i,j}}{G_{i} H_{j}} - \frac{p_{i,j}}{p_{i} p_{j}} \vert \geq \epsilon \rbrace \leq  8\exp(-\epsilon^{2}\eta^{4} MN/20)
\end{align*}
Finally, by applying the union bound to the $W^2$ entries in
$\widehat{\mathbf{E}}$, we obtain the claimed result. 
%
%%%%%%%%%%%%%%%%%%%%%%%%%%%%%%%%%%%%
\subsection{Proof of Lemma~\ref{lem:extreme-E}}
\begin{proof}
We first show that when $\bar{\mathbf R}$ is $\gamma_s$ simplicial and 
$\bm\beta$ is separable, then $\mathbf{Y}= \bar{\mathbf R} \bar{\bm 
\beta}^{\top}$ is at least $\gamma_s$-simplicial. 
Without loss of generality we assume that word $1,\ldots, K$ are the 
novel words for topic $1$ to $K$. By definition, $\bar{\bm \beta}
^{\top} = \left[ \mathbf{I}_{K}, \mathbf{B} \right]$ hence $\mathbf{Y}= 
\bar{\mathbf R} \bar{\bm \beta}^{\top} = \left[ \bar{\mathbf R}, 
\bar{\mathbf R}\mathbf{B} \right] $. Therefore, for convex combination 
weights $c_2,\ldots,c_K\geq 0$ such that $\sum_{j=2}^{K}c_j = 1$, 
\begin{align*}
\Vert \mathbf{Y}_1 -\sum_{j=2}^{K} c_j \mathbf{Y}_j  \Vert \geq \Vert \bar{\mathbf R}_1 -\sum_{j=2}^{K} c_j \bar{\mathbf R}_j  \Vert \geq \gamma_s>0
\end{align*}
Therefore the first row vector $\mathbf{Y}_1$ is at least $\gamma_s$ 
distant away from the convex hull of the remaining rows. Similarly, any 
row of $\mathbf{Y}$ is at least $\gamma_s$ distant away from the convex 
hull of the remaining rows hence $\mathbf{Y}$ is at least $\gamma_s$ 
simplicial. 
The rest of the proof will be exactly the same as for Lemma~\ref{lem:extreme-E}. 
\end{proof}
%
%
%
%%%%%%%%%%%%%%%%%%%%%%%%%%%%%%%%%%%%
\subsection{Proof of Lemma~\ref{lem:topic-regression-E}}
\begin{proof}
We first show that when $\bar{\mathbf R}$ is $\gamma_a$ affine independent and $\bm\beta$ is separable, then $\mathbf{Y}= \bar{\mathbf R} \bar{\bm \beta}^{\top}$ is at least $\gamma_a$ affine independent. Similarly as in the proof of Lemma~\ref{lem:extreme-E},  we assume that word $1,\ldots, K$ are the novel words for topic $1$ to $K$. By definition, $\bar{\bm \beta}^{\top} = \left[ \mathbf{I}_{K}, \mathbf{B} \right]$ hence $\mathbf{Y}= \bar{\mathbf R} \bar{\bm \beta}^{\top} = \left[ \bar{\mathbf R}, \bar{\mathbf R}\mathbf{B} \right] $. $\forall \bm\lambda\in\mathbb{R}^{K}$ such that $\bm\lambda\neq\mathbf{0}$, $\sum_{k=1}^{K}\lambda_k = 0$, then, 
\begin{align*}
\Vert \sum_{k=1}^{K} \mathbf{Y}_k \Vert_2 / \Vert \bm\lambda \Vert_2 \geq \Vert \sum_{k=1}^{K} \bar{\mathbf{R}}_{k} \Vert_2 / \Vert \bm\lambda \Vert_2 \geq \gamma_a
\end{align*}
Hence $\mathbf{Y}$ is affine independent. The The rest of the proof will be exactly the same as that for Lemma~\ref{lem:topic-regression}. 
We note that once the novel words for $K$ topics are detection, we can use only the corresponding columns of $\mathbf{E}$ for linear regression. Formally, let $\mathbf{E}^{*}$ be the $W\times K$ matrix formed by the columns of the $\mathbf{E}$ that correspond to $K$ distinct novel words. Then, $\mathbf{E}^{*} = \bar{\bm\beta} \bar{\mathbf{R}}$. The rest of the proof is again the same as that for Lemma~\ref{lem:topic-regression}.  
\end{proof}
%
%
%%%%%%%%%%%%%%%%%%%%%%%%%%%%%%%%%%%%
\subsection{Proof of Lemma~\ref{lem:extreme-solid-angle} }
\begin{proof}
We first check that if $q_{w}>0$, $w$ must be a novel word. Without loss of generality let word $1,\ldots, K$ be novel words for $K$ distinct topics. $\forall w$, $\mathbf{E}_{w} = \sum\bar{\beta}_{wk}\mathbf{E}_{k}$. $\forall \mathbf{d}\in\mathbb{R}^{W}$, 
\begin{align*}
\langle \mathbf{E}_{w}, \mathbf{d}\rangle =  \sum\bar{\beta}_{wk}\langle \mathbf{E}_{k} , \mathbf{d}\rangle \leq \max_{k} \langle \mathbf{E}_{k} , \mathbf{d}\rangle
\end{align*}
and the last equality holds if, and only if, there exist some $k$ such that $\bar{\beta}_{wk}=1$ which implies $w$ is a novel words.

We then show that for a novel word $w$, $q_{w}>0$. We need to show for each topic $k$, when $\mathbf{d}$ is sampled from an isotropic distribution in $\mathbf{R}^{W}$, there exist a set of directions $\mathbf{d}$ with nonzero probability such that $\langle \mathbf{E}_{k}, \mathbf{d}\rangle > \langle \mathbf{E}_{l}, \mathbf{d}\rangle $ for $l = 1,\ldots, K, l\neq k$. 
First, one can check by definition that  $\mathbf{Y}=(\mathbf{E}_1^{\top},\ldots, \mathbf{E}_{K}^{\top})^{\top} = \bar{\mathbf{R}}\bar{\bm\beta}^{\top}$ is at least $\gamma_s$-simplicial if $\bar{\mathbf{R}}$ is $\gamma_s$-simplicial. Let $\mathbf{E}_1^{*}$ be the projection of $\mathbf{E}_1$ onto the simplex formed by the remaining row vectors $\mathbf{E}_2,\ldots, \mathbf{E}_K$.
By the orthogonality principle,
$
\langle \mathbf{E}_{1} - \mathbf{E}_{1}^{*}, \mathbf{E}_{k} - \mathbf{E}_{1}^{*} \rangle \leq 0
$
for $k=2,\ldots, K$. Therefore, for $\mathbf{d}^1 = \mathbf{E}_{1}^{\top} - \mathbf{E}_{1}^{*\top}$, 
\begin{align*}
& \mathbf{E}_{1}\mathbf{d}^1 - \mathbf{E}_{k}\mathbf{d}^1  
= \Vert \mathbf{d}^1 \Vert^{2} - (\mathbf{E}_{k} -  \mathbf{E}_{1}^{*})\mathbf{d}^1 \geq \gamma_s^{2}  >  0
\end{align*}
Due to the continuity of the inner product, there exist a neighbor on the unite sphere around $\mathbf{d}^1/\Vert \mathbf{d}^1 \Vert_2$ that $\mathbf{E}_1$ has maximum projection value. This conclude our proof.  
\end{proof}
%
%
%%%%%%%%%%%%%%%%%%%%%%%%%%%%%%%%%%%%%%%%%%%
\subsection{Proof of Theorem~\ref{thm:computation}}
\begin{proof}
We first consider the random projection steps (step 3 to 12 in Alg.~\ref{alg:text:rp}).  For projection along direction $\mathbf{d}^{r}$, we first calculate projection values $\mathbf{r} = \bar{\mathbf{X}}^{\prime}\bar{\mathbf{X}}^{\top}\mathbf{d}^{r}$, find the maximizer index $i^*$ and the corresponding set $\hat{J}_{i^*}$, and then evaluate $\mathbb{I}(\forall j \in \hat{J}_{w}, v_{w} > v_{j})$ for all the words $w$ in $\hat{J}_{i^*}^c = \{1, \ldots, W\} \setminus \hat{J}_{i^*}$.  
$(I)$ The set $\hat{J}_{i^*}^c$ have up to $|\mathcal{C}_k|$ elements asymptotically, where $k$ is the topic associated with word $i^{*}$. This is considered a small constant $\mathcal{O}(1)$;  
$(II)$ Note that $\widehat{\mathbf E}\mathbf{d}_{r} = M \bar{\mathbf X}^{\prime} (\bar{\mathbf X}^{\top} \mathbf{d}_{r})$ and each column of $\bar{\mathbf X}$ has at most $N\ll W$ nonzero entries. Calculating the $W\times 1$ projection value vector $\mathbf{v}$ requires two sparse matrix-vector multiplications and takes $\mathcal{O}(M N)$ time. Finding the maximum requires $\mathbf{W}$ running time;  
$(III)$ To evaluate one set $\hat{J}_{i} \leftarrow \{ j : \widehat{E}_{i,i} + \widehat{E}_{j,j} - 2\widehat{E}_{i,j} \geq \zeta/2 \}$ we need to calculate $\widehat{E}_{i,j}, j=1,\ldots,W$. This can be viewed as projecting $\widehat{\mathbf{E}}$ along $\mathbf{d} = \mathbf{e}_{i}$ and takes $\mathcal{O}(MN)$. We also note that the diagonal entries $\mathbf{E}_{w,w}, w=1,\ldots, W$ can be calculated once using $\mathcal{O}(W)$ time.  To sum up, these steps takes $\mathcal{O}(MNP + WP)$ running time. 

We then consider the detecting and clustering steps (step 14 to 21 in Alg.~\ref{alg:text:rp}). We note that all the conditions in Step~17 have been calculated in the previous steps, and recall that the number of novel words are small constant per topic, then, this step will require a running time of $\mathcal{O}(K^2)$. 

We last consider the topic estimation steps in Algorithm~\ref{alg:text:esttopic1}. Here all the corresponding inputs for the linear regression have already been computed in the projection step. Each linear regression has $K$ variables and we upper bound its running time by $\mathcal{O}(K^3)$. Calculating the row-normalization factors $\frac{1}{M}\mathbf{X}\mathbf{1}$ requires $\mathcal{O}(MN)$ time. The row and column re-normalization each requires at most $\mathcal{O}(WK)$ running time. Overall, we need a $\mathcal{O}(WK^3 + MN)$ running time.  

Other steps are also efficient. Splitting each document into two independent halves takes linear time in $N$ for each document since we can achieve it using random permutation over $N$ items.
To generate each random direction $\mathbf{d}_{r}$ requires $\mathcal{O}(W)$ complexity if we use the spherical Gaussian prior. While we can directly sort the empirical estimated solid angles (in $\mathcal{O}(W\log(W))$ time), we only search for the words with largest solid angles whose number is a constant w.r.t $W$, therefore it would take only $\mathcal{O}(W)$ time. 
\end{proof}
\subsection{Proof of Theorem~\ref{thm:novel-word-detection}}
We focus on the case when the random projection directions are sampled
from {\bf any} isotropic distribution. Our proof is not tied to
the special form of the distribution; just its isotropic nature.
We first provide some useful propositions. We denote by
$\mathcal{C}_k$ the set of all novel word of topic $k$, for $k\in
[K]$, and denote by $\mathcal{C}_0$ the set of all non-novel words. We
first show,
\begin{proposition}
\label{prop:similarity}
Let $\mathbf{E}_i$ be the $i$-th row of $\mathbf{E}$. Suppose
$\bm\beta$ is separable and $\bar{\mathbf{R}}$ is
$\gamma_s$-simplicial, then the following is true: For all $k \in [K]$,
\begin{table}[H]
\centering
\begin{tabular}{|c|c|c|}
\hline 
 & $\Vert \mathbf{E}_i - \mathbf{E}_j \Vert$ & $E_{i,i}-2E_{i,j}+E_{j,j}$ \\ 
\hline 
$i\in\mathcal{C}_k, j \in\mathcal{C}_k$ & $0$ & $0$ \\ 
\hline 
$i\in\mathcal{C}_k, j \notin\mathcal{C}_k$ & $\geq (1-b)\gamma_s$ & $\geq (1-b)^{2}\gamma_s^{2}/\lambda_{\max}$ \\ 
\hline 
\end{tabular} 
\end{table}
\noindent where $b=\max_{j\in\mathcal{C}_0,l} \bar{\beta}_{j,l}
$ and $\lambda_{\max} > 0$ is the maximum eigenvalue of
$\bar{\mathbf{R}}$
\end{proposition}
\begin{proof}
We focus on the case $k=1$ since the proofs for other values of $k$
are analogous.
Let $\bar{\bm\beta}_{i}$ be the $i$-th row vector of matrix $\bar{\bm\beta}$. 
To show the above results, recall that $\mathbf{E} =
\bar{\bm\beta}\bar{\mathbf{R}}\bar{\bm\beta}^{\top}$. Then
\begin{align*}
& \Vert \mathbf{E}_i - \mathbf{E}_j \Vert =\Vert (\bar{\bm\beta}_{i} - \bar{\bm\beta}_{j}) \bar{\mathbf{R}} \bar{\bm\beta}^{\top}\Vert \\ 
& E_{i,i}-2E_{i,j}+E_{j,j} = (\bar{\bm\beta}_{i} - \bar{\bm\beta}_{j}) \mathbf{R}^{\prime} (\bar{\bm\beta}_{i} - \bar{\bm\beta}_{j})^{\top}.
\end{align*}
It is clear that when $i,j\in\mathcal{C}_1$, i.e., they are both novel word for the same topic, $\bar{\bm\beta}_{i} = \bar{\bm\beta}_{j} = \mathbf{e}_1$. Hence, $\Vert \mathbf{E}_i - \mathbf{E}_j \Vert = 0 $ and $E_{i,i}-2E_{i,j}+E_{j,j} = 0$.
When $i\in\mathcal{C}_1, j\notin \mathcal{C}_1$, we have $\bar{\bm\beta}_{i} = [1,0,\ldots,0]$,  $\bar{\bm\beta}_{j}= [\bar{\beta}_{j,i}, \bar{\beta}_{j,2},\ldots, \bar{\beta}_{j,K}]$ with $\bar{\beta}_{j,1} < 1$.
Then,
\begin{align*}
\bar{\bm\beta}_{i} - \bar{\bm\beta}_{j} & = [1-\bar{\beta}_{j,i}, -\bar{\beta}_{j,2},\ldots, -\bar{\beta}_{j,K}] \\
& = (1-\bar{\beta}_{j,i})[1,-c_2,\ldots, -c_K ] := (1-\bar{\beta}_{j,i}) \mathbf{e}^{\top}
\end{align*}
and $\sum_{l=2}^{K}c_{l} =1$.
Therefore, defining $\mathbf{Y} := \bar{\mathbf{R}}\bar{\bm\beta}^{\top}$, we get
\begin{align*}
\Vert \mathbf{E}_i - \mathbf{E}_j \Vert_2 =  (1-\bar{\beta}_{j,i})\Vert \mathbf{Y}_1 - \sum\limits_{l=2}^{K} c_{l} \mathbf{Y}_{l}	\Vert_2
% 
% \geq  (1-\bar{B}_{j,1}) \gamma
\end{align*}
Noting that $\mathbf{Y}$ is at least $\gamma_s$-simplicial, we have
$
\Vert \mathbf{E}_i - \mathbf{E}_j \Vert_2  \geq  (1-b) \gamma_s
$
where $b=\max_{j\in\mathcal{C}_0,k} \bar{\beta}_{j,k} < 1$. 

Similarly, note that $\Vert \mathbf{e}^{\top} \bar{\mathbf{R}}
\Vert\geq \gamma$ and let $\bar{\mathbf{R}} =
\mathbf{U}\Sigma\mathbf{U}^{\top}$ be its singular value
decomposition. If $\lambda_{\max}$ is the maximum eigenvalue of
$\bar{\mathbf{R}}$, then we have
\begin{align*}
E_{i,i}-2E_{i,j}+E_{j,j}&  = (1-\bar{\beta}_{j,1})^{2} (\mathbf{e}^{\top}\bar{\mathbf{R}} ) \mathbf{U}\Sigma^{-1}\mathbf{U}^{\top} (\mathbf{e}^{\top}\bar{\mathbf{R}} )^{\top} \\
& \geq   (1-b)^{2} \gamma_s^{2} / \lambda_{\max}.
\end{align*} 
The inequality in the last step follows from the observation that
$\mathbf{e}^{\top}\mathbf{R}^{\prime}$ is within the column space
spanned by $\mathbf{U}$.
\end{proof}
The results in Proposition~\ref{prop:similarity} provide two
statistics for identifying novel words of the same topic, $\Vert
\mathbf{E}_i - \mathbf{E}_j \Vert$ and $E_{i,i}-2E_{i,j}+E_{j,j}$.
While the first is straightforward, the latter is efficient to
calculate in practice with better computational complexity.
Specifically, its empirical version, the set $\mathcal{J}_{i}$ in Algorithm~\ref{alg:text:rp}
\begin{eqnarray*}
\mathcal{J}_{i}=\{ j: \widehat{E}_{i,i} - \widehat{E}_{i,j} -\widehat{E}_{j,i} + \widehat{E}_{j,j} \geq d/2 \}
\end{eqnarray*}
can be used to discover the set of novel words of the same topics
asymptotically.
Formally,
\begin{proposition}
If $\Vert \widehat{\mathbf{E}} - \mathbf{E} \Vert_\infty \leq (1-b)^{2} \gamma_s^{2} / 8\lambda_{\max}$, then, 
\begin{enumerate}
\item For a novel word $i\in\mathcal{C}_k$ , $\mathcal{J}_{i} = \mathcal{C}_{k}^{c}$
\item For a non-novel word $j\in\mathcal{C}_0$, $\mathcal{J}_{i} \supset \mathcal{C}_{k}^{c}$ 
\end{enumerate}
\end{proposition}
Now we start to show that Algorithm~\ref{alg:text:rp} can detect all
the novel words of the $K$ distinct rankings consistently. As illustrated in Lemma~\ref{lem:extreme-solid-angle}, we detect the novel words by ranking ordering the solid angles $q_{i}$.
We denote the minimum solid angle of the $K$ extreme points by $q_{\wedge}$. 
Our proof is to show that the estimated solid angle in Eq~\eqref{eqa:solid-approx}, 
\begin{equation}
\hat{p}_{i} = \frac{1}{P} \sum_{r=1}^{P} \mathbb{I} \lbrace \forall j\in\mathcal{J}_{i}, ~ \widehat{\mathbf{E}}_{j}  {\mathbf d}^r \leq \widehat{\mathbf{E}}_{i} {\mathbf d}^r \rbrace 
\end{equation}
converges to the ideal solid angle
\begin{align}
q_i = \Pr \lbrace \forall j \in \mathcal{S}(i), (\mathbf{E}_i - \mathbf{E}_j) \mathbf{d}\geq 0 \rbrace
\end{align}
as $M,P\rightarrow \infty$.
$\mathbf{d}^1,\ldots, \mathbf{d}^P$ are iid directions drawn from a isotropic distribution.  For a novel word $i\in\mathcal{C}_k, k=1,\ldots, K$, let $ \mathcal{S}(i) = \mathcal{C}_k^{c}$, and for a non-novel word $i\in\mathcal{C}_0$, let $ \mathcal{S}(i) = \mathcal{C}_0^{c}$.
To show the convergence of $\hat{p}_{i}$ to $p_i$, we consider an intermediate quantity, 
\begin{align*}
 p_i(\widehat{\mathbf{E}}) = \Pr \lbrace \forall j \in \mathcal{J}_{i}, (\widehat{\mathbf{E}}_i - \widehat{\mathbf{E}}_j )\mathbf{d}\geq 0   \rbrace 
\end{align*}
First, by Hoeffding's lemma, we have the following result.
\begin{proposition}
\label{prop:heofdingprop}
$\forall t\geq 0, \forall i$, 
\begin{equation}
\Pr\{\vert \hat{p}_i - p_{i}(\widehat{\mathbf{E}})\vert \leq t \} \geq 2\exp(-2Pt^{2})
\end{equation}
\end{proposition}
Next we show the convergence of $p_{i}(\widehat{\mathbf{E}})$ to solid
angle $q_i$:
\begin{proposition}
\label{prop:solidangle}
Consider the case when $\Vert \widehat{\mathbf{E}} - \mathbf{E}\Vert_{\infty} \leq \frac{d}{8}$ and $\bar{\mathbf{R}}$ is $\gamma_s$-simplicial.
If $i$ is a novel word, then, 
\begin{align*}
q_i - p_{i}(\widehat{\mathbf{E}}) \leq \frac{W\sqrt{W}}{\pi d_2} \Vert \widehat{\mathbf{E}}-\mathbf{E} \Vert_{\infty}
\end{align*}
Similarly, if $j$ is a non-novel word, we have, 
\begin{align*}
 p_{j}(\widehat{\mathbf{E}}) - q_i \leq \frac{W\sqrt{W}}{\pi d_2} \Vert \widehat{\mathbf{E}}-\mathbf{E} \Vert_{\infty}
\end{align*}
where $d_2 \triangleq (1-b)\gamma_s$, $d=(1-b)^{2}\gamma_s^{2}/\lambda_{\max}$.
\end{proposition}
\begin{proof}
First note that, by the definition of $\mathcal{J}_{i}$ and Proposition~\ref{prop:similarity},  if $\Vert \widehat{\mathbf{E}} - \mathbf{E}\Vert_{\infty} \leq \frac{d}{8}$, 
then, for a novel word $i \in \mathcal{C}_k$, $\mathcal{J}_{i} = \mathcal{S}(i)$. And for a non-novel word $i\in\mathcal{C}_0$, $\mathcal{J}_{i} \supseteq \mathcal{S}(i)$.
For convenience, let
\begin{align*}
A_{j}=\{\mathbf{d}: (\widehat{\mathbf{E}}_i - \widehat{\mathbf{E}}_j )\mathbf{d}\geq 0 \} & ~~ A=\bigcap\limits_{j\in \mathcal{J}_{i}} A_j\\
B_{j} = \{\mathbf{d}: (\mathbf{E}_i - \mathbf{E}_j) \mathbf{d} \geq 0  \} & ~~ B=\bigcap\limits_{j\in \mathcal{S}(i)} B_j
\end{align*} 
For $i$ being a novel word, we consider
\begin{align*}
q_i - p_{i}(\widehat{\mathbf{E}})  = \Pr\lbrace B \rbrace - \Pr\lbrace A \rbrace \leq \Pr\lbrace B\bigcap A^{c}\rbrace
\end{align*}
Note that $\mathcal{J}_{i} = \mathcal{S}(i)$ when $\Vert \widehat{\mathbf{E}} - \mathbf{E} \Vert \leq d/8$,  
\begin{align*}
& \Pr\lbrace B\bigcap A^{c} \rbrace  =  \Pr\lbrace B\bigcap (\bigcup\limits_{j\in\mathcal{S}(i)}A_{j}^{c}) \rbrace \\
 & \leq \sum\limits_{j\in\mathcal{S}(i)} \Pr\lbrace (\bigcap\limits_{l\in \mathcal{S}(i)} B_l)\bigcap A_{j}^{c} \rbrace   \leq \sum\limits_{j\in\mathcal{S}(i)} \Pr\lbrace  B_{j} \bigcap A_{j}^{c} \rbrace  \\
& = \sum\limits_{j\in\mathcal{S}(i)} \Pr\lbrace (\widehat{\mathbf{E}}_i - \widehat{\mathbf{E}}_j )\mathbf{d} < 0, \text{and}~  (\mathbf{E}_i - \mathbf{E}_j) \mathbf{d} \geq 0  \rbrace \\
& =  \sum\limits_{j\in\mathcal{S}(i)} \frac{\phi_{j}}{2\pi}
\end{align*}
where $\phi_{j}$ is the angle between $\mathbf{e}_{j} = \mathbf{E}_i - \mathbf{E}_j $ and $\widehat{\mathbf{e}}_{j} = \widehat{\mathbf{E}}_i - \widehat{\mathbf{E}}_j$ for any isotropic distribution on $\mathbf{d}$. 
Noting that $\phi \leq \tan (\phi)$, 
\begin{align*}
\Pr\lbrace B\bigcap A^{c} \rbrace & \leq \sum\limits_{j\in\mathcal{S}(i)} \frac{\tan(\phi_{j})}{2\pi} \leq \sum\limits_{j\in\mathcal{S}(i)} \frac{1}{2\pi} \frac{\Vert \widehat{\mathbf{e}}_{j} - \mathbf{e}_{j} \Vert_2}{\Vert \mathbf{e}_{j} \Vert_2} \\
& \leq \frac{W\sqrt{W}}{\pi d_2} \Vert \widehat{\mathbf{E}}-\mathbf{E} \Vert_{\infty}
\end{align*}
where the last inequality is obtained by the relationship between the
$\ell_\infty$ norm and the $\ell_2$ norm, and the fact that for
$j\in\mathcal{S}(i)$, $\Vert \mathbf{e}_{j} \Vert_2 = \Vert
\mathbf{E}_i - \mathbf{E}_j \Vert_2 \geq d_2 \triangleq
(1-b)\gamma_s$.
Therefore for a novel word $i$, we have, 
\begin{align*}
q_i - p_{i}(\widehat{\mathbf{E}}) \leq \frac{W\sqrt{W}}{\pi d_2} \Vert \widehat{\mathbf{E}}-\mathbf{E} \Vert_{\infty}
\end{align*}
Similarly for a non-novel word $i\in\mathcal{C}_0$, $\mathcal{J}_{i} \supseteq \mathcal{S}(i)$,
\begin{align*}
p_{i}(\widehat{\mathbf{E}}) -q_{i}  = &\Pr\lbrace A \rbrace - \Pr\lbrace B \rbrace =  \Pr\lbrace A\bigcap B^{c}\rbrace \\
\leq & \sum\limits_{j\in\mathcal{S}(i)} \Pr\lbrace (\bigcap\limits_{l\in \widehat{\mathcal{S}}(i)} A_l)\bigcap B_{j}^{c} \rbrace \\
\leq & \sum\limits_{j\in\mathcal{S}(i)} \Pr\lbrace  A_{j} \bigcap B_{j}^{c} \rbrace \leq  \frac{W\sqrt{W}}{\pi d_2} \Vert \widehat{\mathbf{E}}-\mathbf{E} \Vert_{\infty}
\end{align*} 
\end{proof}
A direct implication of Proposition~\ref{prop:solidangle} is,
\begin{proposition}
\label{prop:solidangleconverge}
$\forall \epsilon >0$, let $\rho =\min\{ \frac{d}{8}, \frac{\pi d_2 \epsilon}{W^{1.5}} \}$. If $\Vert \widehat{\mathbf{E}} -\mathbf{E} \Vert_\infty \leq \rho$, 
then, $q_{i} - p_{i} (\widehat{\mathbf{E}}) \leq \epsilon$ for a novel word $i$ and $p_{j} (\widehat{\mathbf{E}}) - q_{j} \leq \epsilon$ for a non-novel word $j$.
\end{proposition}
We now prove Theorem~\ref{thm:novel-word-detection}. In order to correctly detect all the novel words of $K$ distinct topics, we decompose the error event to be the union of the following two types,
\begin{enumerate}
\item {\it Sorting error}, i.e., $\exists i\in\bigcup_{k=1}^{K}\mathcal{C}_k, \exists j\in\mathcal{C}_0$ such that $\hat{p}_i < \hat{p}_j$. This event is denoted as $A_{i,j}$ and let $A= \bigcup A_{i,j}$. 
\item {\it Clustering error}, i.e., $\exists k, \exists i,j\in\mathcal{C}_k$ such that $i\notin\mathcal{J}_{j}$. This event is denoted as ${B}_{i,j}$ and let ${B} = \bigcup B_{i,j}$
\end{enumerate} 
We point out that the event $A,B$ are different from the notations we used in Proposition~\ref{prop:solidangle}. According to Proposition~\ref{prop:solidangleconverge}, we also define $\rho =\min\{ \frac{d}{8}, \frac{\pi d_2 q_{\wedge}}{4 W^{1.5}} \}$ and the event that $C = \{ \Vert \mathbf{E} -\widehat{\mathbf{E}} \Vert_\infty \geq \rho \}$.
We note that $B\subsetneq C$. 

Therefore, 
\begin{eqnarray*}
Pe & =& \Pr\{ A \bigcup B \} \leq  \Pr\{ A \bigcap C^{c} \} + \Pr\{C\} \\
&\leq & \sum_{i~novel, j~non-novel} \Pr\{ A_{i,j}\bigcap B^{c} \} + \Pr\{C\}\\
& \leq & \sum_{i,j} \Pr( \hat{p}_i - \hat{p}_j <0 \bigcap \Vert \widehat{\mathbf{E}} - \mathbf{E} \Vert_\infty \geq \rho ) \\
& &+ \Pr(\Vert \widehat{\mathbf{E}} - \mathbf{E} \Vert_\infty > \rho)
\end{eqnarray*}
The second term can be bound by Lemma~\ref{lem:second-order-convergence}. Now we focus on the first term. Note that
\begin{eqnarray*}
\hat{p}_i -  \hat{p}_j &=& \hat{p}_i - \hat{p}_j - p_i(\widehat{\mathbf{E}})+ p_i(\widehat{\mathbf{E}}) \\
& & - q_i + q_i  - p_j(\widehat{\mathbf{E}})+p_j(\widehat{\mathbf{E}})-q_j+q_j \\
& = & \{ \hat{p}_i - p_i(\widehat{\mathbf{E}}) \} + \{ p_i(\widehat{\mathbf{E}}) - q_i \} \\
& & + \{ p_j(\widehat{\mathbf{E}})- \hat{p}_j \} + \{q_j - p_j(\widehat{\mathbf{E}}) \} \\
& & + q_i - q_j
\end{eqnarray*}
and the fact that $q_i - q_j \geq q_{\wedge}$, then,, 
\begin{eqnarray*}
&&\Pr( \hat{p}_i < \hat{p}_j \bigcap \Vert \widehat{\mathbf{E}} - \mathbf{E} \Vert_\infty \leq \rho ) \\
&\leq & \Pr( p_i(\widehat{\mathbf{E}})- \hat{p}_i \geq q_{\wedge}/4) + \Pr( \hat{p}_j - p_j(\widehat{\mathbf{E}}) \geq q_{\wedge}/4  ) \\
& & + \Pr(q_i - p_i(\widehat{\mathbf{E}})\geq q_{\wedge}/4)\bigcap \Vert \widehat{\mathbf{E}} - \mathbf{E} \Vert_\infty \leq \rho   ) \\
& & + \Pr( p_j(\widehat{\mathbf{E}})-q_j\geq q_{\wedge}/4)\bigcap \Vert \widehat{\mathbf{E}} - \mathbf{E} \Vert_\infty \leq \rho   ) \\
&\leq & 2\exp(-P q_{\wedge}^{2}/8) \\
& &+ \Pr(q_i - p_i(\widehat{\mathbf{E}})\geq q_{\wedge}/4)\bigcap \Vert \widehat{\mathbf{E}} - \mathbf{E} \Vert_\infty \leq \rho   ) \\
& & + \Pr( p_j(\widehat{\mathbf{E}})-q_j\geq q_{\wedge}/4)\bigcap \Vert \widehat{\mathbf{E}} - \mathbf{E} \Vert_\infty \leq \rho   )
\end{eqnarray*}
The last equality is by Proposition~\ref{prop:heofdingprop}.
For the last two terms, by Proposition \ref{prop:solidangleconverge} is 0. Therefore, applying Lemma~\ref{lem:second-order-convergence} we obtain,
\begin{eqnarray*}
Pe \leq  2W^2\exp(-P q_{\wedge}^{2}/8)+  8W^2\exp(-\rho^{2}\eta^{4}MN/20 )
\end{eqnarray*}
And this concludes Theorem~\ref{thm:novel-word-detection}. 
%
%
%%%%%%%%%%%%%%%%%%%%%%%%%%%%%%%%%%%%%%%%%%
\subsection{Proof of Theorem~\ref{thm:topic-estimation}}
% Now we show that Algorithm~\ref{alg:text:esttopic1} can consistently estimate the topic matrix $\bm{\beta}$, given the success of the Algorithm~\ref{alg:text:rp}. 
%
Without loss of generality, let $1,\ldots, K$ be the novel words of topic $1$ to $K$. 
We first consider the solution of the constrained linear regression. To simplify the notation, we denote $\mathbf{E}_{i} = \left[ E_{i,1},\ldots, E_{i,K}\right]$ are the first $K$ entries of a row vector without the super-scripts as in Algorithm~\ref{alg:text:esttopic1}. 
\begin{proposition}
\label{prop:optimizationconverge}
Let $\bar{\mathbf{R}}$ be $\gamma_a$-affine-independent.
The solution to the following optimization problem 
\begin{align*}
\widehat{\mathbf{b}}^{*} =  \arg\min_{b_j \geq 0, \sum b_j =1} \Vert \widehat{\mathbf{E}}_i -\sum\limits_{j=1}^{K} b_j \widehat{\mathbf{E}}_j\Vert
\end{align*}
converges to the $i$-th row of $\bar{\bm\beta}$, $\bar{\bm\beta}_{i}$, as $M\rightarrow\infty$.
Moreover,
\begin{align*}
\Pr ( \Vert \widehat{\mathbf{b}}^{*} -  \bar{\bm\beta}_{i}\Vert_\infty \geq \epsilon )\leq 8W^2 \exp(- \frac{\epsilon^2 MN \gamma_a^{2}\eta^4}{320 K})
\end{align*}
where $\eta$ is define the same as in Lemma~\ref{lem:second-order-convergence}.
\end{proposition}
\begin{proof}
We note that $\bar{\bm\beta}_{i}$ is the optimal solution to the following problem with ideal word co-occurrence statistics
\begin{align*}
{\mathbf{b}}^{*} =  \arg\min_{b_j \geq 0, \sum b_j =1} \Vert {\mathbf{E}}_i -\sum\limits_{j=1}^{K} b_j {\mathbf{E}}_j\Vert
\end{align*}
Define $f(\mathbf{E},\mathbf{b}) = \Vert {\mathbf{E}}_i
-\sum_{j=1}^{K} b_j {\mathbf{E}}_j\Vert$ and note the fact that
$f(\mathbf{E},\mathbf{b}^{*}) =0$. Let $\mathbf{Y} =
[\mathbf{E}_1^{\top}, \ldots, \mathbf{E}_K^{\top} ]^{\top}$. Then,
\begin{align*}
& f(\mathbf{E},\mathbf{b}) - f(\mathbf{E},\mathbf{b}^{*})  =\Vert {\mathbf{E}}_i -\sum\limits_{j=1}^{K} b_j {\mathbf{E}}_j \Vert -0 \\
=& \Vert \sum\limits_{j=1}^{K} (b_j - b_j^{*}) {\mathbf{E}}_j \Vert =\sqrt{ (\mathbf{b} - \mathbf{b}^{*}) \mathbf{Y Y^{\top}} (\mathbf{b} - \mathbf{b}^{*})^{\top} }\\
\geq & \Vert \mathbf{b} - \mathbf{b}^{*} \Vert \gamma_a
\end{align*}
The last equality is true by the definition of affine-independence.
Next, note that,
\begin{align*}
 \vert f(\mathbf{E},\mathbf{b}) - f(\widehat{\mathbf{E}},\mathbf{b}) \vert \leq & \Vert \mathbf{E}_i -\widehat{\mathbf{E}}_i + \sum b_j (\widehat{\mathbf{E}}_j - \mathbf{E}_j) \Vert \\
\leq & \Vert \mathbf{E}_i -\widehat{\mathbf{E}}_i \Vert + \sum b_j \Vert \widehat{\mathbf{E}}_j - \mathbf{E}_j \Vert \\
\leq & 2 \max_{w} \Vert \widehat{\mathbf{E}}_w - \mathbf{E}_w \Vert
\end{align*}
Combining the above inequalities, we obtain,  
\begin{align*}
\Vert \widehat{\mathbf{b}}^{*} - \mathbf{b}^{*} \Vert \leq & \frac{1}{\gamma_a} \lbrace f(\mathbf{E},\widehat{\mathbf{b}}^{*}) - f(\mathbf{E},{\mathbf{b}}^{*}) \rbrace \\
= &   \frac{1}{\gamma_a} \lbrace f(\mathbf{E},\widehat{\mathbf{b}}^{*}) -f(\widehat{\mathbf{E}},\widehat{\mathbf{b}}^{*}) + f(\widehat{\mathbf{E}},\widehat{\mathbf{b}}^{*}) \\
& ~~ - f(\widehat{\mathbf{E}},{\mathbf{b}}^{*}) +f(\widehat{\mathbf{E}},{\mathbf{b}}^{*})  - f(\mathbf{E},{\mathbf{b}}^{*}) \rbrace \\
\leq &  \frac{1}{\gamma_a} \lbrace f(\mathbf{E},\widehat{\mathbf{b}}^{*}) -f(\widehat{\mathbf{E}},\widehat{\mathbf{b}}^{*}) +f(\widehat{\mathbf{E}},{\mathbf{b}}^{*})  - f(\mathbf{E},{\mathbf{b}}^{*}) \rbrace \\
\leq & \frac{4 K^{0.5}}{\gamma_a}  \Vert \widehat{\mathbf{E}} - \mathbf{E} \Vert_\infty
\end{align*}
where the last term converges to $0$ almost surely. The convergence rate follows directly from Lemma~\ref{lem:second-order-convergence}.
\end{proof}
We next consider the row renormalization. Let $\hat{\mathbf{b}}^{*} (i)$ be the optimal solution in Proposition~\ref{prop:optimizationconverge} for the $i$-th word, and consider 
\begin{align}
\label{eq:rowscale}
\widehat{\mathbf{B}}_{i}& := \hat{\mathbf{b}}^{*} (i)^{\top} (\frac{1}{M}\mathbf{X}\mathbf{1}_{M\times 1}) \rightarrow  \bm\beta_i \diag(\mathbf{a}) 
\end{align}
To show the convergence rate of the above equation, it is straightforward to apply the result in Lemma~\ref{lem:second-order-convergence}
\begin{proposition}
\label{prop:rowscaling}
For the row-scaled estimation $\hat{\mathbf{B}}_{i}$ as in Eq.~\eqref{eq:rowscale}, we have,
\begin{equation*}
\Pr( \vert \hat{\mathbf{B}}_{i,k} - \bm\beta_{i,k}a_k \vert \geq \epsilon ) \leq 8W^2 \exp(- \frac{\epsilon^2 MN \gamma_a^{2}\eta^4}{1280 K})
\end{equation*}
\end{proposition}
\begin{proof}
By Proposition~\ref{prop:optimizationconverge}, we have,
\begin{align*}
\Pr ( \vert \widehat{\mathbf{b}}^{*}(i)_{k} -  \bar{\bm\beta}_{i,k}\vert \geq \epsilon/2 )\leq 8W^2 \exp(- \frac{\epsilon^2 MN \gamma_a^{2}\eta^4}{1280 K})
\end{align*}
Recall that in Lemma~\ref{lem:second-order-convergence} by McDiarmid's inequality, we have
\begin{align*}
\Pr ( \vert \frac{1}{M}\mathbf{X}\mathbf{1}_{M\times 1} - \mathbf{B}_{i}\mathbf{a} \vert \geq \epsilon/2)\leq \exp(-\epsilon^2 MN/2)
\end{align*}
Therefore, 
\begin{align*}
& \Pr( \vert \hat{\mathbf{B}}_{i,k} - \bm\beta_{i,k}a_k \vert  \geq \epsilon ) \\
 \leq & 8W^2 \exp(- \frac{\epsilon^2 MN \gamma_a^{2} \eta^4}{1280 K}) + \exp(-\epsilon^2 MN/2) \\
\end{align*}
where the second term is dominated by the first term. 
\end{proof}
Finally, we consider the column normalization step to remove the effect of $\diag(\mathbf{a})$:
\begin{align}
\widehat{\bm\beta}_{i,k} : = \widehat{\mathbf{B}}_{i,k} / \sum_{w=1}^{W} \widehat{\mathbf{B}}_{w,k}
\end{align}
And $\sum_{w=1}^{W} \widehat{\mathbf{B}}_{w,k} \rightarrow \mathbf{a}_k$ for $k=1,\ldots, K$. A worst case analysis on its convergence is,
\begin{align*}
\Pr(\vert \sum_{w=1}^{W} \widehat{\mathbf{B}}_{w,k} - \mathbf{a}_k \vert > \epsilon ) & \leq W \Pr( \vert \hat{\mathbf{B}}_{i,k} - \bm\beta_{i,k}a_k \vert \geq \epsilon/W )\\
 & \leq  8W^3 \exp(- \frac{\epsilon^2 MN \gamma_a^{2}\eta^4}{1280 W^2 K})
\end{align*}
Combining all the result above, we can show $\forall i=1,\ldots,W, \forall k = 1,\ldots, K $,
\begin{align*}
 \Pr(\vert \widehat{\bm\beta}_{i,k} -\bm\beta_{i,k} \vert >\epsilon) \leq  8 W^{4} K  \exp(- \frac{\epsilon^2 MN \gamma_a^{2} a_{\min}^{2} \eta^4}{2560 W^2 K})
\end{align*}
where $a_{\min} > 0$ is the minimum value of entries of $\mathbf{a}$.  This concludes the result of Theorem~\ref{thm:topic-estimation}.
 
%
\subsection{Proof of Lemma~\ref{lem:separable-equiv-irreducible}}
\label{sec:irreducibility-proof}
\begin{proof}
We first show that irreducibility implies separability, or
equivalently, if the collection is not separable, then it is not
irreducible. Suppose that $\{\nu_1,\ldots, \nu_K\}$ is not separable.
Then there exists some $k \in [K]$ and a $\delta > 0$ such that,
\begin{align*}
\inf\limits_{A:~\nu_k(A)>0}\max_{j:~j\neq k}
\frac{\nu_{j}(A)}{\nu_{k}(A)} = \delta >0.
\end{align*}
Then $\forall A\in\mathcal{F}: \nu_k(A)>0$, $ \max\limits_{j:~j\neq k}
\frac{\nu_{j}(A)}{\nu_{k}(A)} \geq \delta$.  This implies that
$\forall A \in \mathcal{F}: \nu_k(A)>0$,
\begin{align*}
\sum\limits_{j:~j\neq k}\nu_{j}(A) - \delta\nu_{k}(A) \geq 0.
\end{align*}

On the other hand, $\forall A \in \mathcal{F}:\nu_k(A)=0$, we have
\begin{align*}
\sum\limits_{j:~j\neq k}\nu_{j}(A) - \delta\nu_{k}(A) =
\sum\limits_{j:~j\neq k}\nu_{j}(A) \geq 0.
\end{align*}
Thus the linear combination $\sum_{j\neq k}\nu_{j} - \delta \nu_{k}$
with one strictly negative coefficient $-\delta$ is nonnegative over
all measurable $A$.
%Hence $\sum_{j\neq k}\nu_{j} - \delta \nu_{k}$ is a valid measure
%whose combination weights are not nonnegative. 
This implies that the collection of measures $\{\nu_1,\ldots, \nu_K\}$
is not irreducible.
We next show that separability implies irreducibility. If the
collection of measures $\{\nu_1,\ldots,\nu_K\}$ is separable, then by
the definition of separability, $\forall k$, $\exists
A_{n}^{(k)}\in\mathcal{F}, n=1,2,\ldots,$ such that
$\nu_{k}(A_{n}^{(k)}) >0$ and $\forall j\neq k$,
$\frac{\nu_{j}(A_{n}^{(k)})}{\nu_{k}(A_{n}^{(k)})}\rightarrow 0$ as $n
\rightarrow \infty$. Now consider any linear combination of measures
$\sum_{i=1}^{K}c_i\nu_{i}$ which is nonnegative over all measurable
sets, i.e., for all $A\in \mathcal{F}$, $\sum_{i=1}^{K}c_i\nu_{i}(A)
\geq 0$. Then $\forall k=1,\ldots, K$ and all $n\geq 1$ we have,
\begin{align*}
\sum_{i=1}^{K}c_i\nu_{i}(A_{n}^{(k)}) &\geq 0 \\
\Rightarrow \nu_{k}(A_{n}^{(k)})\left(c_k + \sum_{j\neq k}c_j
\frac{\nu_{j}(A_{n}^{(k)})}{\nu_k(A_{n}^{(k)})}\right) &\geq 0 \\
\Rightarrow c_k \geq -\sum_{j\neq k}c_j
\frac{\nu_{j}(A_{n}^{(k)})}{\nu_k(A_{n}^{(k)})} & \rightarrow 0 \mbox{
  as } n \rightarrow \infty.
\end{align*}
Therefore, $c_k\geq 0$ for all $k$ and the collection of measures is
irreducible.
\end{proof}


\ifCLASSOPTIONcaptionsoff
  \newpage
\fi



\bibliographystyle{IEEEtran}
\bibliography{IEEEabrv,ref}
\end{document}